\documentclass[11pt]{article}

\usepackage{epsfig,amsmath,latexsym,amssymb}
\usepackage{graphicx}
\usepackage{lscape}
\usepackage{picture, eso-pic, tikz} 

\usepackage{dsfont}
\usepackage{listings}
\usepackage{changes}
\usepackage{hyperref}
\usepackage{bbding}
\usepackage{tikz-cd}
\usepackage{mathtools}

\renewcommand{\baselinestretch}{1.1}
\def\R{{\mathbb R}}  
\def\N{{\mathbb N}}  
\def\p{{\mathbb P}}  
\def\q{{\mathbb Q}}  
\def\E{{\mathbb E}}  
\def\Beweis{\footnotesize}

\newcommand{\Remm}[1]{}
\newtheorem{theo}{Theorem}[section]
\newtheorem{lemma}[theo]{Lemma}
\newtheorem{prop}[theo]{Proposition}

\newtheorem{cor}[theo]{Corollary}
\newtheorem{defi}[theo]{Definition}
\newtheorem{model ass}[theo]{Model Assumptions}

\newtheorem{ass}[theo]{Assumption}

\newtheorem{example}[theo]{Example}

\newtheorem{remark}[theo]{Remark}
\newtheorem{rems}[theo]{Remarks}
\newtheorem{summary}[theo]{Summary}

\def\EndProof{\hfill {\scriptsize $\Box$}}
\def\EndExample{\hfill {\scriptsize $\blacksquare$}}
\numberwithin{equation}{section}

\definecolor{MyGray}{rgb}{0.92,0.92,0.92}
\makeatletter\newenvironment{graybox}{%
   \begin{lrbox}{\@tempboxa}\begin{minipage}{\columnwidth}}{\end{minipage}\end{lrbox}%
   \colorbox{MyGray}{\usebox{\@tempboxa}}
}\makeatother

\newcommand{\bl}[1]{\textcolor{blue}{{#1}}}
\newcommand{\red}[1]{\textcolor{red}{{#1}}}

\definecolor{British racing}{rgb}{0.0, 0.5, 0.0}

\def\bX{\boldsymbol{X}}

\def\bm{\boldsymbol{m}}

\def\b0{\boldsymbol{0}}

\newcommand{\dd}{\,\mathrm{d}}
\newcommand{\1}{\mathbf{1}}

\begin{document}
\author{Mario V.~W{\"u}thrich\footnote{\normalfont Department of Mathematics, ETH Zurich; mario.wuethrich@math.ethz.ch}}

\date{\today}
\title{{\sc A Practical Guide on Graphical Model Validation}}
\maketitle

\begin{abstract}
\noindent  
This manuscript formalizes the most popular model validation tools used in general insurance actuarial modeling. These include graphical tools like calibration plots, actual-vs-expected plots, lift charts, Murphy diagrams, as well as classical statistical tools such as Bregman losses, deviance losses, elementary losses, Murphy's decomposition and Gini scores. Particular emphasis is placed on whether calibration and discrimination are studied under a policy-weighted or an exposure-weighted population measure. This distinction is crucial in ensuring that premium schemes are calibrated on the correct scale.

\bigskip

\end{abstract}

\section{Introduction}
The main purpose of this manuscript is to discuss several methods that are useful for model validation in general insurance actuarial pricing. We present graphical tools like calibration plots, actual-vs-expected plots, lift charts, Murphy diagrams, as well as classical statistical tools such as Bregman losses, deviance losses, elementary losses, Murphy's decomposition and Gini scores.

In the presentation of these model validation tools, the whole discussion will be centered around {\it calibration}, {\it discrimination} ({\it resolution}) and {\it risk ranking}. These are essential features that actuarial pricing models should possess. A crucial point in these considerations is the correct scale and probability measure. We discuss the difference between the {\it policy-weighted} population measure and the {\it exposure-weighted} population measure. This distinction is important in actuarial pricing, but only the recent paper of Lindholm et al.~\cite{LLP} systematically discusses this distinction. In fact, many papers discuss the theory under a policy-weighted view and then present an applied actuarial example in the exposure-weighted view. The alignment of these two views requires additional assumptions; we discuss these in Section \ref{sec: Calibration}.

We emphasize that none of the methods presented in this paper is  new. They have been floating around for a long time in actuarial practice and in the applied actuarial literature. However, often, they are not properly mathematically formalized. This is precisely the main contribution of this manuscript. We bring all these different methods and concepts on a common mathematical ground which facilitates understanding and comparison.
For example, what does it mathematically mean to perform insurance policy binning resulting in bins of equal exposures, and how can we write this under the correct population measure? Having a common mathematical view on these different concepts, we will realize that many of them present the similar statistics in slightly different views.

\paragraph{Literature overview.}
We skip the literature review at this stage, but all the important references are cited throughout the manuscript whenever they are relevant.

\paragraph{AI declaration.} The author developed the concepts, mathematics, numerical examples, and the original manuscript. ChatGPT-5.6 Sol was used to assist through several iterations of reviewing and revisions of the manuscript.

\paragraph{Acknowledgement.} The author kindly thanks Alexej Brauer, {\L}ukasz Delong, Selim Gatti, Mathias Lindholm, Filip Lindskog, Christian Lorentzen, Marco Maggi, Michael Mayer, Ronald Richman and Dimitri Semenovich for discussing and challenging many of the items presented in this manuscript.

\section{Problem setting}
\label{sec: Problem setting}
\subsection{Loss costs per unit exposure}
An insurance policy is observed for an exposure period $V>0$, and it generates a total loss $Z$ during that period. This yields the loss costs per unit exposure $Y=Z/V$. Equivalently, we scale the total premium $P$ of that contract resulting in the unit premium $\Pi=P/V$. This normalization makes insurance contracts of different exposure lengths comparable. For general insurance pricing, one is then equipped with the random tuple $(Y,V,\Pi)$:
\begin{itemize}
\item  $Y \ge 0$ describes the non-negative {\it loss costs per unit exposure}, 
\item $V>0$ is a strictly positive {\it exposure} (sometimes also called {\it case weight}), and 
\item $\Pi \ge 0$ is the positive {\it unit premium}. 
\end{itemize}
Actuarial modeling considers these normalized quantities, and these yield:
\begin{itemize}
\item the {\it total loss} $Z=VY$, and
\item the {\it total premium} $P=V \Pi$.
\end{itemize}

The statistical literature is often not considering such a split of the total loss $Z$ into loss costs per unit exposure $Y=Z/V$ and an exposure $V$. In actuarial modeling, this split is very common and useful to compare insurance policyholders with contracts of different exposure lengths. From a mathematical viewpoint, this split introduces some complications that we discuss in this section and in Section \ref{sec: Calibration}, below.

\bigskip

\begin{rems}\normalfont\label{Remark covariate}
\begin{itemize}
\item
Insurance policies and their premiums are usually characterized by covariates (features) $\bX$. The above setting covers this situation, because we can think of the unit premium $\Pi=\Pi(\bX)$ being a measurable function of the covariates (policy characteristics).
\item The premiums $\Pi$ and $P$ are always understood as pure risk premiums in this manuscript. That is, these premiums cover the expected loss costs, and they do not contain any additional margins, e.g., for administrative expenses, solvency costs or a profit margin.
\item We treat $(Y,V, \Pi)$ as a random tuple on an underlying probability space $(\Omega, {\cal F}, \p)$. We interpret $\p$ as the population distribution. A realization of $(Y, V, \Pi)$ then corresponds to a randomly selected insurance policy from that population. Naturally, loss costs $Y$ and unit premiums $\Pi$ should be positively associated for $\Pi$ to be a meaningful risk-based pricing rule. The bigger this association the better the premium matches the loss costs, this dependence is implicitly reflected by selecting a suitable population distribution $\p$. This association is formalized and quantified by the resolution term in Murphy's decomposition in Section \ref{sec: Murphy's decomposition}.
\item Throughout, we assume that all considered moments exist and are finite.
\end{itemize}
\end{rems}

\subsection{The exposure-weighted $\q$-measure}
\label{sec: Q-calibration}
An insurance policy is described by the random tuple $(Y,V, \Pi)$ that follows the population distribution $\p$. As emphasized by Lindholm et al.~\cite{LLP}, working with exposure-scaled quantities requires to work under a second distribution $\q$ which accounts for the exposure scaling. 

\medskip

\begin{center}
\noindent\fbox{%
  \begin{minipage}{0.9\textwidth}
    \centering
    ~\\
    $\Longrightarrow$~This manuscript mainly works under the exposure-weighted distribution $\q$.
\\~    
  \end{minipage}%
}
\end{center}

\bigskip

\underline{Briefly explained:} There are two different ways of averaging across an insurance portfolio:
\begin{itemize}
\item The {\it policy-weighted average} assigns the same weight to each policy when computing averages: this is described by the population measure $\p$. 
\item The {\it exposure-weighted average} considers an exposure-weighted average accounting for the different exposure lengths: this is described by the population measure $\q$.
\end{itemize}

\medskip

Working with exposure-scaled quantities $Y$ and $\Pi$ requires ensuring that premium computations capture the dependence between loss costs and exposures. This is achieved by introducing the following exposure-weighted probability measure $\q$. This step simplifies many of the subsequent considerations.

\medskip

Let $\E[\cdot]$ be the expectation under the population measure $\p$.
We define the {\it exposure-weighted population measure} $\q$ by the Radon--Nikodym derivative
\begin{equation}\label{Q Radon Nikodym}
  \frac{\dd \q}{\dd \p} = \frac{V}{\E[V]}.
\end{equation}
This exposure-weighted measure $\q \sim \p$ is an equivalent probability measure and we denote its expectation operator by $\E_{\q}[\cdot]$. 
The $\q$-probability of an event $A \in {\cal F}$ is computed as
\begin{equation*}
\q \left( A \right) =\E_{\q} \left[ \1_A \right] = \frac{1}{\E[V]} \, \E \left[V\1_A\right].
\end{equation*}
\medskip
\begin{graybox}
Our main object of interest are the exposure-weighted loss costs
\begin{equation}\label{exposure-weighted loss costs}
\E_{\q}\left[Y\right] = \frac{1}{\E[V]}\, \E \left[ VY \right]
= \frac{1}{\E[V]}\, \E \left[ Z \right].
\end{equation}
\end{graybox}
\medskip

These expected loss costs per unit exposure $Y$ under $\q$, given by \eqref{exposure-weighted loss costs}, are directly related to the total loss $Z=VY$ under $\p$, and the dependence between $Y$ and $V$ is correctly accounted for in \eqref{exposure-weighted loss costs}. Similarly, we have for the total premium $P$ under $\p$
\begin{equation*}
  \E \left[P \right]
   = \E \left[V \Pi \right] 
  =
\E \left[V\right] \E_{\q}\left[\Pi \right].
\end{equation*}

\medskip
\begin{graybox}
Consequently, {\it global unbiasedness} of $P$ for $Z$ has the two equivalent formulations
\begin{equation} \label{global unbiasedness definition}
  \E \left[Z \right] =\E \left[P \right]  \qquad \Longleftrightarrow \qquad
  \E_{\q}\left[Y\right] = \E_{\q}\left[\Pi \right].
\end{equation}
\end{graybox}

\bigskip

Why does the exposure-weighted $\q$-measure matter? It plays a crucial role in insurance pricing, model fitting and model validation because it correctly accounts for the dependence between the loss costs $Y$ and the exposure $V$. For example, under positive correlation between $Y$ and $V$ under $\p$, the above computations imply
\begin{equation}\label{positive correlation}
\E_{\q}\left[Y\right] = \frac{1}{\E \left[V\right]}\,\E \left[Z\right]
> \E\left[Y\right].
\end{equation}
That is, under positive correlation, the expected loss costs $\E[Y]$ systematically underestimate the scaled expected total loss $\E[Z]/\E[V]$,
and we would charge a too low insurance premium if we used the quantity $\E[Y]$ for pricing.

\begin{example}[Model fitting under $\p$ and $\q$]\label{ex: model fitting}\normalfont
The present notes are mainly dedicated to model validation. Nevertheless, we briefly illustrate that the exposure-weighted $\q$-measure also matters for model fitting. Select the mean squared loss -- the general theory of strictly consistent loss functions is presented in Section \ref{sec: Strictly consistent loss function}, below. Assuming square-integrability, the expected value of $Y$ is found by solving the minimization problem
\begin{equation*}
\E[Y] ~=~ \underset{ x \in \R }{\arg\min}~ \E \left[ \left(Y-x\right)^2 \right].
\end{equation*}
As described in \eqref{positive correlation}, this expected value $\E[Y]$ is not directly useful for pricing the total loss $Z$, because it does not capture the dependence structure of the loss costs $Y$ and the exposure $V$. One therefore generally considers the minimization problem
\begin{equation}\label{Q minimization}
\E_{\q}[Y] ~=~ \underset{ x \in \R }{\arg\min}~ \E_{\q} \left[ \left(Y-x\right)^2 \right]~=~ \underset{ x \in \R }{\arg\min}~ \E \left[ V \left(Y-x\right)^2 \right].
\end{equation}
This gives the expected total claim
\begin{equation*}
\E[V]\,\E_{\q}[Y] = \E[Z].
\end{equation*}
The right-hand side of minimization \eqref{Q minimization} justifies why in the context of actuarial model fitting and validation, the considered loss function 
$L(Y,x)$ is always scaled with the exposure $V$. Namely, this reflects minimization under the exposure-weighted $\q$-measure.
This completes the example.
\EndExample
\end{example}

\medskip

Below, we also need the conditional $\q$-expectation. For a sub-$\sigma$-field ${\cal A} \subset {\cal F}$, it is given by
\begin{equation}\label{conditional Bayes}
  \E_{\q} \left[ X \mid {\cal A}\right] = \frac{1}{\E[V\mid {\cal A}]} \, \E \left[ V X \mid {\cal A} \right];
\end{equation}
this is known as Bayes' rule for conditional expectations, and it reflects prediction under partial information ${\cal A}$. 
The special case of an ${\cal A}$-measurable exposure $V$ yields
\begin{equation}\label{conditional Bayes VV}
  \E_{\q} \left[ X \mid {\cal A}\right] = \E \left[ X \mid {\cal A} \right]
  \qquad \text{ for $V$ being ${\cal A}$-measurable.}
\end{equation}
Thus, when the exposure $V$ is observable w.r.t.~the information ${\cal A}$, we can equivalently work with the conditional population distribution $\p(\cdot \mid {\cal A})$ or with the conditional exposure-scaled distribution $\q(\cdot \mid {\cal A})$. We frequently use this property below.

\section{Calibration}
\label{sec: Calibration}

There are two main properties that non-egalitarian insurance premium schemes should fulfil: {\it Discrimination} and {\it Calibration}.
\begin{itemize}
\item \underline{Discrimination and resolution} concerns the ability of an insurance pricing scheme to distinguish low-risk insurance policies from high-risk policies. This may result in a fine-grained tariff that properly classifies propensity to claims among the considered insurance policyholders. The topic of discrimination and resolution is studied in Sections \ref{sec: Murphy diagram}-\ref{sec: Murphy's decomposition}, below. Implicitly, discrimination and resolution also involves a correct risk ranking which is discussed in Section \ref{section Gini score}, below.
  \item \underline{Calibration} considers the question whether the total premium $P=V \Pi$ charged is sufficient to cover the total loss $Z=VY$. A crude answer is to require global unbiasedness \eqref{global unbiasedness definition}, meaning that on the population level the expected total premium covers the expected total loss. In most of the cases -- excluding an egalitarian pricing system -- we want to have calibration at a finer resolution so that systematic cross-financing between different price cohorts is avoided. This is the topic of this section.
\end{itemize}

\subsection{Notions of calibration}

Global unbiasedness is a requirement at the population level to ensure that the overall premium level is correct. We typically want stronger properties at a finer granularity.
\begin{defi}[$PZ$-calibration] The total premium $P$ is $PZ$-calibrated for the total loss $Z$ if
\begin{equation}\label{PZ calibration}
P = \E \left[Z \mid P \right] \qquad \text{ a.s.}
\end{equation}
\end{defi}
$PZ$-calibration \eqref{PZ calibration} tells us that every price cohort $P$ is on average self-financing for their total loss $Z$, and there is no systematic cross-financing between different price cohorts. This is important, but from an actuarial view point, $PZ$-calibration \eqref{PZ calibration} is not fully satisfactory because it does not disentangle the role of the unit premium $\Pi$ and the exposure $V$ -- low unit premium policies are interpreted as low-risk policies, and high-risk policies have a high unit premium. Only considering the total premium $P$, we may have low-risk policies with long exposures and high-risk policies with short exposures in the same price cohort $P$. Naturally, we should also try to disentangle potential within-price-cohort cross-subsidy in that case, because high-risk policies should not be systematically cross-financed by low-risk ones, or vice versa. For this reason, we are interested in understanding calibration at the unit premium level $\Pi$. This motivates the following definition.

\begin{defi}[$\q$-calibration] The unit premium $\Pi$ is $\q$-calibrated for $(Y,V)$ if
\begin{equation*}
\Pi = \E_{\q} \left[Y \mid \Pi  \right] \qquad \text{ a.s.}
\end{equation*}
\end{defi}

\bigskip
\begin{graybox}
Reformulating $\q$-calibration yields the equivalent formulations
\begin{equation}\label{Q calibration}
\Pi = \E_{\q} \left[Y \mid \Pi  \right]
\qquad \Longleftrightarrow \qquad
\E\left[P \mid \Pi  \right]  = \E \left[Z \mid \Pi  \right] \qquad \text{ a.s.}
\end{equation}
\end{graybox}
\bigskip

Because generally the $\sigma$-fields generated by $P=V \Pi$ and by $\Pi$ do not coincide, $PZ$-calibration \eqref{PZ calibration} and $\q$-calibration \eqref{Q calibration} are different. The latter says that the expected premium collected for any unit premium cohort $\Pi$ is on average self-financing to cover their total loss $Z$.

\medskip

$\q$-calibration \eqref{Q calibration} is also different from $\p$-calibration which is the classical definition of calibration in statistics.
$\p$-calibration is given by
\begin{equation}\label{classical calibration}
\Pi = \E \left[Y \mid \Pi \right] \qquad \text{ a.s.}
\end{equation}
$\q$-calibration \eqref{Q calibration} accounts for the dependence in $(\Pi,V)$ and in $(Y,V)$, whereas $\p$-calibration \eqref{classical calibration} does not, see \eqref{positive correlation}. Therefore, we are generally {\it not} interested in
$\p$-calibration \eqref{classical calibration}.

\medskip

There is yet another peculiarity in actuarial pricing, namely, the information about the premium $\Pi$, the exposure $V$ and the loss costs $Y$ becomes available at different time points. Often, the premium $\Pi$ and the exposure $V$ are available at contract inception -- we call this {\it ex-ante} -- and the total loss $Z=VY$ is only available {\it ex-post}, when the contract expires. In this situation, one may evaluate calibration on the pair $(\Pi, V)$; we further discuss the role of an ex-ante available exposure in Remark \ref{debatable assumptions}, below.
Based on ex-ante information $(\Pi, V)$, there is the following calibration definition.

\begin{defi}[$\q|_V$-calibration] The unit premium $\Pi$ is $\q|_V$-calibrated for $(Y,V)$ if
\begin{equation*}
\Pi = \E_{\q} \left[Y \mid \Pi, V \right] \qquad \text{ a.s.}
\end{equation*}
\end{defi}
Reformulating $\q|_V$-calibration \eqref{actuarial calibration} yields the equivalent formulations
\begin{equation}\label{actuarial calibration}
\Pi = \E_{\q} \left[Y \mid \Pi, V \right]
\quad \Longleftrightarrow \quad
\Pi = \E \left[Y \mid \Pi, V \right]
\quad \Longleftrightarrow \quad
P= \E \left[Z \mid \Pi, V \right] \qquad \text{ a.s.}
\end{equation}
The first equivalence says that conditional on $(\Pi, V)$, the loss costs can either be evaluated under the exposure-weighted measure $\q$ or the original population measure $\p$; this follows directly from
\eqref{conditional Bayes VV}. The second equivalence expresses that in this case it does not matter whether we consider total losses $Z$ or unit loss costs $Y$.
In this sense, $\q|_V$-calibration could also be called {\it exposure-conditional calibration} because it does not depend on the considered probability measure $\p$ or $\q$.

\bigskip

\begin{graybox}
\begin{summary}\normalfont
We summarize the different definitions of calibration: 
\begin{itemize}
\item $PZ$-calibration \eqref{PZ calibration}: $P = \E[Z \mid P]$.
\item $\q$-calibration \eqref{Q calibration}:
$\Pi = \E_{\q}[Y \mid \Pi] ~\Longleftrightarrow~ \E[P \mid \Pi]  = \E[Z \mid \Pi]$.
\item $\p$-calibration \eqref{classical calibration}:
$\Pi = \E[Y \mid \Pi]$.
\item $\q|_V$-calibration \eqref{actuarial calibration}:
$\Pi = \E_{\q} [Y \mid \Pi, V]~\Longleftrightarrow~ \Pi =  [Y \mid \Pi, V]~\Longleftrightarrow~ P =  [Z \mid \Pi, V]$.
\end{itemize}
\end{summary}
\end{graybox}

\bigskip

The following properties are proved by the tower-property of conditional expectation.
\begin{prop}\label{prop: calibration relations}
We have the following three implications:
\begin{eqnarray*}
\text{$\q|_V$-calibration \eqref{actuarial calibration}}
&\Longrightarrow&
\text{$\q$-calibration \eqref{Q calibration}}
\\
\text{$\q|_V$-calibration \eqref{actuarial calibration}}
&\Longrightarrow&
\text{$\p$-calibration \eqref{classical calibration}}
\\
\text{$\q|_V$-calibration \eqref{actuarial calibration}}
&\Longrightarrow&
\text{$PZ$-calibration \eqref{PZ calibration}}
\end{eqnarray*}
\end{prop}

Opposite implications or further implications require additional assumptions. Note that currently we have not made any model assumptions, except that all considered quantities have finite means and that the exposure is strictly positive, a.s. Thus, Proposition \ref{prop: calibration relations} holds in full generality.
We next make a conditional mean independence assumption -- this is our first model assumption -- it is further discussed in Remark \ref{debatable assumptions}, below, and we will abandon it again because it is not generally satisfied in practice.

\begin{ass}[Conditional mean independence] \label{independence of exposure 1}
Assume that 
\begin{equation}\label{eq: independence of exposure 1}
\E \left[Y \mid \Pi, V \right] = \E \left[Y \mid \Pi \right] \qquad \text{ a.s.}
\end{equation}
\end{ass}
\medskip

Under this additional assumption \eqref{eq: independence of exposure 1} we have the following equivalences.
\begin{prop}  \label{prop: strong equivalences}
We have the following equivalences:
  \begin{eqnarray*}
\text{$\q|_V$-calibration \eqref{actuarial calibration}}
& \Longleftrightarrow &
    \text{$\q$-calibration \eqref{Q calibration}}
    \,+\, \text{conditional mean independence \eqref{eq: independence of exposure 1}}
    \\
    & \Longleftrightarrow &
    \text{$\p$-calibration \eqref{classical calibration}}
    \,+\, \text{conditional mean independence \eqref{eq: independence of exposure 1}}.
  \end{eqnarray*}
\end{prop}
Consequently, under conditional mean independence, the three definitions of $\q|_V$-calibration, $\q$-calibration and $\p$-calibration coincide, but this requires Assumption \ref{independence of exposure 1}, otherwise only the implications of Proposition \ref{prop: calibration relations} hold.

\medskip

\begin{remark}[Ex-ante exposures and conditional mean independence]\label{debatable assumptions}\normalfont
We discuss the two items: (1) ex-ante vs.~ex-post exposures -- ex-ante exposures are implicitly used in $\q|_V$-calibration \eqref{actuarial calibration} to make it a practically meaningful object --  and (2) the conditional mean independence assumption
\eqref{eq: independence of exposure 1}. The question whether items (1) and (2) are realistic assumptions in practice is closely related.

\begin{itemize}
\item[(1)] \underline{Ex-ante availability is an information-timing condition.}
$\q|_V$-calibration \eqref{actuarial calibration} is a meaningful consideration if the exposure $V$ is known ex-ante, meaning at contract inception. This is a general assumption made in many actuarial modeling approaches, and it looks reasonable because the insurance contract specifies the insurance term. However, many insurance products contain a lapse option, e.g., a car insurance policy can be terminated in case the policyholder changes their vehicle. Consequently, the realized exposure only becomes available at contract expiry/termination (ex-post); this is a critical point raised and discussed in Lindholm et al.~\cite{LLP}. Thus, early termination and temporary suspension turns the exposure into an ex-post available variable, and in this situation it seems more meaningful to study $\q$-calibration \eqref{Q calibration} because $(\Pi, V)$ is not in the available information set at contract inception.
\item[(2)] \underline{The conditional mean independence is a stochastic relationship.}
The conditional mean independence assumption
\eqref{eq: independence of exposure 1} seems to be even more problematic in practical applications. We highlight some crucial features in the following list:
\begin{itemize}
\item Property \eqref{eq: independence of exposure 1} is a conditional {\it mean} independence, this is weaker than conditional independence. It says that on each unit premium level $\Pi$, there is no systematic effect of the exposure $V$ on the {\it expected} loss costs per unit exposure.
This is interpreted that $\Pi$ is mean-sufficient for $Y$, and the exposure $V$ does not proxy a missing risk factor for unit loss cost prediction.
In particular, we have a proportional mean behavior in $V$ of the total loss $Z$. 
\item It is important to realize that \eqref{eq: independence of exposure 1} is a substantive modeling assumption, and not an automatic consequence of defining $Y=Z/V$. That is, the unit loss costs can always be defined by $Y=Z/V$, but this does not imply that the expected loss has a linearity in the exposure; see also the correlation
statement \eqref{positive correlation}. This linearity is a substantial assumption that can be imposed in the form of \eqref{eq: independence of exposure 1}. Most classical actuarial models impose a linearity assumption, e.g., working within the exponential dispersion family (EDF), one typically assumes that conditionally, given the true mean of the unit loss costs, the conditional mean is independent of the exposure and the conditional variance scales inversely proportionally to the exposure; this is also the common assumption in the B\"uhlmann--Straub credibility model \cite{BS}. For more discussion in a regression context, we also refer to 
Lindholm--Nazar \cite{LindholmNazar}.
\item Property \eqref{eq: independence of exposure 1} may be a reasonable assumption if we have ex-ante exposures (known at contract inception) that do not impact the average loss costs per unit exposures. Such a property may be violated if, e.g., high-risk profiles systematically sign shorter contracts or if there are seasonal patterns resulting in risk profiles that are not ceteris paribus over the entire insured period (e.g., avalanches are more likely during winter periods).
\item If exposures are ex-post, i.e., only known at contract termination, then likely \eqref{eq: independence of exposure 1} is violated if there is an endogenous mechanism between claims and contract termination. A common reason for lapsing a contract is an accident (because, e.g., the insured object is replaced). Such accident-induced lapses directly act on the exposures and on the loss costs, thus, conditional mean independence is not a reasonable assumption in such settings because the exposure contains additional information about the loss.
\item The conditional mean independence can be analyzed graphically by a calibration plot which additionally stratifies with exposure; calibration plots are discussed in Section \ref{sec: Calibration plot}, below. Alternatively, we can consider a two-dimensional kernel smoothed heatmap that considers 
\begin{equation}\label{regression H0}
(\pi,v)~\mapsto~ \E[Y\mid \Pi=\pi,V=v]. 
\end{equation}
We can also fit a regression model to \eqref{regression H0} to understand whether there is a systematic effect of the exposure $V$ on $Y$, given $\Pi$.
\end{itemize}
\end{itemize}
\end{remark}


\bigskip

\begin{center}
\noindent\fbox{%
  \begin{minipage}{0.9\textwidth}
    ~\\
    \underline{Conclusion from Remark \ref{debatable assumptions}:} We should generally doubt the validity of the conditional mean independence assumption \eqref{eq: independence of exposure 1} in general insurance pricing. Consequently, we will not require that the exposure is ex-ante, and our focus is on $\q$-calibration \eqref{Q calibration} throughout the remainder of this manuscript.
\\~    
  \end{minipage}%
}
\end{center}

\bigskip


\underline{What does the actuarial literature do?}
\begin{itemize}
\item No exposures ($V\equiv 1$): Denuit et al.~\cite{DenuitEtAl2021, DenuitEtAl2024} do not involve exposures, however, in their estimation procedure they use a total premium $V\Pi$-view to restore calibration. Fissler et al.~\cite{FisslerLorentzenMayer} do not involve exposures, however, they discuss estimation in their Section 5.2.2 relating to a $\q$-view. W\"uthrich \cite{WGini, WTest}, Denuit--Trufin \cite{DenuitTrufin2023, DenuitTrufin2024} and Delong--W\"uthrich \cite{DW1} do not involve exposures. W\"uthrich--Ziegel \cite{WZiegel} do not involve exposures, though their example considers $\q$-calibration through the application of the isotonic regression with case weights. 
\item $PZ$-calibration: W\"uthrich--Merz \cite[Formula (7.39)]{WM2023} and Denuit et al.~\cite{DenuitHuyghe}.
\item $\p$-calibration:  Delong et al.~\cite{DGW, DW2} and Brauer et al.~\cite{BrauerMenzel}, these papers additionally assume that conditional mean independence holds, thus, $\p$-calibration is equivalent to $\q$-calibration and $\q|_V$-calibration, see Proposition \ref{prop: strong equivalences}.
\item $\q$-calibration: Lindholm et al.~\cite{LLP} and Gatti \cite[formula (4.3)]{Gatti}.
\end{itemize}

From this list we see that most actuarial literature excludes variable exposures. Under exposures there are either the $PZ$-calibration view or the $\q$-calibration view used (sometimes under the additional conditional mean independence assumption). When it comes to the more applied actuarial literature, it is usually $\q$-calibration that is studied, though this is not particularly emphasized in the notation, e.g., Goldburd et al.~\cite[Section 7.2.1]{Goldburd} compare weighted unit loss cost averages against weighted unit premium averages in price buckets that have roughly the same exposure, this equi-exposure view corresponds to a discretized $\q$-calibration view, which will be denoted by $G_{\Pi}^{\q}$, below.

\subsection{Recalibrated unit premium}
The previous section has introduced different versions of calibration, and our focus is on $\q$-calibration \eqref{Q calibration} which has the equivalent definitions
\begin{equation*}
\Pi = \E_{\q} \left[Y \mid \Pi  \right]
\qquad \Longleftrightarrow \qquad
\E\left[P \mid \Pi  \right]  = \E \left[Z \mid \Pi  \right] \qquad \text{ a.s.}
\end{equation*}
That is, we stratify w.r.t.~the unit premium $\Pi$ to understand whether the resulting premium cohorts are on average self-financing for their claims.

\begin{defi}[Recalibrated unit premium]\label{def: Recalibrated unit premium}
The {\it recalibrated unit premium} under the exposure-weighted $\q$-measure is defined by
\begin{equation*}
m_{\q}(\Pi) = \E_{\q} \left[Y \mid \Pi \right],
\end{equation*}
and under the original population measure $\p$ by
\begin{equation*}
m_{\p}(\Pi) = \E \left[Y \mid \Pi \right].
\end{equation*}
\end{defi}

\medskip
\begin{graybox}
We have the following interesting result, a.s.,
\begin{equation}\label{correlation is important}
m_{\q}(\Pi) -m_{\p}(\Pi) = \frac{\operatorname{Cov} \left(V,Y \mid \Pi \right)}{\E \left[V \mid \Pi \right]}.
\end{equation}
\end{graybox}
\medskip

This result indicates that Assumption \ref{independence of exposure 1} is sufficient to have an identity $m_{\q}(\Pi) =m_{\p}(\Pi)$, a.s., but necessary is only conditional uncorrelatedness between $V$ and $Y$, given $\Pi$.

\bigskip
\begin{graybox}
Our main interest is in the exposure-weighted $\q$-calibration, and our goal is to verify
\begin{equation}  \label{Q calibration valid}
m_{\q}(\Pi) = \Pi \qquad \text{ a.s.}
\end{equation}
\end{graybox}
\medskip

\begin{remark}\normalfont
\label{remark qq calibrated}
\begin{itemize}
\item Since $\q \sim \p$ are equivalent probability measures, the term a.s.~(almost surely) is correct under any of the two population measures.
\item The recalibrated unit premiums are calibrated, i.e., 
\begin{equation}\label{recalibration step works}
m_{\q}(\Pi) = \E_{\q} \left[Y \mid m_{\q}(\Pi) \right]
\qquad \text{ and } \qquad 
m_{\p}(\Pi) = \E \left[Y \mid m_{\p}(\Pi) \right]
\qquad \text{ a.s.}
\end{equation}
This directly follows from the tower property of conditional expectation, and it motivates the isotonic recalibration step discussed in Section \ref{Actual-vs-expected plot -- statistical smoothing methods}, below.
\end{itemize}
\end{remark}

\medskip

\subsection{Stylized example}
\label{sec: stylized example}

\bigskip

\begin{center}
\noindent\fbox{%
  \begin{minipage}{0.9\textwidth}
    ~\\
This section constructs a stylized example that is $\p$-calibrated but not $\q$-calibrated. This example will be used throughout the subsequent sections that introduce graphical tools and quantitative statistical methods to analyze calibration and discrimination.
\\~    
  \end{minipage}%
}
\end{center}

\medskip

The example is constructed in two parts: Part 1 shows the methodological construction, and Part 2 gives a numerical example.

\paragraph{Part 1 of the stylized example.}

We construct a stylized example that is $\p$-calibrated but not $\q$-calibrated.
This requires that the conditional mean independence
\eqref{eq: independence of exposure 1} is violated, see Proposition \ref{prop: strong equivalences}, to obtain a non-zero conditional correlation in \eqref{correlation is important}. Select a bounded measurable function $h(\cdot)$ such that
\begin{eqnarray}\label{violation conditional mean independence}
\E\left[Y \mid \Pi , V \right] &=& \Pi + r(\Pi, V)  
\\&=& \nonumber
\Pi + h(V) - \E[h(V) \mid \Pi] >0 
\qquad \text{ a.s.}
\end{eqnarray}
This implies $\p$-calibration because we have centered residuals $\E[r(\Pi, V)\mid \Pi]=0$, that is,
\begin{eqnarray}\nonumber
m_{\p}(\Pi)&=&
\E\left[Y \mid \Pi \right] ~=~ 
\E \left[ \E\left[Y \mid \Pi , V \right] \mid \Pi  \right]
\\&=& \label{example P-calibration 0}
\E \left[\Pi + r(\Pi, V) \mid \Pi  \right]
~=~ \Pi
\qquad \qquad \text{ a.s.}
\end{eqnarray}
Using \eqref{correlation is important}, we compute the following covariance term
\begin{eqnarray*}
\operatorname{Cov} \left(V,Y \mid \Pi \right)
&=&
\operatorname{Cov} \left(V,\E\left[Y\mid \Pi, V\right] \mid \Pi \right)
\\&=&
\operatorname{Cov} \left(V,\Pi + r(\Pi, V) \mid \Pi \right)
\\&=& \operatorname{Cov} \left(V, h(V) \mid \Pi \right)\qquad \text{ a.s.}
\end{eqnarray*}
Using \eqref{example P-calibration 0}, this implies, a.s.,
\begin{equation*}
m_{\q}(\Pi) -\Pi \,=\,
m_{\q}(\Pi) -m_{\p}(\Pi) \,=\, \frac{\operatorname{Cov} \left(V,Y \mid \Pi \right)}{\E \left[V \mid \Pi \right]}
\,=\, \frac{\operatorname{Cov} \left(V, h(V) \mid \Pi \right)}{\E \left[V \mid \Pi \right]}.
\end{equation*}
Consequently, if the last correlation term is different from zero with positive probability, we cannot have $\q$-calibration. This is the case if $\E[Vr(\Pi, V)\mid \Pi]\neq 0$.
We provide an example. Assume that there exists $v_0>0$ and $\pi_0>0$ such that for $\p_{\Pi}$ almost every $\pi$
\begin{equation}\label{exposure assumption}
  \left\{
    \begin{array}{ll}
 V|_{\Pi=\pi} \text{ is non degenerate and supported in }(0,v_0] &\text{ for }\pi \le \pi_0,\\
      V|_{\Pi=\pi} \text{ is non degenerate and supported in }(v_0, \infty) &\text{ for }\pi > \pi_0.
    \end{array}
    \right.
\end{equation}
Hence, the support of the unit premium is partitioned into two parts $\Omega_0=\{\Pi \le \pi_0\}$ and $\Omega_1=\{\Pi > \pi_0\}$. On $\Omega_0$ the unit premiums provide non-degenerate exposure distributions on the interval $(0,v_0]$ and 
on $\Omega_1$ on the interval $(v_0, \infty)$. Thus, low unit premiums have low exposures, and high unit premiums high exposures.

Finally, we assume that $h$ is strictly increasing on $(0,v_0]$ and strictly decreasing on $(v_0,\infty)$. Based on these assumptions we have
\begin{equation*}
\left\{
\begin{array}{ll}
\operatorname{Cov} \left(V, h(V) \mid \Pi=\pi \right)
>0 & \text{ for $\pi \le \pi_0$,}\\
\operatorname{Cov} \left(V, h(V) \mid \Pi=\pi \right)
<0 & \text{ for $\pi > \pi_0$.}
\end{array}
\right.
\end{equation*}
Consequently,
\begin{equation*}
\left\{
\begin{array}{ll}
m_{\q}(\Pi) >m_{\p}(\Pi)=\Pi & \text{ for a.e.~$\Pi \le \pi_0$,}\\
m_{\q}(\Pi) <m_{\p}(\Pi)=\Pi & \text{ for a.e.~$\Pi > \pi_0$.}
\end{array}
\right.
\end{equation*}
Henceforth, the total losses are underestimated on small unit premiums and overestimated on large unit premiums.
This closes Part 1 of the stylized example.

\paragraph{Part 2 of the stylized example.}

For the sections on graphical and statistical methods, below, we equip the above example with explicit functions and numerical values. We assume that the unit premium follows a scaled and translated beta distribution under $\p$
\begin{equation}\label{translated beta}
\Pi \sim G_\Pi \quad \text{ with } \qquad 
\frac{\Pi - 800}{400} \sim \operatorname{Beta}(\alpha=4, \beta=4).
\end{equation}
Thus, $\Pi$ is supported in $(800, 1200)$ and its density is symmetric around $\pi_0:=1000$. The exposure \eqref{exposure assumption} is assumed to behave differently below and above this critical point $\pi_0$, namely,
\begin{equation}\label{selection of exposure distribution}
V|_{\Pi=\pi}\sim \left\{
\begin{array}{ll}
\operatorname{Uniform}(0.5,1)
& \text{ for }\pi \le \pi_0,\\
\operatorname{Uniform}(1,1.5)
& \text{ for }\pi > \pi_0.
\end{array}
\right.
\end{equation}
Next, we select the measurable function $h$ that enters \eqref{violation conditional mean independence}. We set on $(0.5, 1.5)$
\begin{equation*}
h(v) =  - 3000 \left(v-1\right)^2 ~ \ge ~-750.
\end{equation*}
This function is increasing on $I_1=(0.5, 1)$ and it is decreasing
on $I_2=(1, 1.5)$. We set
\begin{equation*}
\mu(\Pi,V)=\E\left[Y \mid \Pi , V \right] = \Pi + h(V) - \E[h(V) \mid \Pi] >0 
\qquad \text{ a.s.}
\end{equation*}
This example provides $\p$-calibration $m_{\p}(\Pi)=\Pi$, a.s., see \eqref{example P-calibration 0}, and we have miscalibration under the
$\q$-measure
\begin{equation}\label{exact miscalibration}
m_{\q}(\Pi) -\Pi \,=\,
\frac{\operatorname{Cov} \left(V,Y \mid \Pi \right)}{\E \left[V \mid \Pi \right]}
\,=\, \frac{\operatorname{Cov} \left(V, h(V) \mid \Pi \right)}{\E \left[V \mid \Pi \right]} =\left\{\begin{array}{ll}
1000/24 & \text { for $\Pi \le \pi_0$},\\
-1000/40 & \text { for $\Pi > \pi_0$}.
\end{array}\right.
\end{equation}
Lastly, we need to model the unit loss costs $Y$. We assume the conditional distribution
\begin{equation}\label{gamma responses}
Y|_{\Pi, V} \sim \operatorname{Gamma}\left(\gamma, \gamma/\mu(\Pi, V)\right),
\end{equation}
with shape parameter $\gamma>0$. This conditional distribution has expected value 
$\mu(\Pi, V)$, variance $\mu(\Pi, V)^2/\gamma$ and coefficient of variation of $1/\sqrt{\gamma}$; due to \eqref{conditional Bayes VV}, this distribution is identical under both population measures $\p$ and $\q$.

\medskip

This model specification is going to be used in the examples in the next sections.

\section{Graphical tools to assess calibration}
\label{Graphical calibration tools}

Our first goal is to assess $\q$-calibration \eqref{Q calibration} in the stylized example introduced in Section \ref{sec: stylized example} using different graphical tools. Recall that this example is $\p$-calibrated (which is not of much interest in actuarial pricing), but it is not $\q$-calibrated (our main interest), see \eqref{exact miscalibration}. Our goal is to introduce graphical tools that allow us to identify this $\q$-miscalibration.

\medskip

Our first step in Section \ref{distribution G} is to present the distribution of the unit premium $\Pi$ both under the $\p$-measure and the $\q$-measure. As will be seen below, the latter distribution under $\q$ is important in an empirical set-up. 

\medskip

The remainder of this section is then divided into two parts:

\medskip

\underline{Part 1 -- {\it Population Version}}: In Section \ref{Graphical tools: distributional versions}, we assume that the above (true) data generating model is known. This allows us to study the calibration question under the ground truth. Of course, this is unrealistic in practice, but it helps us to shape ideas and to introduce a clean notation.

\medskip

\underline{Part 2 -- {\it Sample Version}}: In Section \ref{Graphical tools: empirical versions}, we work under an unknown population model specification, and we try to answer the calibration question empirically from an observed sample. For this second step, we assume to have an i.i.d.~{\it test sample} ${\cal T}=(Y_i,V_i,\Pi_i)_{i=1}^n$ following the same law as $(Y,V,\Pi)$. Test sample means that we perform a proper out-of-sample validation that does not consider the learning data that has been used to derive the unit premium rule $\Pi$ -- in this sense, all considerations are understood conditionally on the learning sample, which is kept fixed (a mathematically fully consistent notation would require to have a conditional notation, given the learning sample,  for notational convenience we do not do that). 

\subsection{Unit premium distributions under $\p$ and $\q$}
\label{distribution G}

The unit premium $\Pi \sim G_\Pi$ has a scaled and translated beta distribution under the population measure $\p$, see \eqref{translated beta}. Figure \ref{Plot Pi distribution G} (lhs) shows this distribution $\Pi \sim G_\Pi$. Additionally, we illustrate the deciles of $G_{\Pi}$ by the dotted lines. These are obtained by selecting $K=10$ and setting
\begin{equation*}
G_\Pi^{-1}(k/K) \qquad \text{ for $k\in \{0,\ldots, K\}$,}
\end{equation*}
where $G_\Pi^{-1}$ is the generalized left-continuous inverse of $G_\Pi$ given by
\begin{equation*}
G_\Pi^{-1}(u) = \inf \left\{ y;~ G_\Pi(y) \ge u \right\}
\qquad \text{ for $u \in [0,1]$.}
\end{equation*}
Consequently, each interval
$$I_k = \left(G_\Pi^{-1}\left(\frac{k-1}{K}\right), G_\Pi^{-1}\left(\frac{k}{K}\right)\right]
$$
has an equal probability $\p(\Pi \in I_k)=1/K$ for $k \in \{1,\ldots, K\}$; note that $G_{\Pi}$ is absolutely continuous in our case. 

\begin{figure}[htb!]
\begin{center}
\begin{minipage}[t]{0.45\textwidth}
\begin{center}
\includegraphics[width=\textwidth]{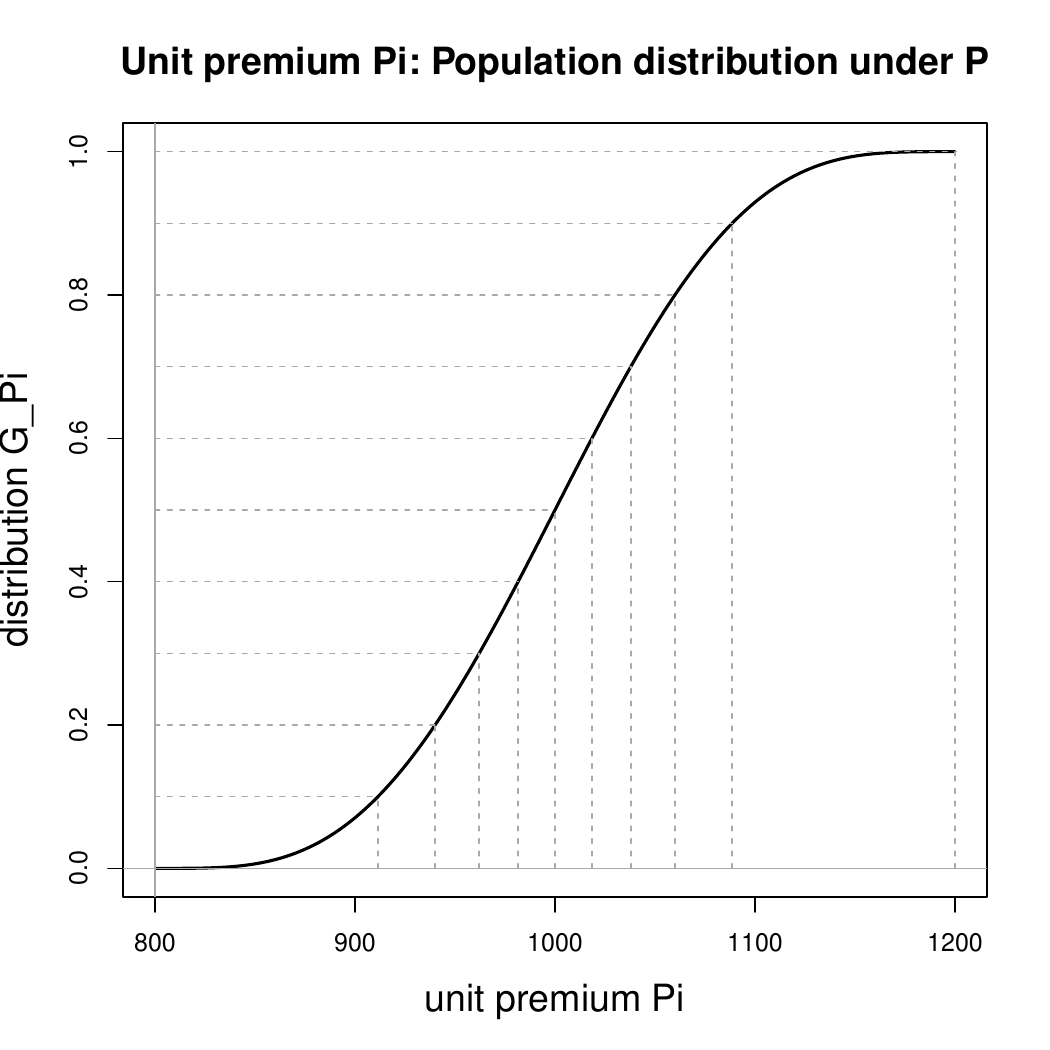}
\end{center}
\end{minipage}
\begin{minipage}[t]{0.45\textwidth}
\begin{center}
\includegraphics[width=\textwidth]{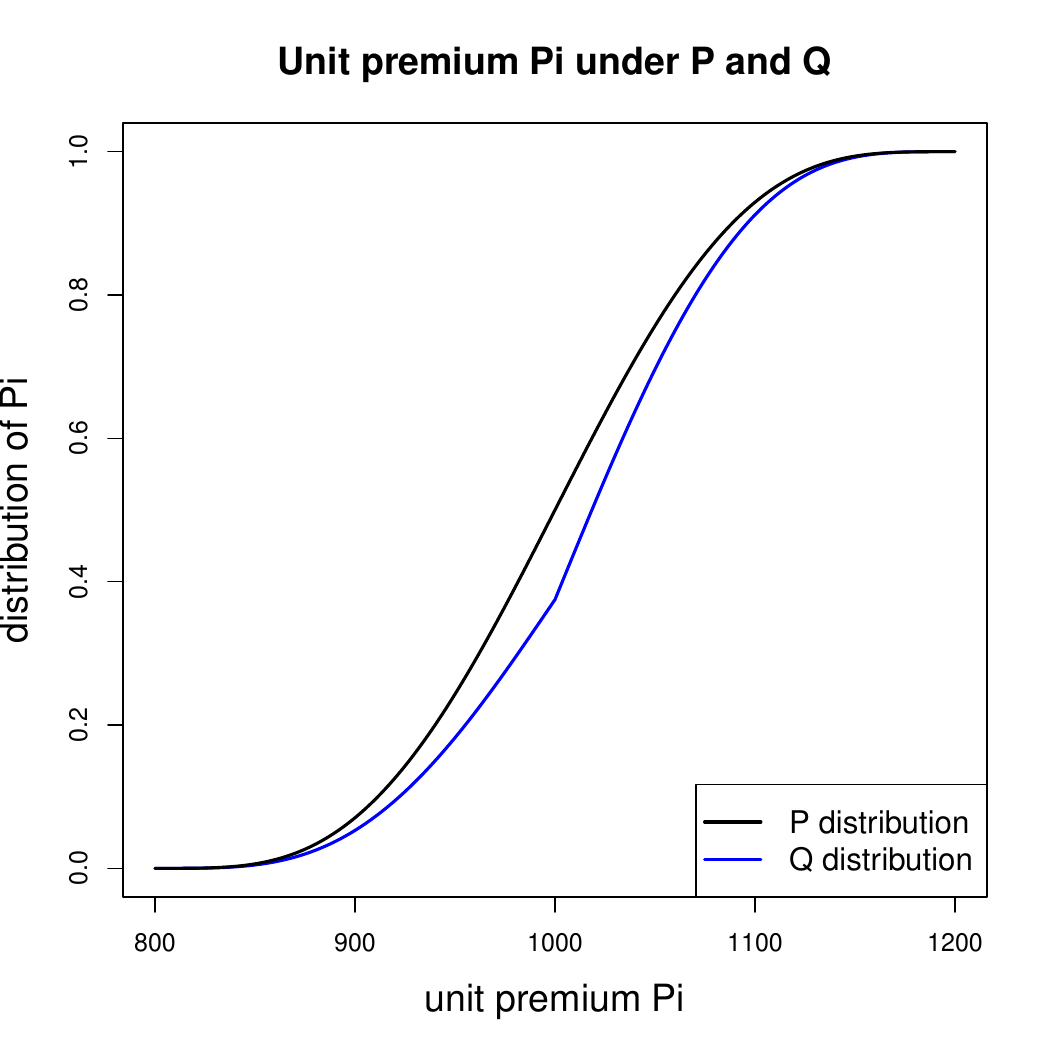}
\end{center}
\end{minipage}
\end{center}
\vspace{-.7cm}
\caption{(lhs) Unit premium distribution $\Pi \sim G_\Pi$ under population distribution $\p$, and (rhs) comparison of $G_\pi$ and $G_\pi^{\q}$.}
\label{Plot Pi distribution G}
\end{figure}

For parameter estimation, we prefer intervals that contain roughly equal aggregated exposures. Intuitively, this means that in all intervals there is roughly an equal amount of information available for parameter estimation (this statement would need additional assumptions to be made precise). Consequently, we do not want intervals $I_k$ with equal probabilities for $\Pi$, but rather with equi-exposures of $V$ in all intervals.
This motivates the unit premium distribution under the $\q$-measure
\begin{equation}\label{exposure-weighted Pi distribution}
G^{\q}_{\Pi}(\pi) = \q \left( \Pi \le \pi \right)
= \frac{1}{\E[V]}\, \E \left[ V \1_{\{\Pi \le \pi \}}\right] \qquad \text{ for $\pi \in \R$,}
\end{equation}
with generalized left-continuous inverse
\begin{equation*}
(G^{\q}_\Pi)^{-1}(u) = \inf \left\{ y;~ G^{\q}_\Pi(y) \ge u \right\}
\qquad \text{ for $u \in [0,1]$.}
\end{equation*}
The new intervals for $k \in \{1,\ldots, K\}$ under the exposure-weighted measure $\q$ are given by
$$I^{\q}_k = \left((G^{\q}_\Pi)^{-1}\left(\frac{k-1}{K}\right), (G^{\q}_\Pi)^{-1}\left(\frac{k}{K}\right)\right].
$$
Under this $\q$-measure set-up, we can compute the average exposure in each interval $I^{\q}_k$
\begin{equation}\label{decile binning formula}
1/K ~\approx~ 
\q\left( \Pi \in I^{\q}_k\right)
~=~\frac{1}{\E[V]}\, 
\E \left[ V \1_{\{ \Pi \in I^{\q}_k\}}\right].
\end{equation} 
The approximation is exact for continuous distributions.
Thus, we need to work under the exposure-weighted $\q$-measure framework for ensuring that all intervals have the same average observed exposure under equi-spaced quantile levels.
Figure \ref{Plot Pi distribution G} (rhs) shows the differences of the distributions $G_\Pi$ and $G^{\q}_\Pi$
 of $\Pi$ under $\p$ and $\q$.

\medskip

\begin{center}
\noindent\fbox{%
  \begin{minipage}{0.9\textwidth}
    ~\\
We generally work under the exposure-weighted population measure $\q$, and the subsequent derivations generally use $G^{\q}_\Pi$ to ensure equi-exposures by quantile binning, see \eqref{decile binning formula}.
\\~    
  \end{minipage}%
}
\end{center}

\medskip

\subsection{Graphical tools: Population Version}
\label{Graphical tools: distributional versions}
In this section, we assume that the true data generating model introduced in Section \ref{sec: stylized example} is known.  This {\it population version} will help us to shape ideas and to properly define all the relevant objects. In Section \ref{Graphical tools: empirical versions}, we  turn to the real-world situation of an unknown population model. This requires approximation by its empirical counterpart using the test sample ${\cal T}$.
This latter case is called the {\it sample version}.

\medskip

The first object of core interest is the $\q$-calibration property \eqref{Q calibration}. We present graphical tools to analyze it. We discuss the {\it calibration plot} and the {\it lift chart} which present the identical information, but in a slightly different structure.

\subsubsection{Calibration plot}
\label{sec: Calibration plot}

The calibration plot has many different names, e.g., in Gneiting--Resin \cite{GneitingResin} it is called {\it $T$-reliability diagram} (in our case $T$ is the mean functional), in Pohle \cite{Pohle} and W\"uthrich--Merz \cite{WM2023} it is called auto-calibration plot,
and a particular case of a sample version of the calibration plot is the {\it actual-vs-expected plot}; see Goldburd et al.~\cite{Goldburd}; we come back to the actual-vs-expected (AvE) plot in Section \ref{Actual-vs-expected plot -- quantile binning}, below.

\medskip

The {\it calibration plot} considers the graph\footnote{Function \eqref{calibration function} needs some care: $m_{\q}(\cdot)$ denotes a selected measurable version of the conditional mean $\E_{\q}[Y | \Pi]$, and calibration is then the a.s.~identity $m_{\q}(\Pi) = \Pi$; moreover, we impose additional display conventions for graphical interpolation between observed premium values.}
\begin{equation}\label{calibration function}
\pi ~\mapsto ~ m_{\q}(\pi) = \E_{\q} \left[Y \middle| \Pi=\pi \right],
\end{equation}
for $\pi$ in the convex hull of the premium range attained by $\Pi$. If this plot shows a diagonal line $m_{\q}(\pi) =\pi$, we have perfectly $\q$-calibrated unit premiums $\Pi$ for $Y$, otherwise not. 

\begin{figure}[htb!]
\begin{center}
\begin{minipage}[t]{0.45\textwidth}
\begin{center}
\includegraphics[width=\textwidth]{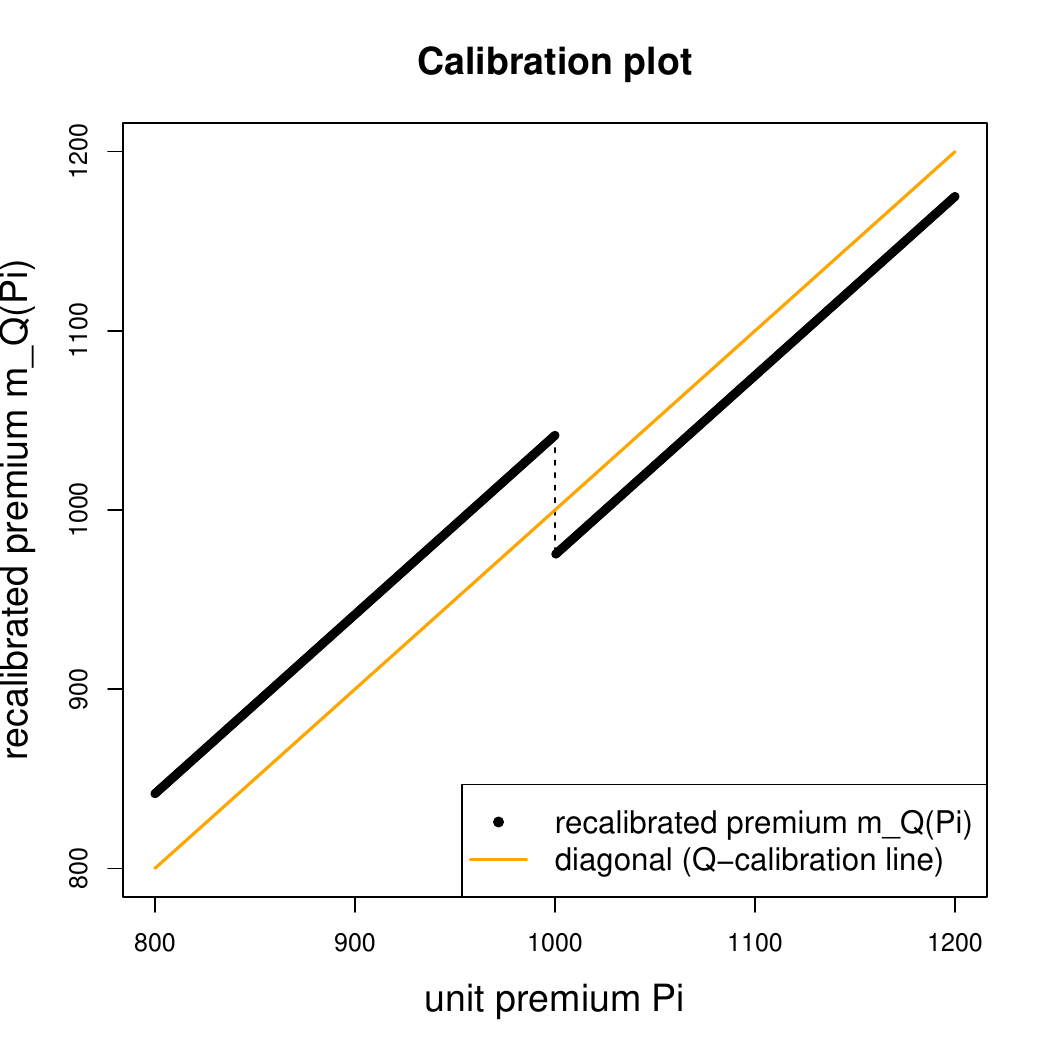}
\end{center}
\end{minipage}
\end{center}
\vspace{-.7cm}
\caption{Calibration plot \eqref{calibration function} under the exposure-weighted $\q$-measure; the orange diagonal line corresponds to perfect $\q$-calibration.}
\label{Calibration Plot 1}
\end{figure}

Figure \ref{Calibration Plot 1} shows the results of the recalibration step \eqref{calibration function} under $\q$, which in our modeling set-up results in the exact formula 
\eqref{exact miscalibration}. The black line shows the 
$\q$-recalibrated premiums $m_{\q}(\pi)$ in the unit premium range
$(800,1200)$, and the orange line corresponds to the diagonal. We observe that there is $\q$-miscalibration: for $\pi \le \pi_0=1000$, the expected unit loss costs are underestimated, for $\pi > \pi_0$, they are overestimated under $\q$. Thus, in this example we do not have $\q$-calibration  \eqref{Q calibration} because, a.s.,
\begin{eqnarray*}
\E \left[ Z \mid \Pi \right] > \E\left[P \mid \Pi \right]&&\text{ for $\Pi \le \pi_0$},\\
\E \left[ Z \mid \Pi \right] < \E\left[P \mid \Pi \right]&&\text{ for $\Pi > \pi_0$}.
\end{eqnarray*}
This leads to a systematic cross-subsidy from high-unit premium policies to low-unit premium policies. Note that $m_{\q}(\Pi)$ is $\q$-calibrated, see Remark \ref{remark qq calibrated}.

\begin{remark}[Risk ranking]\normalfont
\label{wrong risk ranking}
Figure \ref{Calibration Plot 1} shows a situation where $\q$-calibration of $\Pi$ for $Y$ fails to hold. In fact, this failure is not only on the level of the fitted unit premiums $\Pi$, but it also provides a wrong risk ranking: for $\Pi$ in a small neighborhood around $\pi_0=1000$, the risk ranking from the $x$-axis and the $y$-axis differ, e.g., $m_{\q}(\pi_0)>m_{\q}(\pi_0+\varepsilon)$ for small $\varepsilon >0$. Risk rankings will be assessed by Gini scores in Section \ref{section Gini score}, below.
\end{remark}

\subsubsection{Lift chart}
The {\it lift chart} shows the same statistics as the calibration plot, but it uses a different scale on the $x$-axis that is based on quantile levels. Sample versions of lift charts have been considered, e.g., in Goldburd et al.~\cite{Goldburd}, we come back to this in Section \ref{(Empirical) lift chart}, below.

\medskip

At the population level, the lift chart considers the two functions
\begin{equation}\label{lift chart 0}
u \in (0,1) \mapsto  
\left\{
\begin{array}{l}
m_{\q}\left((G^{\q}_\Pi)^{-1}(u)\right) = \E_{\q} \left[Y \middle| \Pi=(G^{\q}_\Pi)^{-1}(u) \right],
\\
 (G^{\q}_\Pi)^{-1}(u).
\end{array}
\right.
\end{equation}
The first function gives the $\q$-recalibrated unit premium, and the second one the unit premium at the quantile levels $u \in (0,1)$. If these two functions are identical, we have $\q$-calibration \eqref{Q calibration}.

\begin{figure}[htb!]
\begin{center}
\begin{minipage}[t]{0.45\textwidth}
\begin{center}
\includegraphics[width=\textwidth]{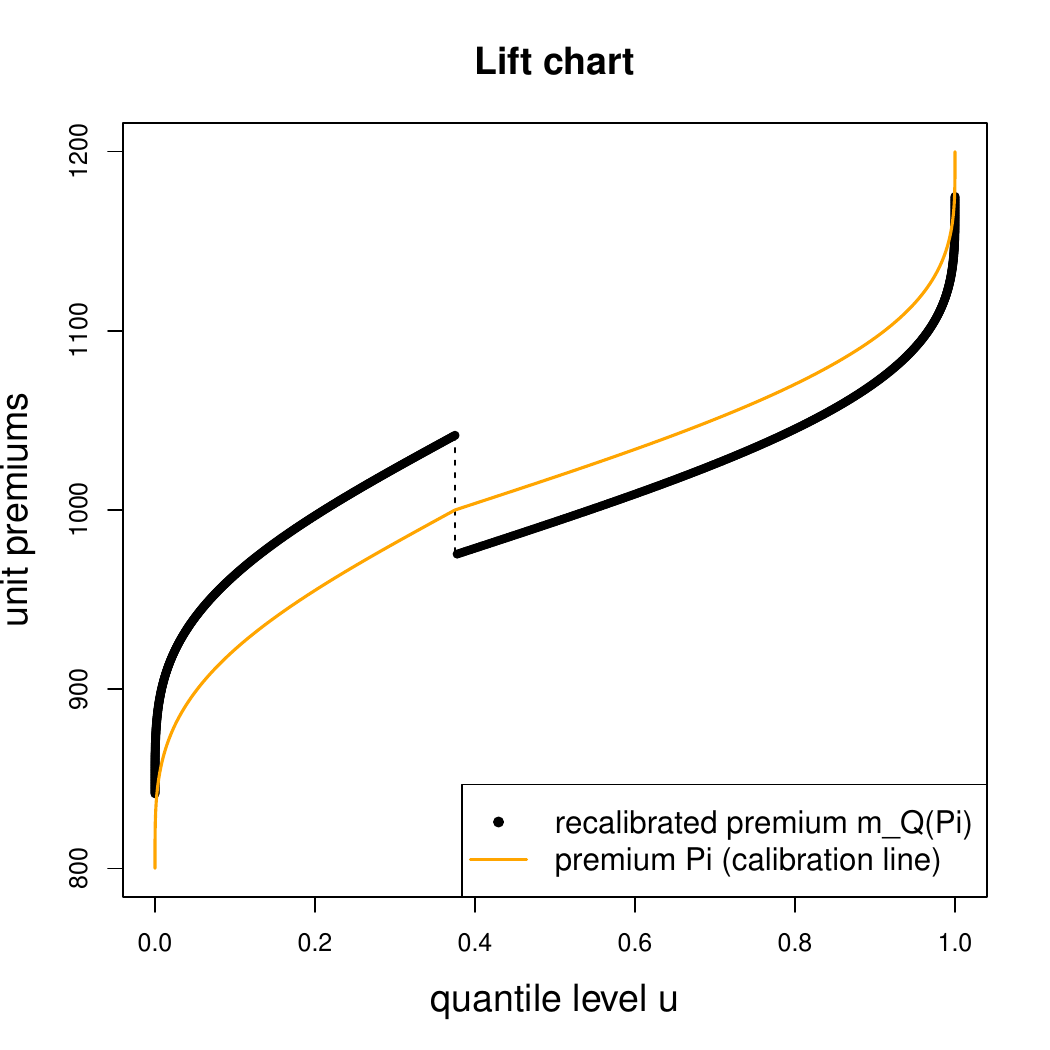}
\end{center}
\end{minipage}
\end{center}
\vspace{-.7cm}
\caption{Lift chart considering the two graphs \eqref{lift chart 0}.}
\label{Lift chart 1}
\end{figure}

Figure \eqref{Lift chart 1} provides the lift chart at the population level, the black curve gives the first function in \eqref{lift chart 0} and the orange line corresponds to the second function in \eqref{lift chart 0} highlighting the calibration line. This lift chart is identical to the calibration plot of Figure \ref{Calibration Plot 1}, we only changed the $x$-axis from the observation scale $\pi$ to the quantile level scale $u=G^{\q}_\Pi(\pi)$.

\medskip

The term {\it lift} refers to the fact that the wider the scale on the $y$-axis, the bigger the lift from the lowest to the biggest unit premiums. Thus, the lift corresponds to {\it discrimination} (studied in Section \ref{sec: Murphy diagram}, below) regardless whether this lift is justified by the calibration consideration of $\Pi$ for $Y$ or not. This lift is then interpreted as the premium discrimination between lowest and highest risk profiles (assigned by $\Pi$). If we use the global mean premium $\Pi=\E_{\q}[Y]$ (egalitarian price), there is calibration but there is no discrimination because every policyholder is charged the identical unit premium and the lift is zero.

\medskip

Intuitively, if there are two premium schemes $(Y, V,\Pi^a)$ and $(Y,V, \Pi^b)$, then the better calibrated one for $Y$, encloses a smaller ``area between the curves'' of the black and the orange graphs in Figures \ref{Calibration Plot 1} and \ref{Lift chart 1}, respectively. 
There are some difficulties in assessing this calibration accuracy for two different premium schemes:
\begin{itemize}
\item The two Figures \ref{Calibration Plot 1} and \ref{Lift chart 1} consider different scales on the $x$-axis, and one premium scheme may enclose a smaller area on one scale and a bigger one on the other scale. That is, there is no universality or a canonical scale for measuring the calibration error. One may though argue that in our case preference should be given to the lift chart because its scale on the $x$-axis does not depend on the selected premium scheme.
\item The graphs in Figures \ref{Calibration Plot 1} and \ref{Lift chart 1} require that we can compute the recalibrated premium $m_{\q}(\Pi)$. In most applications, this is not the case and this recalibration step needs to be estimated from observations. Therefore we need to turn from this clean mathematical formulation to empirical sample versions. This is the topic of the next section.
\item The ``area between the curves'' on its own is not a qualitative criteria for prediction accuracy, it only evaluates calibration, but not discrimination. E.g., we can have two pricing principles $\Pi^a=Y+c$ for a constant $c>0$ and 
$\Pi^b=\E_{\q}[Y]$. The latter is $\q$-calibrated, the former not, but the reader will certainly agree that the former one provides a better predictor for $Y$, if $c$ is small because it is the true claim $Y$ slightly shifted by $c$, i.e., it has the correct resolution but it is not fully calibrated; we come back to this discussion in Section \ref{sec: Murphy's decomposition}, below.
\end{itemize}

\subsection{Graphical tools: Sample Version}
\label{Graphical tools: empirical versions}
The previous graphical results were based on the knowledge of the (true) population model of $(Y,V,\Pi)$. In most applied situations, this population model is unknown and we need to resolve the calibration question from a sample version. This naturally involves noise, also called irreducible risk. Assume we have an i.i.d.~test sample
${\cal T}=(Y_i,V_i,\Pi_i)_{i=1}^n$ following the same law as $(Y,V,\Pi)$. We call ${\cal T}$ a test sample because it should be independent of the learning sample that has been used to find the unit premium rule $\Pi$.

The calibration plot considers the graph \eqref{calibration function}. This has been illustrated in Figure \eqref{Calibration Plot 1} under the knowledge of the population model. In absence of this knowledge we approximate it empirically by using the test sample ${\cal T}$. The most crude version is to replace the conditional means
$m_{\q}(\pi)=\E_{\q} \left[Y \middle| \Pi=\pi \right]$ empirically by the observations $Y_i | \Pi_i=\pi$. However, irreducible risk (noise) on the response scale $Y$ makes the resulting plot not very expressive. 

\begin{figure}[htb!]
\begin{center}
\begin{minipage}[t]{0.45\textwidth}
\begin{center}
\includegraphics[width=\textwidth]{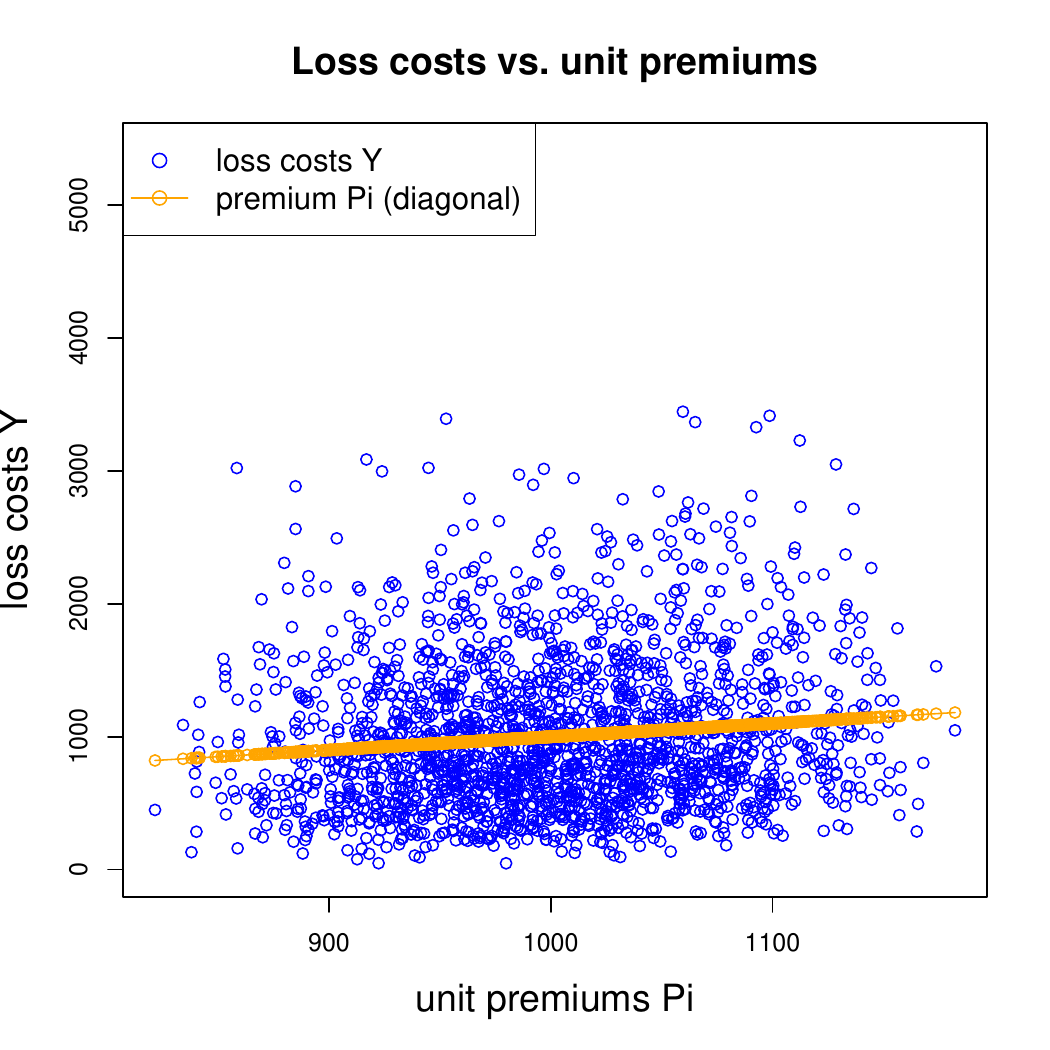}
\end{center}
\end{minipage}
\end{center}
\vspace{-.7cm}
\caption{Loss costs $Y_i$ plotted against unit premiums $\Pi_i$ on test sample ${\cal T}$.}
\label{ActualvsPredict 1}
\end{figure}

Figure \ref{ActualvsPredict 1} shows the loss costs $Y_i$ plotted against the unit premiums $\Pi_i$ on the test sample ${\cal T}=(Y_i,V_i,\Pi_i)_{i=1}^n$.
Basically, we replace the recalibrated unit premiums $m_{\q}(\Pi_i)$ -- black dots in Figure
\ref{Calibration Plot 1} -- by the observed responses $Y_i$ -- blue dots in
Figure \ref{ActualvsPredict 1}. We observe that this latter sample plot is dominated by the irreducible risk in $Y_i$, given $\Pi_i$, and it is hard to judge whether we have calibration of $\Pi$ for $Y$, or not.

\subsubsection{Actual-vs-expected plot -- quantile binning}
\label{Actual-vs-expected plot -- quantile binning}

In order to reduce the irreducible risk in Figure \ref{ActualvsPredict 1}, we need to aggregate to benefit from the law of large numbers. This can be achieved by quantile binning on the unit premium scale.
For this, we estimate the exposure-weighted unit premium $G^{\q}_{\Pi}$ distribution by
\begin{equation}\label{empirical distribution 2}
\widehat{G}^{\q}_\Pi (\pi) = \frac{1}{\sum_{i=1}^n V_i}\, \sum_{i=1}^n V_i\,\1_{\{\Pi_i \le \pi \}}.
\end{equation}
This approximates $G^{\q}_{\Pi}(\pi)$ pointwise in $\pi$, a.s., as the sample size $n \to \infty$ (by the law of large numbers). In contrast to the classical empirical distribution, we select the step sizes in \eqref{empirical distribution 2} by the exposures $(V_i)_{i=1}^n$. This is the sample version of the unit premium distribution under $\q$.

For quantile binning, select the number of bins $K\in \N$. This yields the equi-distributed quantile binning of the test sample under the empirical exposure-weighted unit premium distribution
\begin{equation}\label{Q bins}
\widehat{\cal I}^{\q}_k = \left\{ i \in \{1,\ldots, n\} \,\middle|\, 
(\widehat{G}^{\q}_\Pi)^{-1}((k-1)/K)< \Pi_i \le (\widehat{G}^{\q}_\Pi)^{-1}(k/K)\right\},
\end{equation}
for $k=1,\ldots, K$. As indicated in \eqref{decile binning formula}, every bin will contain approximately the same aggregated exposure
\begin{equation*}
\frac{V^+_k}{\sum_{i=1}^n V_i}  \,:=\,  \frac{1}{\sum_{i=1}^n V_i}\,\sum_{i \in \widehat{\cal I}^{\q}_k} V_i ~\approx~1/K.
\end{equation*}
This is verified in Figure \ref{Exposure per QBin 0}, there is a total exposure of $100,000$, and applying decile binning, $K=10$, each of the folds has a total exposure of approximately $10,000$.

\begin{figure}[htb!]
\begin{center}
\begin{minipage}[t]{0.45\textwidth}
\begin{center}
\includegraphics[width=\textwidth]{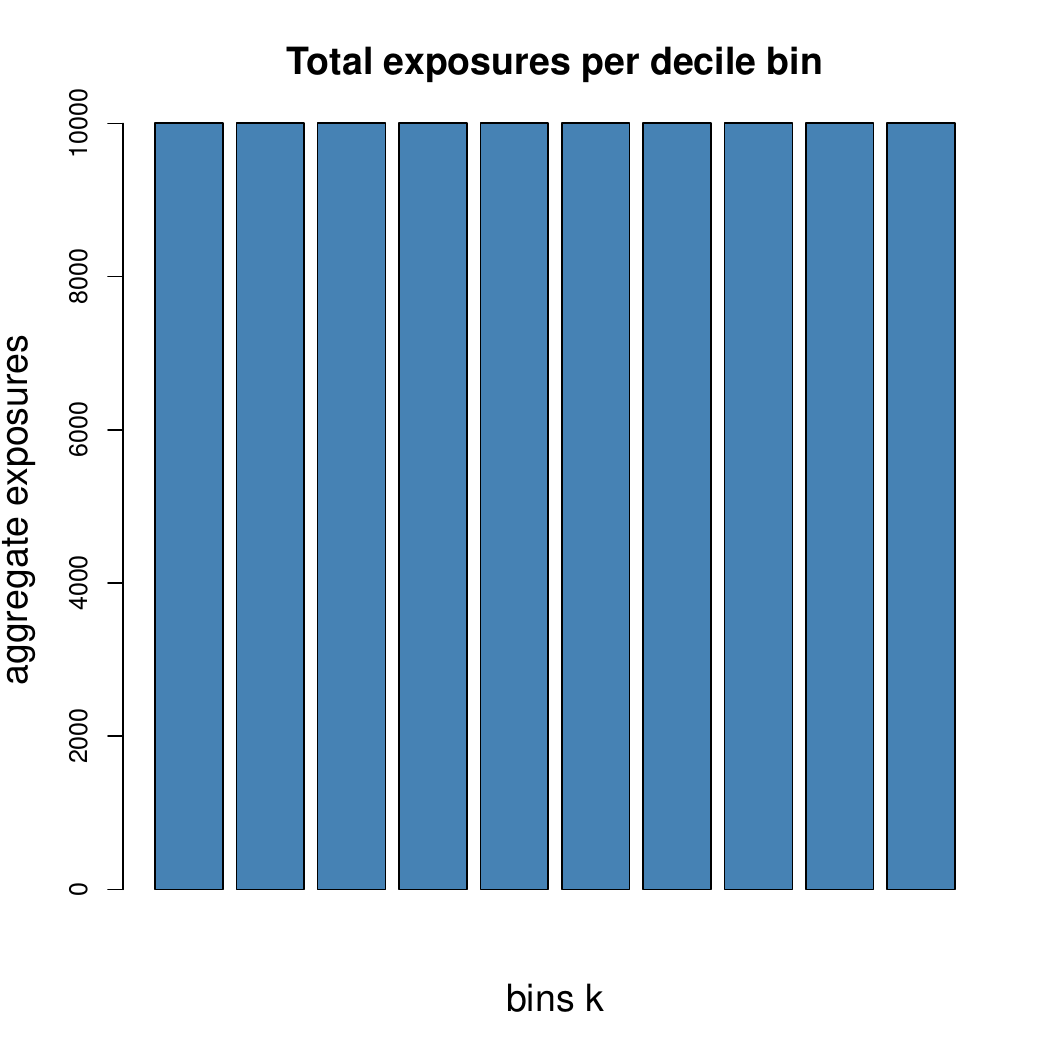}
\end{center}
\end{minipage}
\end{center}
\vspace{-.7cm}
\caption{Total exposures in each bin $\widehat{\cal I}^{\q}_k$ for decile binning $K=10$ under $\widehat{G}^{\q}_\Pi$.}
\label{Exposure per QBin 0}
\end{figure}

The idea now is to build the weighted bin averages. For the observed loss costs, we have weighted sample mean in each bin
\begin{equation}\label{binning procedure 1}
\overline{Y}_k  \,=\,  \frac{1}{V_k^+}\, \sum_{i \in \widehat{\cal I}^{\q}_k} V_iY_i \,= \,\frac{1}{V_k^+}\,\sum_{i \in \widehat{\cal I}^{\q}_k} Z_i.
\end{equation}
Equivalently, we compute the exposure-weighted average unit premium in each bin
\begin{equation}\label{binning procedure 2}
\overline{\Pi}_k   \,=\,  \frac{1}{V_k^+}\, \sum_{i \in \widehat{\cal I}^{\q}_k} V_i\Pi_i \,= \,\frac{1}{V_k^+}\,\sum_{i \in \widehat{\cal I}^{\q}_k} P_i.
\end{equation}
This weighted version relates to the $\q$-measure introduced in
Section \ref{sec: Q-calibration}, by considering the corresponding exposure-weighted quantities.

\begin{figure}[htb!]
\begin{center}
\begin{minipage}[t]{0.45\textwidth}
\begin{center}
\includegraphics[width=\textwidth]{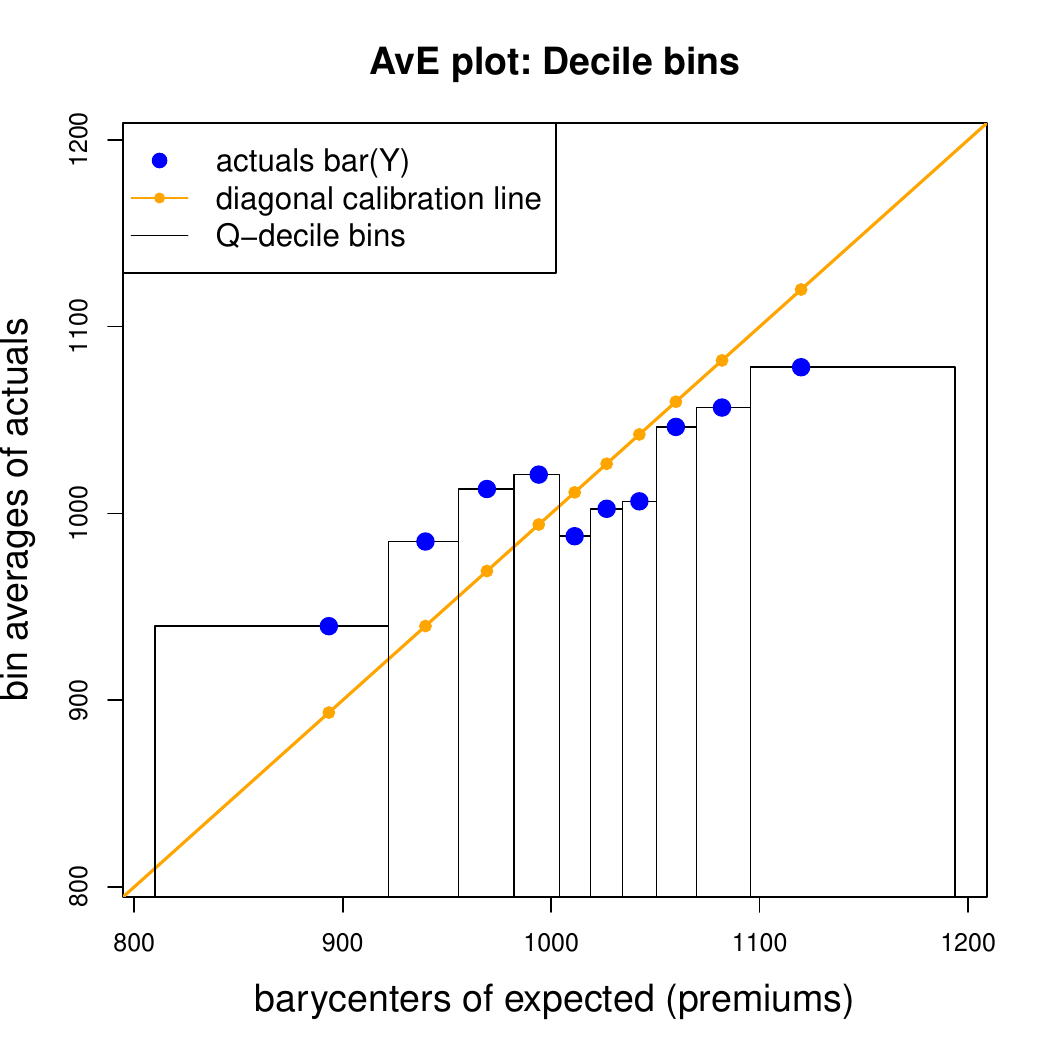}
\end{center}
\end{minipage}
\begin{minipage}[t]{0.45\textwidth}
\begin{center}
\includegraphics[width=\textwidth]{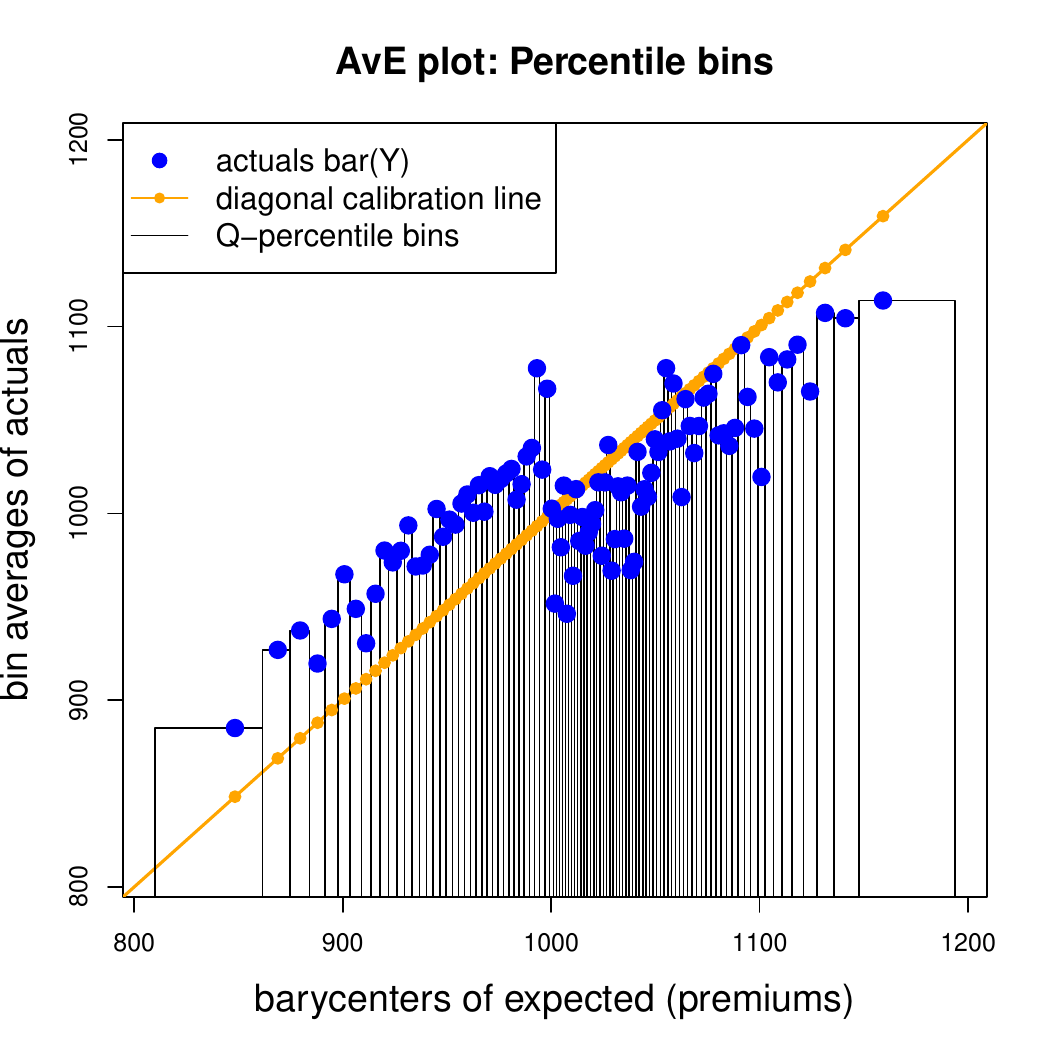}
\end{center}
\end{minipage}
\end{center}
\vspace{-.7cm}
\caption{AvE plot using decile binning $K=10$ (lhs) and percentile binning $K=100$ (rhs).}
\label{Exposure per QBin}
\end{figure}

\medskip
\begin{graybox}
The actual-vs-expected (AvE) plot considers the following sample version of the calibration plot
\begin{equation}\label{AvE plot definition 1}
\left(\overline{\Pi}_k, \,\overline{Y}_k\right) 
\quad {and} \quad
\left(\overline{\Pi}_k,\, \overline{\Pi}_k\right) 
 \qquad \text{ for $k=1, \ldots, K$.}
\end{equation}
\end{graybox}
\medskip

The AvE plot is also called actual-vs-predicted plot.

\medskip

Figure \ref{Exposure per QBin} shows the sample results for decile binning $K=10$ and percentile binning $K=100$.
The blue dots in Figure \ref{Exposure per QBin} show the (recalibrated) pairs $(\overline{\Pi}_k, \,\overline{Y}_k)_{k=1}^K$, and the orange lines in these figures correspond to the diagonal calibration line reflecting the pairs $(\overline{\Pi}_k,\, \overline{\Pi}_k)_{k=1}^K$. The blue dots clearly differ from the orange diagonal line in Figure \ref{Exposure per QBin}. This questions $\q$-calibration. Naturally, this is not a statistical test or a proof of miscalibration because it only corresponds to a graphical inspection that involves noise (irreducible risk). The noise is bigger for $K=100$ and smaller for $K=10$ due to the law of large numbers. A proper statistical test would involve a standard deviation estimate that quantifies the magnitude of this noise. However, from Figure \ref{Exposure per QBin} we see that the blue dots do not seem to randomly fluctuate around the orange diagonal, but there is a systematic pattern. This is a clear indication that $\q$-calibration is violated.

The black lines (rectangles) in Figure \ref{Exposure per QBin} show the quantiles $(\widehat{G}^{\q}_\Pi)^{-1}(k/K)$, $k\in \{0,\ldots, K\}$, these are impacted by the unit premium distribution (which is a beta distribution under $\p$) and by the exposure distribution. We see that there are more decile/percentile bins above the value $\pi_0$ because the exposure tends to be bigger for bigger unit premiums due to model assumption \eqref{selection of exposure distribution}.

\subsubsection{(Empirical) lift chart -- quantile binning}
\label{(Empirical) lift chart}
The sample version of the lift chart is then straightforward from \eqref{AvE plot definition 1}. Instead of plotting the quantile levels on the $x$-axis, we simply use the bin labels $k\in \{1,\ldots, K\}$ instead. 

\medskip
\begin{graybox}
The (sample version of the) lift chart considers the two graphs
\begin{equation}\label{empirical lift chart}
\left(k,\, \overline{Y}_k\right) 
\quad \text{and} \quad
\left(k, \,\overline{\Pi}_k\right) 
 \qquad \text{ for $k=1, \ldots, K$.}
\end{equation}
\end{graybox}
\medskip

These two graphs contain precisely the same information as \eqref{AvE plot definition 1}, but we represent the $x$-axis on a different scale (premium scale vs.~quantile level scale).

\medskip

\begin{figure}[htb!]
\begin{center}
\begin{minipage}[t]{0.45\textwidth}
\begin{center}
\includegraphics[width=\textwidth]{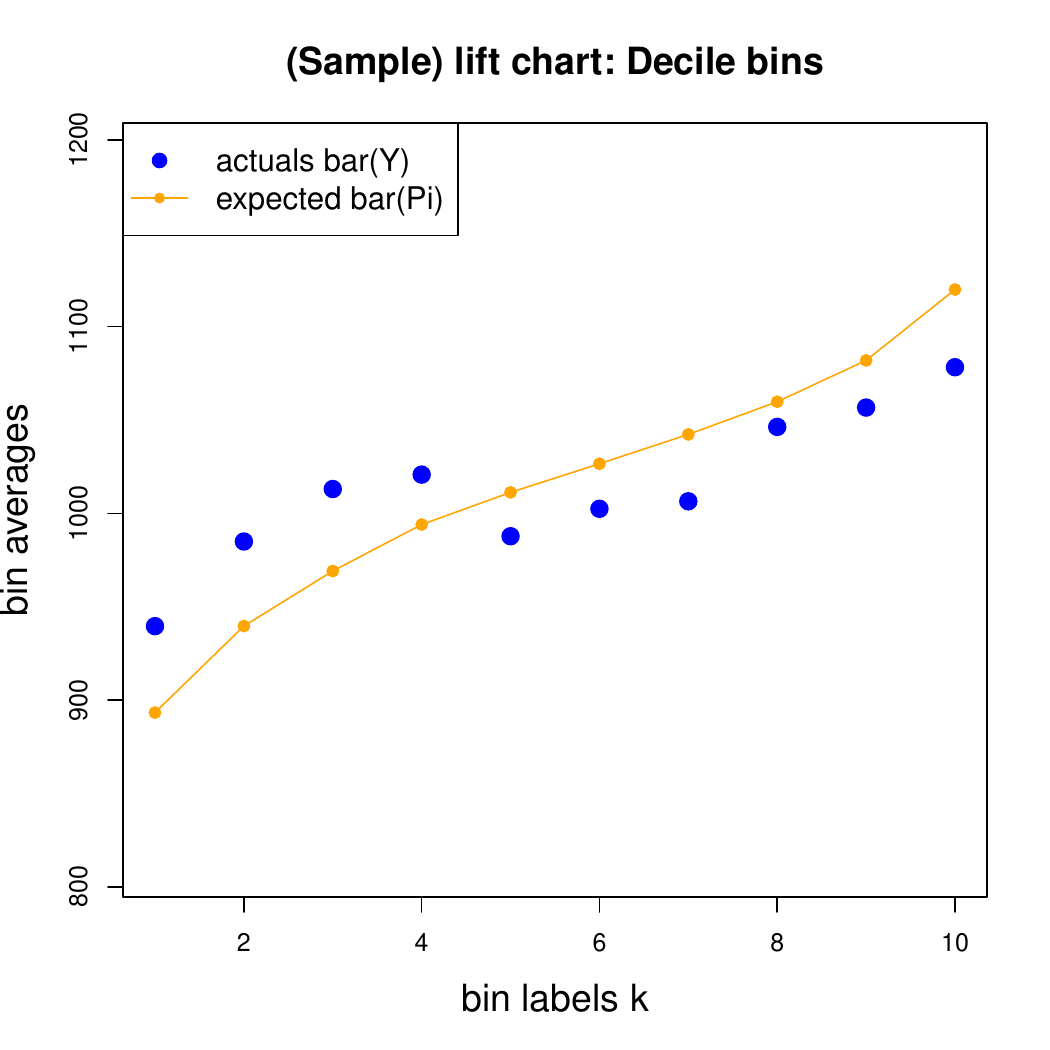}
\end{center}
\end{minipage}
\begin{minipage}[t]{0.45\textwidth}
\begin{center}
\includegraphics[width=\textwidth]{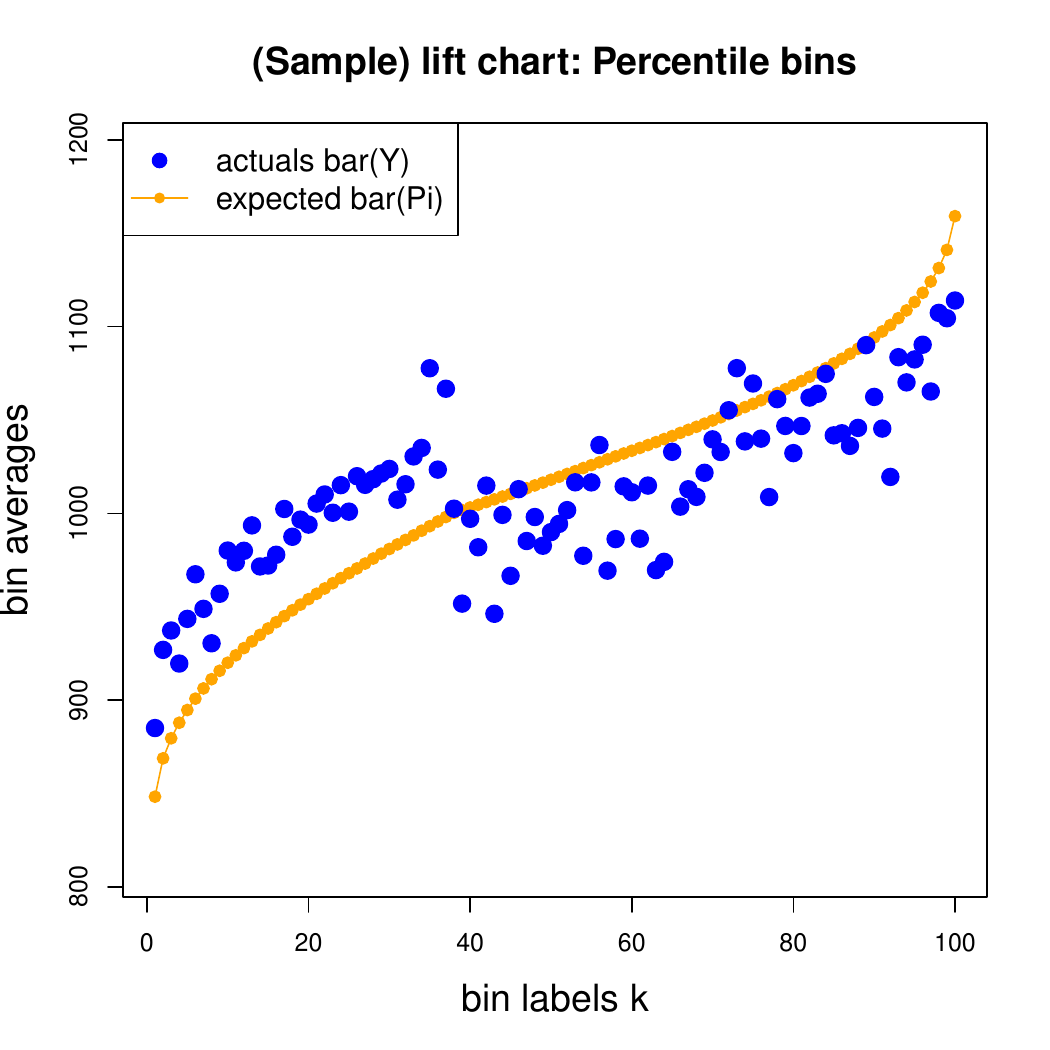}
\end{center}
\end{minipage}
\end{center}
\vspace{-.7cm}
\caption{Sample versions of the lift chart \eqref{empirical lift chart} using decile binning $K=10$ (lhs) and percentile binning $K=100$ (rhs).}
\label{LiftChart Bins 1}
\end{figure}

Figure \ref{LiftChart Bins 1} shows the lift chart \eqref{empirical lift chart} from decile binning and percentile binning using the bins \eqref{Q bins}. The conclusions are essentially the same as from the AvE plot in Figure \ref{Exposure per QBin}, though they may look a bit more obvious in the lift chart (underestimation for small unit premiums and overestimation for large unit premiums).

\medskip 

To turn the AvE plot and the lift chart into statistical methods, we would need to estimate the variance of (the uncertainty in) the empirical means $\overline{Y}_k$ as well as the impact of the quantile binning bounds. This would then allow us to turn the lift chart of Figure \ref{LiftChart Bins 1} into a $\chi^2$-test for the null hypothesis that we have a $\q$-calibrated model; such approaches have been considered, for example, in Gatti \cite{Gatti}, and the binary case is known as the Hosmer--Lemeshow test \cite{Hosmer}; see also Henzi et al.~\cite{Henzi}.

\subsubsection{Actual-vs-expected plot -- statistical smoothing methods}
\label{Actual-vs-expected plot -- statistical smoothing methods}
Quantile binning replaces the individual observations by a finite number of exposure-weighted
bin averages thereby benefiting from the law of large numbers.
We could also use more statistically guided methods that try to interpolate observations by regression functions. The two natural candidates are local regression of Loader \cite{Loader} and isotonic regression of 
Ayer et al.~\cite{Ayer}, Brunk et al.~\cite{Brunk}, Miles \cite{Miles},  Barlow et al.~\cite{BarlowEtAl}, Barlow--Brunk \cite{BarlowBrunk}, Kruskal \cite{Kruskal}. The pool adjacent violators (PAV) algorithm gives a fast implementation solving the isotonic regression numerically; see Leeuw et al.~\cite{Leeuw}. The local regression is a flexible smoothing method, its disadvantage is that it heavily relies on a good hyper-parameter selection. The isotonic regression is in some sense more crude, it essentially relies on the assumption of having a correct risk ranking in $\Pi$. Both methods can be problematic in applications, but they are still the best tools that are currently available.

\paragraph{Exposure-weighted Local Regression.}

This outline follows W\"uthrich et al.~\cite[Section 4.2.1]{AITools}.
For a {\it local regression}, one selects a bandwidth $\delta(\Pi)>0$ that may depend on the unit premium rule $\Pi$. This gives the smoothing window (interval)
\begin{equation*}
\Delta (\Pi) = \Big( \Pi - \delta(\Pi), \,\Pi + \delta(\Pi)\Big).
\end{equation*}
For a local regression around $\Pi$ only the instances $(Y_i,V_i, \Pi_i)$ in this smoothing window $\Pi_i \in \Delta(\Pi)$ are considered. Typically, one chooses $\delta (\Pi)$ such that the smoothing window $\Delta (\Pi)$ contains 10\% or 20\% of the available sample. Then, one select a weighting function $w_\Pi: \Delta(\Pi) \to \R_+$. This weighting function acts as a kernel that weighs observations closer $\Pi$ more than those that are at the boundary of the smoothing window. A popular choice is a (scaled) tricube weighting function. Finally, one select a class of splines, e.g., quadratic polynomials
\begin{equation}\label{quadratic splines}
x~\mapsto~\mu_{\boldsymbol{\vartheta}}(x;\Pi)=\vartheta_0 + \vartheta_1(x-\Pi)+\vartheta_2(x-\Pi)^2,
\end{equation}
with regression parameter $\boldsymbol{\vartheta}=(\vartheta_0,\vartheta_1,\vartheta_2)^\top$.
This motivates the local regression problem
\begin{equation*}
\widehat{\boldsymbol{\vartheta}}^{\Pi} = \underset{\boldsymbol{\vartheta} \in \R^3}{\arg\min}\,
\sum_{i=1}^n V_i\, \1_{\{\Pi_i \in \Delta(\Pi)\}}\,
w_\Pi\left(\Pi_i\right) 
\Big( Y_i - \mu_{\boldsymbol{\vartheta}}(\Pi_i;\Pi)\Big)^2.
\end{equation*}
The fitted local regression value in $\Pi$ is then obtained by setting
\begin{equation*}
\widehat{\mu}^{\rm loc}(\Pi)= \mu_{\widehat{\boldsymbol{\vartheta}}^{\Pi}}(\Pi;\Pi)= \widehat{\vartheta}^{\Pi}_0.
\end{equation*}
This is the local regression method as implemented in the {\sf R} package {\tt locfit} of Loader \cite{Loader}. The hyper-parameters involved are the bandwidth $\delta(\Pi)$ (usually a nearest neighbor fraction around $\Pi$), the weighting function $w_\Pi$ (usually tricube) and the splines, usually the step function (local constant fit), linear, quadratic or cubic polynomials.

The bandwidth controls the bias-variance trade-off. A small $\delta(\Pi)$ follows local variation closely but may produce noisy curves, whereas a large $\delta(\Pi)$ considers more information which results in a better smoothing, which however is more prone to local miscalibration.

\begin{figure}[htb!]
\begin{center}
\begin{minipage}[t]{0.45\textwidth}
\begin{center}
\includegraphics[width=\textwidth]{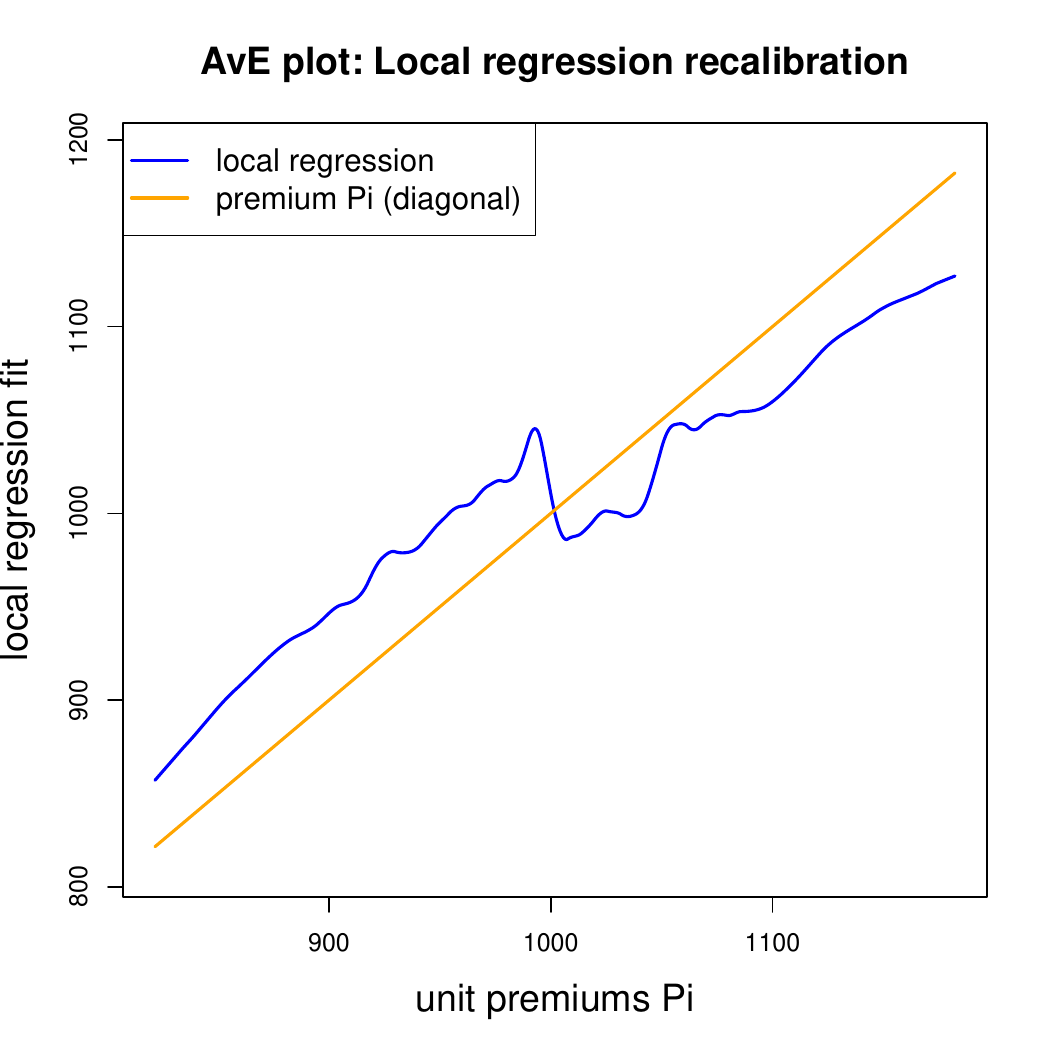}
\end{center}
\end{minipage}
\begin{minipage}[t]{0.45\textwidth}
\begin{center}
\includegraphics[width=\textwidth]{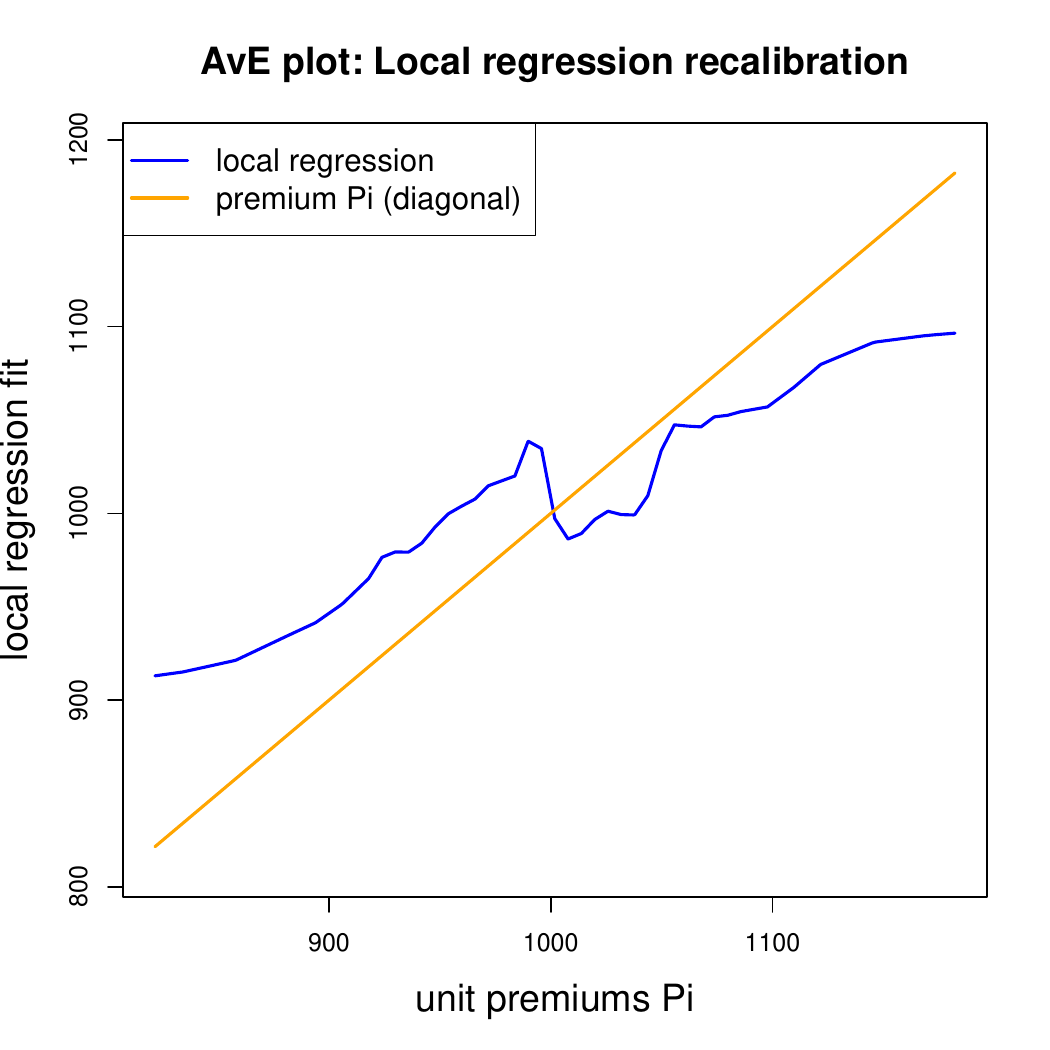}
\end{center}
\end{minipage}
\end{center}
\vspace{-.7cm}
\caption{AvE plot using local regression smoothing with (lhs) quadratic polynomials and (rhs) step function spline, both plots use a nearest neighbor fraction of 10\%.}
\label{Local regression 1}
\end{figure}

Figure \ref{Local regression 1} shows the local regression smoothed 
AvE plots. We use a nearest neighbor fraction of 10\% for the smoothing windows $\Delta(\Pi)$, the tricube kernel for $w_\Pi$, and the left-hand side uses  the quadratic regression splines \eqref{quadratic splines} and the right-hand side step functions. Note that the nearest neighbor fraction does not account for the exposures, but it is selected under the population measure $\p$ -- this is implied by the available software package and naturally we would prefer to have a $\q$-measure option. The step function spline can be seen as a rolling window approach of the quantile binning, though, still using the tricube weighting in this rolling window.

From Figure \ref{Local regression 1} we observe two wiggly curves that essentially show the right miscalibration picture. However, in a more complicated miscalibration situation it is often difficult to say whether the local regression is wiggly because of the noisy data (bias-variance trade-off) or a miscalibration. Therefore, local regression should be viewed as an indicative rather than definitive diagnostics.

\paragraph{Exposure-weighted Isotonic Regression.}

A different approach assumes that the original premium rule $\Pi$ gives the correct risk ranking, so
that $m_{\q}(\Pi)$ is non-decreasing in $\Pi$ (or comonotonic in the sense of random variables). Under this ranking assumption, define the {\it isotonic estimator} by
\begin{equation}\label{isotonic regression 0}
  \widehat{\bm}_{\q}^{\rm iso} ~=~\underset{\bm \in \R^n}{\arg\min}~
  \sum_{i=1}^n V_i \left(Y_i - m_i\right)^2 \qquad \text{subject to:
  $\Pi_i \le \Pi_j$ implies  $m_i \le m_j$.}
\end{equation}
Thus, the premiums $(\Pi_i)_{i=1}^n$ provide a ranking on the instances $i\in \{1,\ldots, n\}$, and this ranking is preserved by the isotonic regression solution $\widehat{\bm}_{\q}^{\rm iso} \in \R^n$. In between these values $\widehat{\bm}_{\q}^{\rm iso} \in \R^n$, we select a step function interpolation.

The step function interpolation implies that the isotonic regression results in a binning with constant recalibrated unit premiums in the bins. In contrast to quantile binning, the amount of data in each bin is not a selected hyper-parameter, but the algorithm \eqref{isotonic regression 0} decides on the optimal bin sizes such that isotonicity w.r.t.~the unit premiums $(\Pi_i)_{i=1}^n$ is preserved. Since on each bin, $\widehat{\bm}_{\q}^{\rm iso}$ corresponds to the exposure-weighted sample mean, this automatically implies that we have sample $\q$-calibration on all bins. Thus, this method performs (a discretized version of) the recalibration step; this was used in an essential way in W\"uthrich--Ziegel \cite{WZiegel}.

The main advantages of isotonic regression over local regression are that it does not involve any hyper-parameter selection and it results in a calibrated solution that is optimally binned w.r.t.~\eqref{isotonic regression 0}. The disadvantages are that it relies on a correct risk ranking in $(\Pi_i)_{i=1}^n$, this is not satisfied in our example, and under noisy data, the number of bins can be very small resulting in a very crude step function. Moreover, isotonic regression tends to overfit at the boundary observations of the unit premium.
For more discussion, we refer to W\"uthrich et al.~\cite[Section 4.2.2]{AITools}. Generally, we recommend isotonic regression for calibration inspection, but the resulting recalibrated regression function is often too crude for insurance pricing because a low signal-to-noise ratio typically leads to a crude binning.

\medskip

\begin{figure}[htb!]
\begin{center}
\begin{minipage}[t]{0.45\textwidth}
\begin{center}
\includegraphics[width=\textwidth]{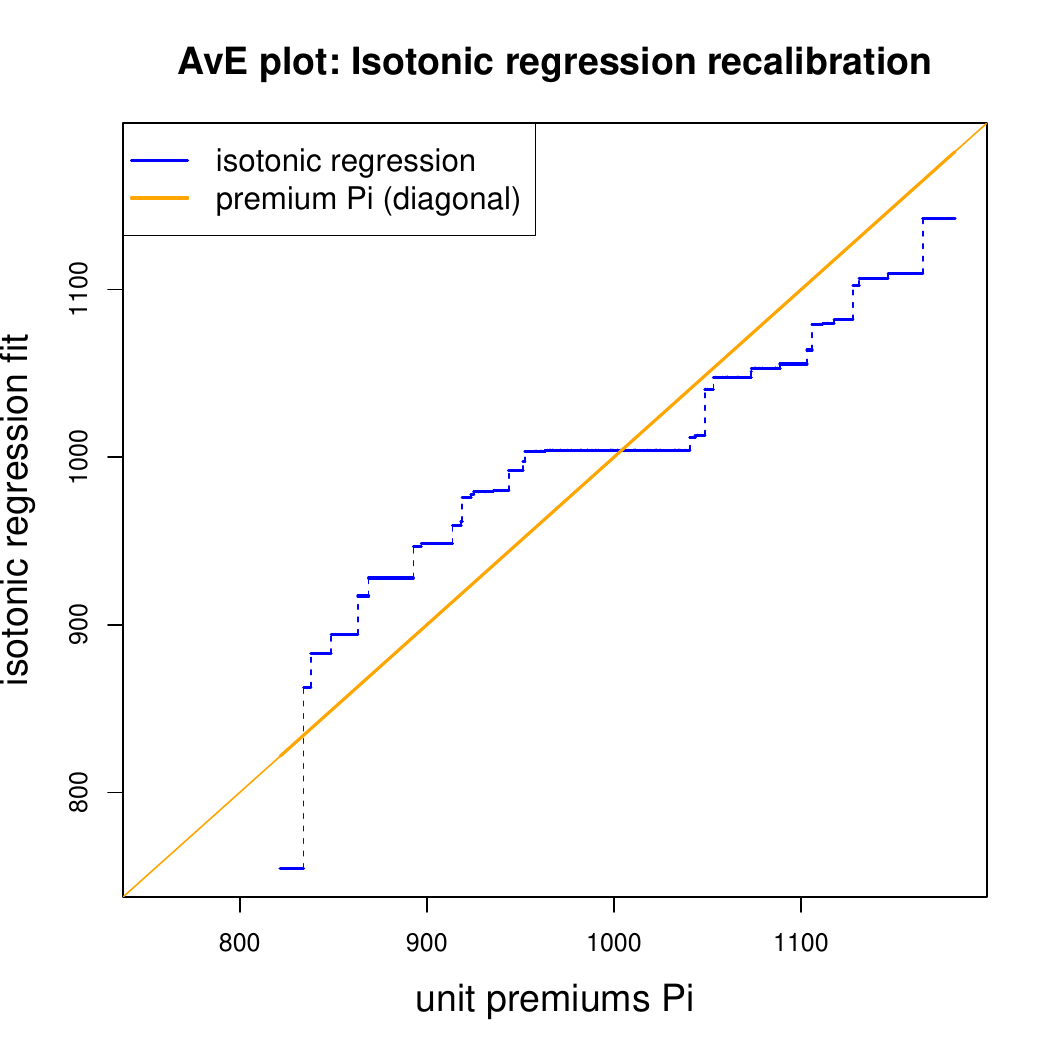}
\end{center}
\end{minipage}
\begin{minipage}[t]{0.45\textwidth}
\begin{center}
\includegraphics[width=\textwidth]{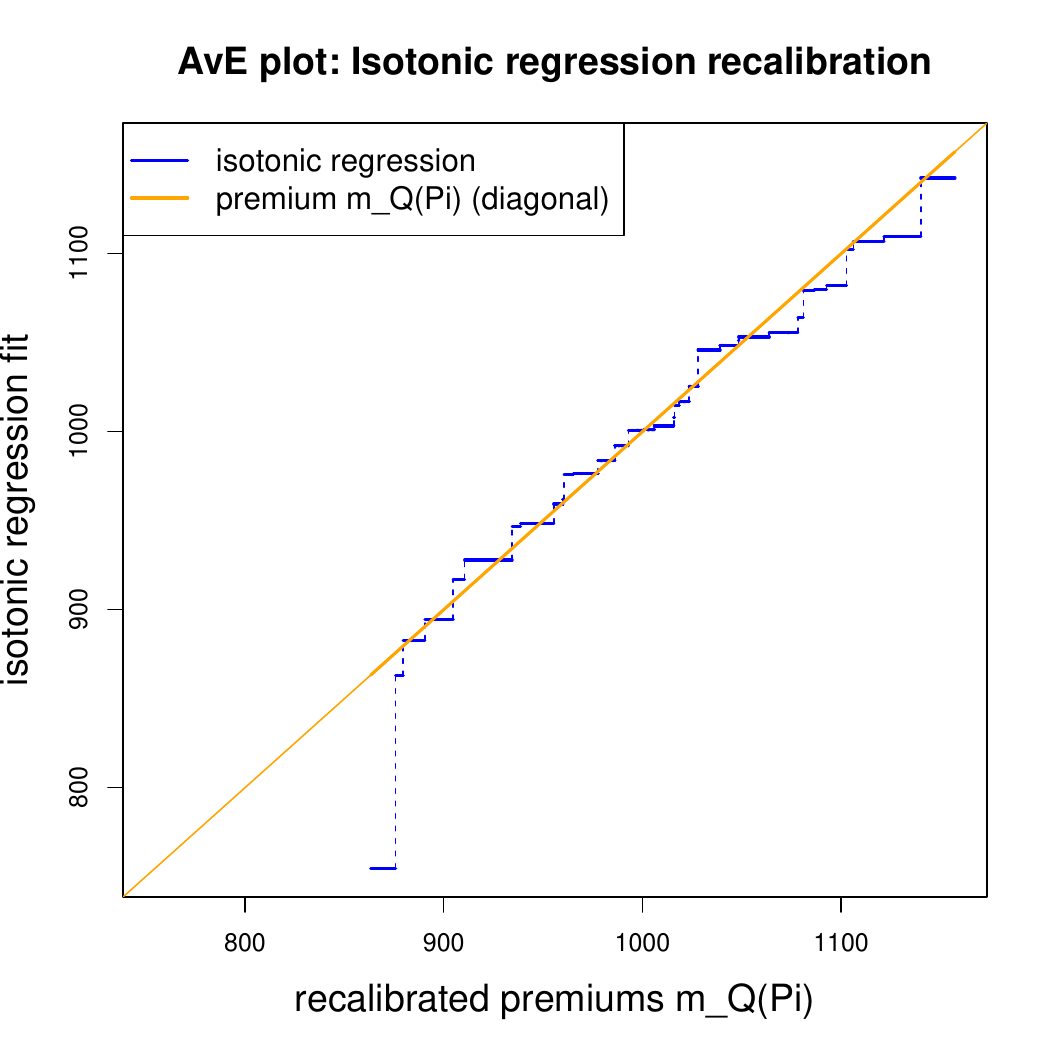}
\end{center}
\end{minipage}
\end{center}
\vspace{-.7cm}
\caption{AvE plot using isotonic regression smoothing with (lhs) risk ranking $\Pi$ and (rhs) risk ranking $m_{\q}(\Pi)$.}
\label{Isotonic regression 1}
\end{figure}

Figure \ref{Isotonic regression 1} shows that isotonic regression solution, on the left-hand side we use the risk ranking of the unit premiums $(\Pi_i)_{i=1}^n$ and on the right-hand side the (correct) risk ranking of the recalibrated premium  $(m_{\q}(\Pi_i))_{i=1}^n$. The former does not provide the correct risk ranking, see
Remark \ref{wrong risk ranking}, the latter does. Figure \ref{Isotonic regression 1} is also called CORP (consistent, optimally binned, reproducible and PAV) diagram in Dimitriadis et al.~\cite{DimitriadisCORP} and Gneiting--Resin \cite{GneitingResin}.

\medskip

We have the following observations:
\begin{itemize}
\item Figure \ref{Isotonic regression 1} (lhs) based on risk ranking $(\Pi_i)_{i=1}^n$ detects the calibration issue in $\Pi$. 
\item Figure \ref{Isotonic regression 1} (rhs) based on risk ranking $(m_{\q}(\Pi_i))_{i=1}^n$ correctly fluctuates around the orange diagonal, in fact, under infinite sample sizes the isotonic regression should coincide with this orange diagonal, i.e., $\widehat{\bm}_{\q}^{\rm iso}$ should match $m_{\q}(\Pi)$ if the isotonic regression is performed on the correct risk ranking. The difference in Figure \ref{Isotonic regression 1} (rhs) is a consequence of the noise in the responses $Y$.
\item The wrong risk ranking in $(\Pi_i)_{i=1}^n$ leads to a flat price around the center in Figure \ref{Isotonic regression 1} (lhs).
\item There is some overfitting in the tails, especially, in the lower tail. If the lowest unit premium has the smallest claim, the isotonic regression just reports this claim, i.e., no averaging takes place. For this reason, in real applications, smallest bins should be merged as well as largest bins.
\end{itemize}

As seen from Figure \ref{Isotonic regression 1}, the isotonic regression gives a natural binning that is optimal according to \eqref{isotonic regression 0}. Instead of plotting the AvE plot as in Figure \ref{Isotonic regression 1}, we could also label the bins in increasing order that results in a lift chart (with isotonically optimal bins w.r.t.~the initial premium $\Pi$). We provide such a plot in the real-data example below, see Figure \ref{MTPL isotonic regression lift chart}.

\subsubsection{Double-lift charts}
\label{sec: doublelift}

Our main focus in the previous section was on $\q$-calibration which is one part of a good predictive model. The other part is discrimination (resolution) which is going to be studied in the next section. We have briefly touched upon discrimination in the lift charts of Figures \ref{Lift chart 1} and \ref{LiftChart Bins 1} to explain the terminology {\it lift}. Naturally, if we have two pricing schemes $\Pi^a$ and $\Pi^b$, we would like to compare them w.r.t.~the lift and resolution they provide. A disadvantage of the sample version of the lift charts is that the quantile binning is done w.r.t.~the pricing schemes $\Pi^a$ and $\Pi^b$, respectively. Consequently, the resulting bins will generally not contain exactly the identical policies $i\in \{1,\ldots, n\}$, e.g., policy $i=1$ may be in the second smallest bin for $\Pi^a$ and in the third smallest one for $\Pi^b$. This is an inherent difficulty of binning approaches, and it makes a direct comparison difficult.

The double lift chart presented in Goldburd et al.~\cite{Goldburd} solves the binning problem by simultaneously considering both premium schemes $\Pi^a$ and $\Pi^b$ for the binning. Assume that all considered unit premiums are strictly positive, a.s. We analyze the ratio
\begin{equation*}
 \kappa=\frac{\Pi^b}{\Pi^a}.
\end{equation*}
Small values of $\kappa$ identify policies for which $\Pi^b$ is low relative to $\Pi^a$, whereas large values of $\kappa$ identify the reverse. We can then study the distribution $G_\kappa^{\q}$ of this ratio $\kappa$ under the exposure-weighted measure $\q$ for the quantile binning. Its sample version is given by
\begin{equation*}
 \widehat G_{\kappa}^{\q}(t)
 =\frac{1}{\sum_{i=1}^n V_i}\,
  \sum_{i=1}^n V_i\,\1_{\{\kappa_i\le t\}},
\end{equation*}
and this yields the quantile binning for $k=1,\ldots, K$
\begin{equation*}
\widehat{\cal I}^{\kappa}_k = \left\{ i \in \{1,\ldots, n\} \,\middle|\, 
(\widehat{G}^{\q}_\kappa)^{-1}((k-1)/K)< \kappa_i \le (\widehat{G}^{\q}_\kappa)^{-1}(k/K)\right\}.
\end{equation*}
The remaining parts are then completely analogous to the lift chart. We define the exposure-weighted averages
\begin{align*}
 V_k^{\kappa}
 &=\sum_{i\in\widehat{\cal I}^{\kappa}_k}V_i,\\
 \overline Y_k^{\kappa}
 &=\frac{1}{V_k^{\kappa}}
   \sum_{i\in\widehat{\cal I}^{\kappa}_k}V_iY_i,\\
 \overline\Pi_k^{a,\kappa}
 &=\frac{1}{V_k^{\kappa}}
   \sum_{i\in\widehat{\cal I}^{\kappa}_k}V_i\Pi_i^a,
 \qquad
 \overline\Pi_k^{b,\kappa}
 =\frac{1}{V_k^{\kappa}}
   \sum_{i\in\widehat{\cal I}^{\kappa}_k}V_i\Pi_i^b.
\end{align*}

\medskip
\begin{graybox}
The double-lift chart considers the three graphs
\begin{equation}
 \Big(k,\overline Y_k^{\kappa}\Big),
 \quad
 \Big(k,\overline\Pi_k^{a,\kappa}\Big)
 \quad \text{and} \quad
 \Big(k,\overline\Pi_k^{b,\kappa}\Big) \qquad
 \text{ for $k=1,\ldots,K,$}
 \label{eq:double-lift-plot}
\end{equation}
on the same axes.
\end{graybox}
\medskip

\begin{figure}[htb!]
\begin{center}
\begin{minipage}[t]{0.45\textwidth}
\begin{center}
\includegraphics[width=\textwidth]{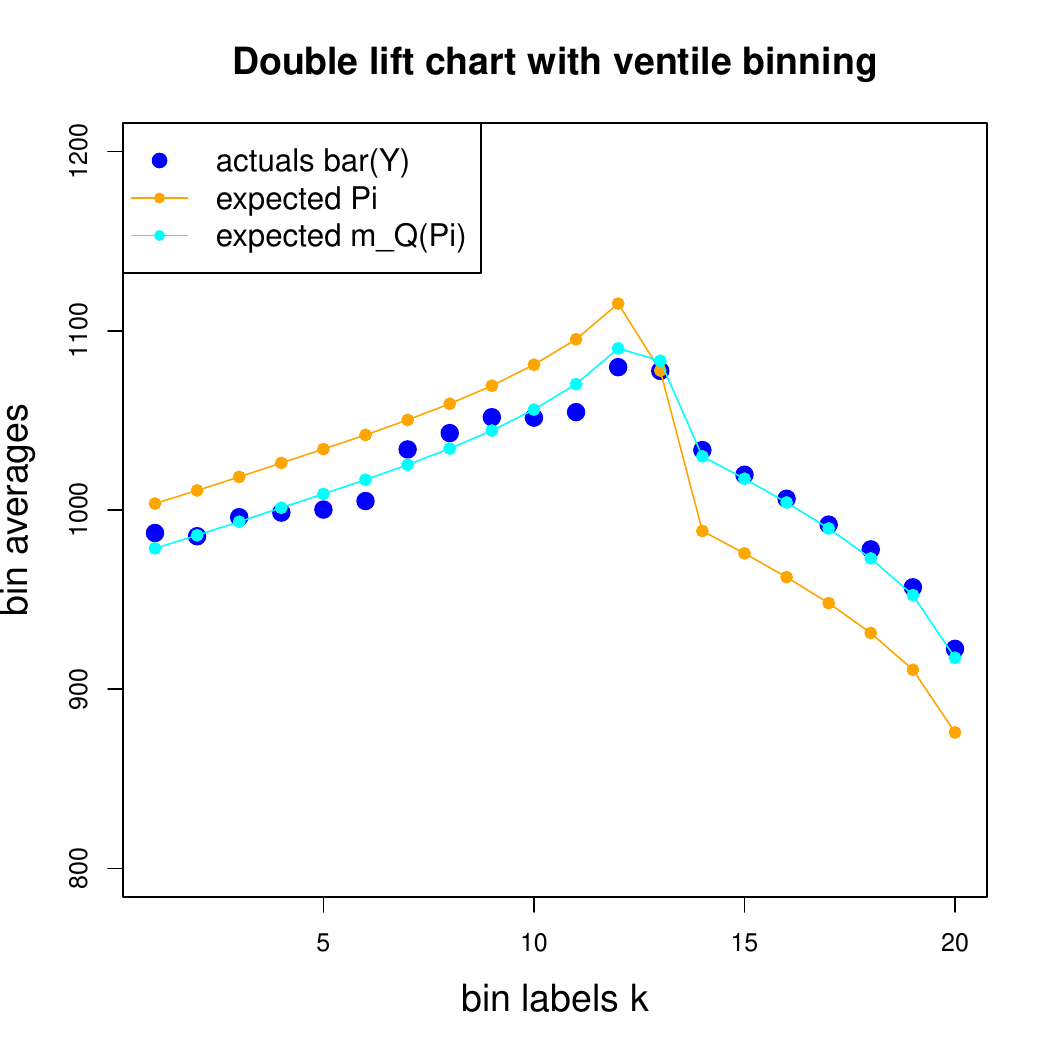}
\end{center}
\end{minipage}
\end{center}
\vspace{-.7cm}
\caption{Double lift chart using ventile binning $K=20$ for the ratio $\kappa=m_{\q}(\Pi)/\Pi$.}
\label{double lift chart}
\end{figure}

Figure \ref{double lift chart} gives the double lift chart for the ratio 
$\kappa=\Pi^b/\Pi^a=m_{\q}(\Pi)/\Pi$ using ventile binning $K=20$. On the left-hand side of the plot there are the policies where the unit premium $\Pi^ a=\Pi$ most severely overestimates the losses relative to $\Pi^b=m_{\q}(\Pi)$, and on the right-hand side it most severely underestimates the losses; remark that $m_{\q}(\Pi)$ is $\q$-calibrated by construction. The conclusion of this plot is that the expected values $\overline{m_{\q}(\Pi)}^{b, \kappa}$ match the actuals $\overline{Y}_k^\kappa$ much better than the expected values $\overline{\Pi}_k^{a,\kappa}$, thus, there is a clear preference for the recalibrated premiums from this double lift chart.

\section{Tools to assess discrimination and resolution}
\label{sec: Murphy diagram}
The previous graphical methods have mainly served to evaluate the calibration property \eqref{Q calibration}. These previous tools did not directly target at comparing two different pricing schemes $(Y, V, \Pi^a)$ and $(Y, V, \Pi^b)$ in terms of discrimination and resolution. The present section introduces a single score (summary statistic) that allows one for selecting among multiple competing pricing schemes mainly w.r.t.~discrimination, but there will also be some calibration terms involved. This score is based on {\it strictly consistent loss functions}; see Gneiting--Raftery \cite{GneitingRaftery} and Gneiting \cite{Gneiting}. For interpretation of strictly consistent loss functions, we introduce another graphical tool called the {\it Murphy diagram}.

\subsection{Strictly consistent loss function}
\label{sec: Strictly consistent loss function}
\subsubsection{Bregman divergence}
We give a brief introduction and present the main tools that are relevant for our purposes.
Consider a loss function $(y,x)\mapsto L(y,x)$ that compares responses $y$ and predictions $x$; the support of this loss function is a possibly infinite rectangle in $\R^2$.
Assume that the response $Y$ has a finite mean $\E[Y]$, and our goal is to determine this mean through an expected loss minimization
\begin{equation}\label{expected loss minimization}
\widehat{\mu}_0 ~\in~ \underset{ x }{\arg\min}~ \E \left[ L(Y,x) \right].
\end{equation}
Generally, this minimization \eqref{expected loss minimization} will fail to find the true mean $\E[Y]$, which is the motivation to define strictly consistent scoring for mean estimation.

\begin{defi}\label{strictly consistent}
A loss function $L$ is consistent
for mean estimation of $Y$ if 
\begin{equation}\label{consistent}
\E \left[ L(Y,\E[Y]) \right] ~\le~ \E \left[ L(Y,x) \right]
\qquad \text{ for all $x$ where the right-hand side is defined.}
\end{equation}
The loss function $L$ is strictly consistent for mean
estimation if an equality in \eqref{consistent} holds if and only if $x$ is the true mean $x=\E[Y]$.
\end{defi}

Under a strictly consistent loss function $L$ for mean estimation, the solution to \eqref{expected loss minimization} is the singleton $\widehat{\mu}_0=\E[Y]$. This makes it natural to measure the quality of a predictor $x$ for $Y$ by the strictly consistent loss $L(Y,x)$ -- smaller is better, and the (unique) minimum is the true mean $\E[Y]$. Remark that generally, a loss $L$ does not need to be symmetric in its arguments.

\medskip

The crucial mathematical result of  Savage \cite{Savage} and Gneiting \cite[Theorem 7]{Gneiting} yields that under mild technical conditions the (strictly) consistent loss functions are exactly the {\it Bregman divergences} \cite{Bregman} with (strictly) convex generators $\varphi$, i.e., the (strictly) consistent loss functions for mean estimation take the Bregman divergence form
\begin{equation}\label{Bregman loss definition}
  L_\varphi(y,x) =\varphi(y)-\varphi(x)-\varphi'(x)(y-x) ~\ge~ 0,
\end{equation}
with (sub-)gradient $\varphi'$ of the convex generator $\varphi$. This choice is regardless of the underlying probability law $\p$ as long as the left-hand side of \eqref{consistent} exists. We therefore always work with {\it Bregman losses} \eqref{Bregman loss definition} for mean estimation and  mean validation. Examples of Bregman losses are the square loss function, the Poisson deviance loss, the gamma deviance loss, and generally, all deviance losses from the EDF.

\medskip

This result of (strict) consistency for mean estimation carries over to conditional probability laws $\p(\cdot\mid{\cal A})$, yielding that the conditional expectation 
$\E[Y\mid {\cal A}]$ can be found by the Bregman loss minimization among the ${\cal A}$-measurable predictors; we come back to this in Section \ref{sec: Murphy's decomposition}, below. This makes the expected Bregman losses naturally suited to not only compare deterministic predictors $x$ as in \eqref{consistent}, but we can equally use it for ranking premiums $\Pi$. This motivates the following consideration:

For a given Bregman loss $L_\varphi$ and two competing pricing schemes $(Y, V, \Pi^a)$ and $(Y, V,\Pi^b)$, the one with the smaller expected Bregman loss should be preferred, i.e., 
\begin{equation}\label{preference order 2B}
\text{prefer $\Pi^a$ over $\Pi^b$ for $Y$} \qquad \Longleftrightarrow
\qquad
\E[L_\varphi(Y,\Pi^a)] \le \E[L_\varphi(Y,\Pi^b)].
\end{equation}
\medskip

{\bf Attention.} The difficulty with \eqref{preference order 2B} is that this preference order depends on the specific choice of the convex generator $\varphi$. Preference \eqref{preference order 2B} generally does not hold simultaneously for all convex generators $\varphi$; the simultaneous dominance under all convex generators (with aligned supports) is called {\it forecast dominance}, see Kr\"uger--Ziegel \cite[Definition 2.1]{KrugerZiegel}. In most practical forecast problems forecast dominance is not expected to hold, basically it is related to one premium rule using more information than the other, if both are calibrated, a precise mathematical statement involves convex orders; see Kr\"uger--Ziegel \cite[Theorem 3.1]{KrugerZiegel}.

\medskip

Testing preference \eqref{preference order 2B} for all convex generators $\varphi$ may not be feasible. Section \ref{Elementary losses and Murphy diagram} considers the simpler class of {\it elementary losses} which are parametrized through one single parameter $\theta \in \R$. This makes it easier to test \eqref{preference order 2B} for all elementary losses because we work on a parametrized class of (simple) losses. Moreover, the elementary losses can be seen as the building blocks of Bregman losses, this is explained in Section 
\ref{Elementary loss representation of the empirical Bregman loss}, below.

\subsubsection{Exposure-weighted Bregman divergence}

The above consistent scoring introduction considers the classical situation without exposures $V>0$. As was highlighted by Lindholm et al.~\cite{LLP}, in actuarial modeling, one typically considers weighted Bregman losses, this connects to Example \ref{ex: model fitting}. For a strictly convex generator $\varphi$,  \eqref{expected loss minimization} is replaced by 
\begin{equation}\label{expected loss minimization 2}
\mu_{\q} ~=~ \underset{ x }{\arg\min}~ \E \left[V\, L_\varphi(Y,x) \right]=\underset{ x }{\arg\min}~ \E_{\q} \left[L_\varphi(Y,x) \right],
\end{equation}
the latter uses the measure transformation \eqref{Q Radon Nikodym}. This weighted minimization provides a solution that typically differs from $\E[Y]$, namely, it provides the $\q$-mean of $Y$
\begin{equation}\label{realized-exposure target}
\mu_{\q} =\E_{\q}\left[Y \right] =\frac{1}{\E[V]}\,\E\left[VY \right]
=\frac{1}{\E[V]}\,\E\left[Z \right].
\end{equation} 
Concluding, in actuarial model fitting \eqref{expected loss minimization 2}, one targets the so-called {\it realized-exposure target} \eqref{realized-exposure target}. 
This accounts correctly for the dependence between the loss costs $Y$ and the exposure $V$, we also refer to \eqref{positive correlation}, and $Z=VY$ is the loss that the insurer will cover.

\medskip
\begin{graybox}
A natural consequence of this weighting is that also the preference order \eqref{preference order 2B} may change, and we instead consider
for a given Bregman loss $L_\varphi$:
\begin{eqnarray}\label{preference order 2V}
\text{prefer $\Pi^a$ over $\Pi^b$ for $Y$} &\quad \Longleftrightarrow
\quad&
\E[V\,L_\varphi(Y,\Pi^a)] \le \E[V\,L_\varphi(Y,\Pi^b)]
\\\nonumber
&\quad \Longleftrightarrow \quad&
\E_{\q}[L_\varphi(Y,\Pi^a)] \le \E_{\q}[L_\varphi(Y,\Pi^b)].
\end{eqnarray}
\end{graybox}
\medskip

We come back to the recalibration step of Definition \ref{def: Recalibrated unit premium}.

\begin{example}[Recalibration step and forecast dominance]
\label{Recalibration step and forecast dominance}
\normalfont
The recalibrated unit premium under the $\q$-measure is given by
\begin{equation*}
m_{\q}(\Pi) = \E_{\q} \left[Y \mid \Pi \right]
~\in~
\underset{ x }{\arg\min}~ \E_{\q} \left[ L_\varphi(Y,x)\mid \Pi \right]
\qquad \text{ a.s.},
\end{equation*}
for any convex generator $\varphi$ where the expected values exist. This implies
\begin{equation*}
\E_{\q} \left[ L_\varphi(Y,m_{\q}(\Pi)) \right]
\le \E_{\q} \left[ L_\varphi(Y,\Psi) \right],
\end{equation*}
for any $\sigma(\Pi)$-measurable random variable $\Psi$.
In particular, this applies to $\Psi=\Pi$, and as a consequence
\begin{equation*}
\text{prefer $m_{\q}(\Pi)$ over $\Pi$ for $Y$ under $L_\varphi$.}
\end{equation*}
Since this holds for any convex generator $\varphi$, we obtain forecast dominance in the sense of Kr\"uger--Ziegel \cite[Definition 2.1]{KrugerZiegel}. 
\medskip

\begin{center}
\noindent\fbox{%
  \begin{minipage}{0.9\textwidth}
    ~\\
Thus, the recalibrated unit premium $m_{\q}(\Pi)$ is the most accurate $\sigma(\Pi)$-measurable predictor of $Y$ under $\q$ in the forecast dominance sense.
\\~    
  \end{minipage}%
}
\end{center}

\medskip

This concludes the example.
\EndExample
\end{example}

\subsection{Elementary losses and Murphy diagram}
\label{Elementary losses and Murphy diagram}
\subsubsection{Elementary losses}
We make a first step towards a better understanding of the Bregman loss and a specific choice of a convex generator $\varphi$. This is done by discussing the {\it elementary losses}. The elementary losses have been introduced by Ehm et al.~\cite{Ehm}, and they allow for a mixture representation of the Bregman loss.
The {\it elementary losses} are based on the convex functions, for $\theta \in \R$, 
\begin{equation}\label{elementary generator}
x\in \R ~\mapsto~\varphi_\theta(x)=(x-\theta)_+,
\end{equation}
and they are given by
\begin{equation}\label{elementary scores}
L_\theta(y, x)=(y-\theta)_+-(x-\theta)_+-(y-x)\1_{\{x>\theta\}}.
\end{equation}
The last term in \eqref{elementary scores} gives a sub-derivative of $x\mapsto \varphi_\theta(x)$ with a non-differentiability occurring at $x=\theta$, an other possible choice is $\1_{\{x\ge \theta\}}$. 

\medskip

\underline{Notation:}
$L_\varphi$ denotes a Bregman loss for a general convex generator $\varphi$, and $L_\theta=L_{\varphi_\theta}$ is an elementary loss with generator $\varphi_\theta$ for fixed $\theta \in \R$.

\medskip

The elementary loss \eqref{elementary scores} has many equivalent reformulations
\begin{eqnarray}
L_\theta(y,x)&=&\nonumber
(y-\theta)_+-(x-\theta)_+-(y-x)\1_{\{x>\theta\}}
\\&=&\nonumber (y-\theta)\1_{\{y>\theta\}}-(x-\theta)\1_{\{x>\theta\}}-(y-x)\1_{\{x>\theta\}}
\\&=&\nonumber
(y-\theta)\1_{\{y>\theta\}}-(y-\theta)\1_{\{x>\theta\}}
\\&=&\nonumber (y-\theta)\left(\1_{\{y>\theta\}}-\1_{\{x>\theta\}}\right)
\\&=&\nonumber (y-\theta)\1_{\{y>\theta\ge x\}}
      + (\theta-y)\1_{\{y\le \theta <x\}}
  \\&=&\label{elementary scores 2}
        |y-\theta|\left(\1_{\{y\le \theta <x\}}+\1_{\{x\le \theta<y\}}\right).
\end{eqnarray}
The last expression says that it measures the distance $|y-\theta|$, whenever $\theta$ is between the response $y$ and the forecast $x$. This will be illustrated and discussed in an AvE plot in Figure \ref{Murphy Pi 2}, below.

\subsubsection{Murphy diagram}
The elementary loss $L_\theta$ is a Bregman loss and it can be used for 
preference ordering \eqref{preference order 2V}, it is a consistent loss function for mean estimation, but it is not strictly consistent because \eqref{elementary generator} is only convex. Similar to the selection of the convex generator $\varphi$ for preference ordering \eqref{preference order 2V}, we may raise the question about the specific choice of $\theta$ in the elementary loss $L_\theta$ selection. The {\it Murphy diagram} considers the expected elementary loss as a function of $\theta$
\begin{equation}\label{Murphy graph}
\theta \in \R ~\mapsto ~ \E_{\q} \left[L_\theta(Y,\Pi) \right]
=\frac{1}{\E \left[V \right]}\,\E \left[V\,L_\theta(Y,\Pi) \right].
\end{equation}
We now generally work with the exposure-weighted measure $\q$.

\medskip
\begin{graybox}
The sample version of the Murphy diagram on the test sample ${\cal T}=(Y_i,V_i,\Pi_i)_{i=1}^n$ is given by
\begin{equation}\label{empirical Murphy graph}
\theta ~\mapsto~ \widehat{\E}_{\q}[L_\theta(Y,\Pi)] = \frac{1}{ \sum_{i=1}^n V_i}\, \sum_{i=1}^n
V_i\,L_\theta(Y_i,\Pi_i).
\end{equation}
\end{graybox}

\medskip

The Murphy diagram was introduced in Ehm et al.~\cite{Ehm}, and it has become a popular graphical model validation tool, see, e.g., Dimitriadis et 
al.~\cite{Dimitriadis}.

\medskip

\begin{figure}[htb!]
\begin{center}
\begin{minipage}[t]{0.45\textwidth}
\begin{center}
\includegraphics[width=\textwidth]{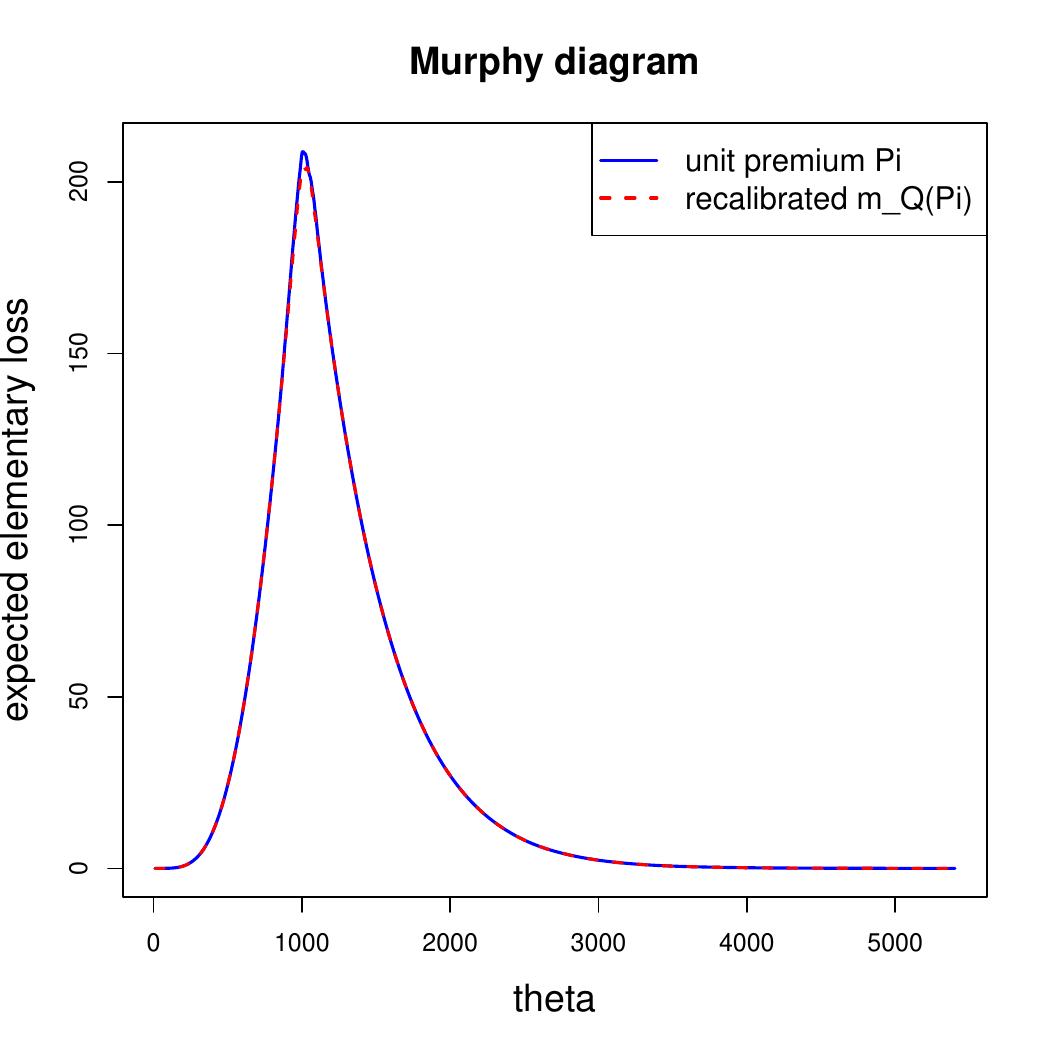}
\end{center}
\end{minipage}
\begin{minipage}[t]{0.45\textwidth}
\begin{center}
\includegraphics[width=\textwidth]{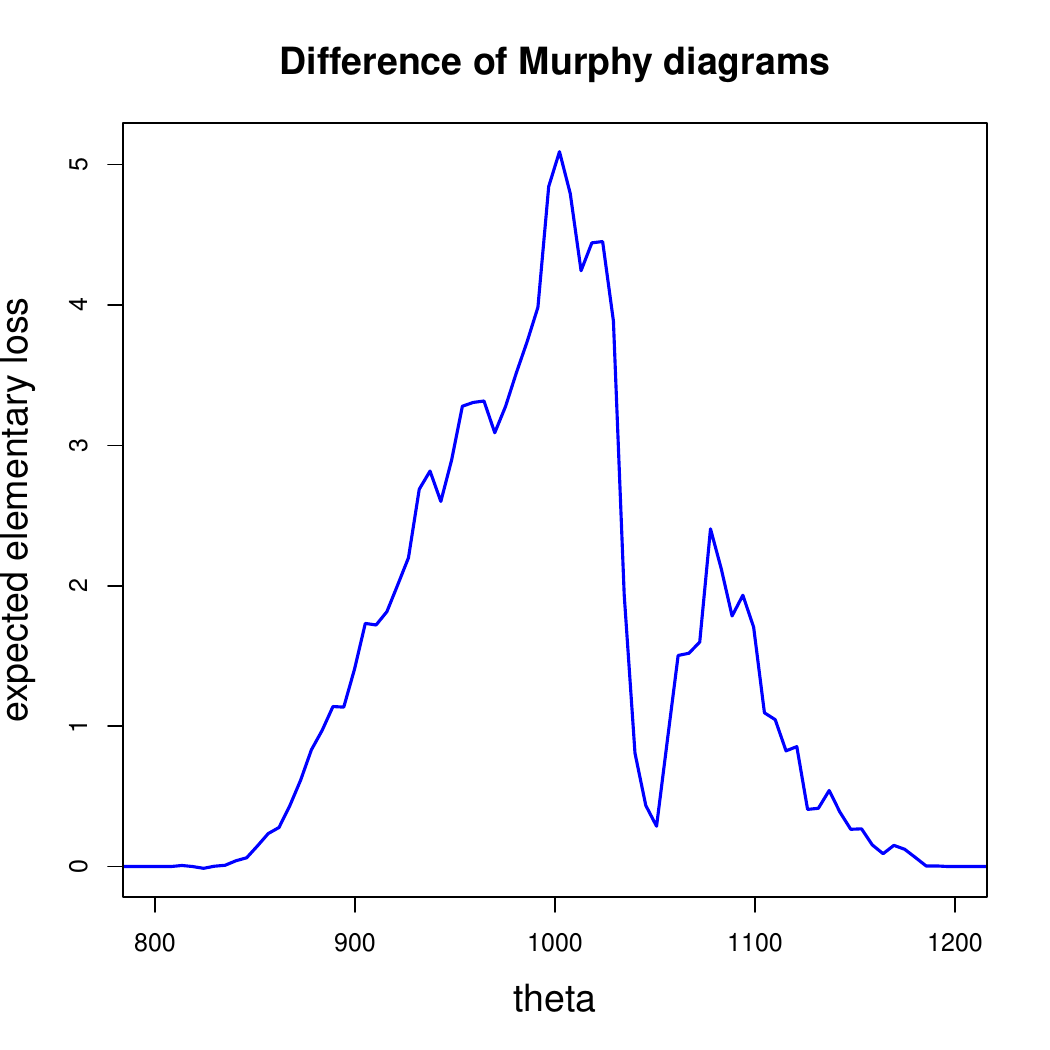}
\end{center}
\end{minipage}
\end{center}
\vspace{-.7cm}
\caption{(lhs) Murphy diagrams of the unit premium rules $\Pi$ and $m_{\q}(\Pi)$, and (rhs) their difference 
$\Delta_{\q}^{\rm Murphy}(\theta; \Pi, m_{\q}(\Pi))$.}
\label{Murphy Pi 1}
\end{figure}

Figure \ref{Murphy Pi 1} shows the (empirical) Murphy diagrams \eqref{empirical Murphy graph} of the two premium rules 
$\Pi$ and $m_{\q}(\Pi)$ in blue and red on the left-hand side, and the right-hand side shows their difference
\begin{eqnarray}\label{Murphy Delta difference}
\theta ~\mapsto ~ \Delta_{\q}^{\rm Murphy}\left(\theta; \Pi, m_{\q}(\Pi)\right)
&=&
\widehat{\E}_{\q}[L_\theta(Y,\Pi)] -\widehat{\E}_{\q}[L_\theta(Y,m_{\q}(\Pi))]
\\&=&\nonumber~ \frac{1}{ \sum_{i=1}^n V_i}\, \sum_{i=1}^n
V_i\left(L_\theta(Y_i,\Pi_i)-L_\theta(Y_i,m_{\q}(\Pi_i))\right).
\end{eqnarray}
In Figure \ref{Murphy Pi 1} (rhs), we
observe positivity of this difference $\Delta_{\q}^{\rm Murphy}(\theta; \Pi, m_{\q}(\Pi))$ for most of the values $\theta \in \R$ (up to minor noise perturbations). This suggests to prefer 
premium rule $m_{\q}(\Pi)$ over $\Pi$ for forecasting $Y$. Of course, this is clear in this example because it verifies the forecast dominance statement discussed in Example \ref{Recalibration step and forecast dominance}.

\medskip

From a practical point of view, we see the following issues:
\begin{itemize}
\item For two general pricing rules $\Pi^a$ and $\Pi^b$, we do not expect such a clear preference picture as in Figure \ref{Murphy Pi 1} (rhs). First, we do not expect that there is this dominance for all $\theta \in \R$ if we consider two general pricing rules that are roughly equally accurate, for example, comparing a gradient boosting rule $\Pi^a$ to a neural network rule $\Pi^b$. Moreover, the noise in the responses $Y$ will contaminate the Murphy decision diagram.
\item  Figure \ref{Murphy Pi 1} shows a situation in which we can explicitly compute the recalibrated premium $m_{\q}(\Pi)$ because we know the population model in our example. Generally, this is not the case, and the empirical recalibration step will add an other major source of uncertainty (inaccuracy) to this decision making problem.
\item The Murphy diagram in Figure \ref{Murphy Pi 1} looks appealing and interpretable. However, it is not so obvious from that graph in which range the premiums fit well and in which part they do not.
The elementary loss measures the distance $|Y_i-\theta|$, whenever $\theta$ is between the response $Y_i$ and the premium $\Pi_i$, see \eqref{elementary scores 2}.
As a result, the sample Murphy diagram \eqref{empirical Murphy graph} presents an overlap of different pairs $(Y_i,\Pi_i)$ for which the selected threshold $\theta$ is in between these two values. We will discuss this a bit further in the next graph.
\end{itemize}

\begin{figure}[htb!]
\begin{center}
\begin{minipage}[t]{0.9\textwidth}
\begin{center}
\includegraphics[width=\textwidth]{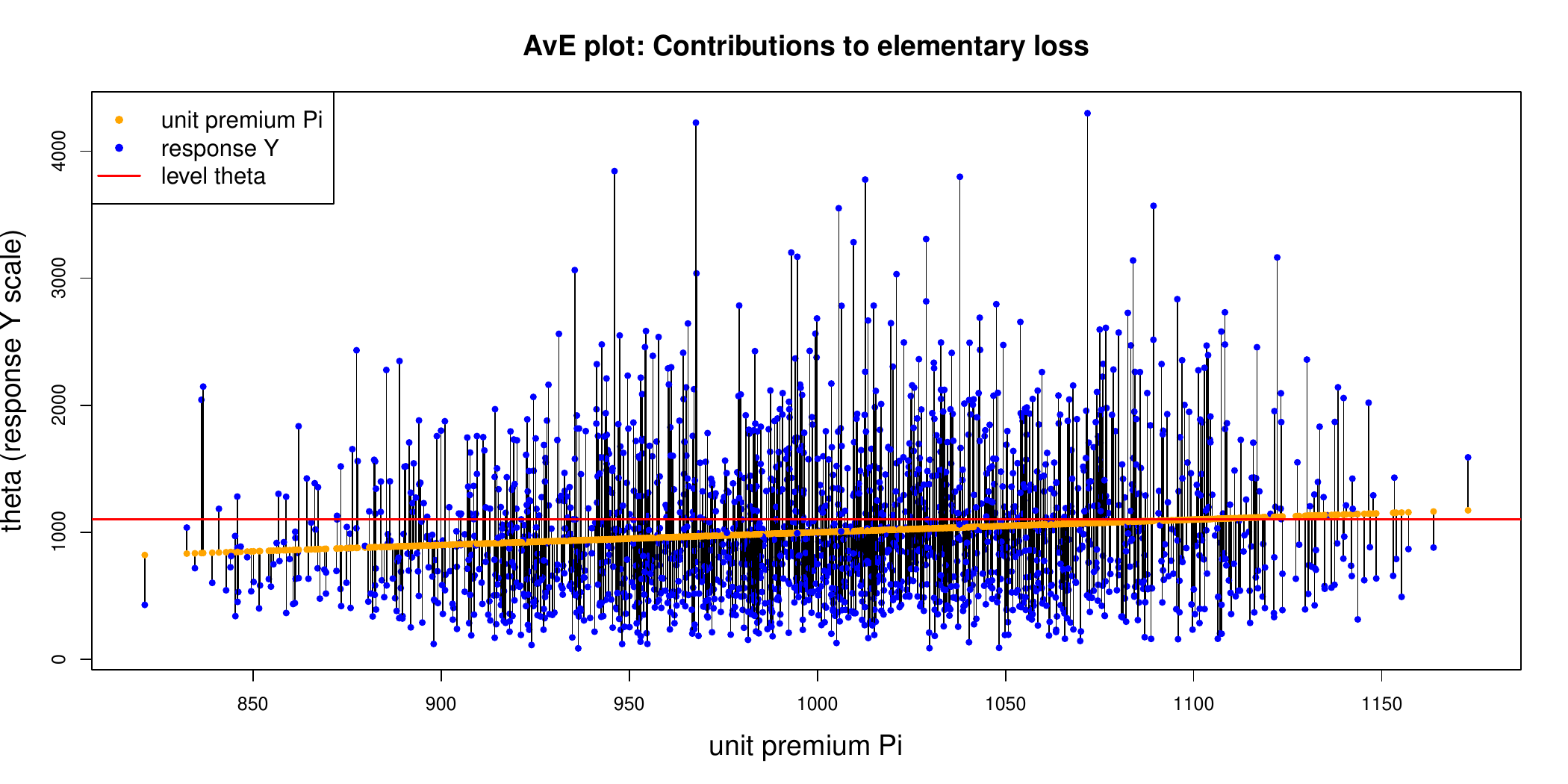}
\end{center}
\end{minipage}
\end{center}
\vspace{-.7cm}
\caption{AvE plot: Contributions to the elementary loss $L_\theta$ for fixed $\theta$ (red horizontal line).}
\label{Murphy Pi 2}
\end{figure}

We recall the AvE plot of Section \ref{Actual-vs-expected plot -- quantile binning}. The AvE plot uses the bin averages \eqref{AvE plot definition 1}. We replace these bin averages by the individual observations (for actuals $Y_i$ vs.~expected $\Pi_i$)
\begin{equation*}
\left(\Pi_i, \,Y_i\right) 
\quad {and} \quad
\left(\Pi_i,\, \Pi_i\right) 
 \qquad \text{ for $i=1, \ldots, n$,}
\end{equation*}
i.e., we discard the binning.  We then add vertical black segments to connect the 
blue dots \bl{$(\Pi_i,Y_i)$} with the orange diagonal dots
\textcolor{orange}{$(\Pi_i,\Pi_i)$} in Figure \ref{Murphy Pi 2}. Finally, we select a value $\theta$ on the $y$-axis in Figure \ref{Murphy Pi 2} and we plot a horizontal \red{red line} on this level. Every instance $i$, for which the black vertical segment between \bl{$(\Pi_i,Y_i)$} and \textcolor{orange}{$(\Pi_i,\Pi_i)$} intersects the \red{red line}, contributes to the sample elementary loss \eqref{empirical Murphy graph} at the selected level $\theta$, and the size of the contribution is equal to $|Y_i-\theta|$, see \eqref{elementary scores 2}. For example, the instance with the highest unit premium (to the very right in Figure \ref{Murphy Pi 2}) does not contribute to the selected level $\theta$.
This is the formal procedure of computing
the sample elementary loss
\begin{equation*}
\widehat{\E}_{\q}[L_\theta(Y,\Pi)] \,=\,
\frac{1}{ \sum_{i=1}^n V_i}\, \sum_{i=1}^n V_i\,L_\theta(Y_i,\Pi_i)
\,=\,\frac{1}{ \sum_{i=1}^n V_i}\, \sum_{i=1}^n
V_i \left(Y_i-\theta\right)\left(\1_{\{Y_i>\theta\}}-\1_{\{\Pi_i>\theta\}}\right).
\end{equation*}
From Figure \ref{Murphy Pi 2} it is difficult to interpret these contributions because there are too many instances $i \in \{1,\ldots, n\}$ in the plot. Therefore, we consider an aggregated version (using quantile binning). This aggregated version does not give the same Murphy diagram, but an interpretable aggregated version. 

\subsubsection{The asymmetrically binned Murphy diagram}
Figure \ref{Murphy Pi 2}, being based on individual instances, is hardly interpretable. We therefore build a binned version thereof. This binned version does {\it not} reproduce the sample Murphy diagram \eqref{empirical Murphy graph}, but it is {\it only} used as a graphical tool for interpretation. However, it has a useful application that gives much deeper insight into pricing and risk classification, see Remark \ref{binned Murphy}
and Section \ref{Illustration of resolution and miscalibration}, below.

\medskip

Our goal is to compare the two premium rules $\Pi^a$ and $\Pi^b$. For binning we need to select one of the two, e.g., $\Pi^a$, and the resulting plot will depend on this choice -- that is why we call the binned version ``asymmetric'' when comparing the two premium rules $\Pi^a$ and $\Pi^b$. 
We call the resulting graph the {\it $\Pi^a$-binned Murphy diagram}.
Select premium rule $\Pi^a$ for binning, and define the $K \in \N$ bins w.r.t.~$\Pi^a$ by
\begin{equation}\label{Q bins A}
\widehat{\cal I}^{a}_k = \left\{ i \in \{1,\ldots, n\} \,\middle|\, 
(\widehat{G}^{\q}_{\Pi^a})^{-1}((k-1)/K)< \Pi^a_i \le (\widehat{G}^{\q}_{\Pi^a})^{-1}(k/K)\right\}.
\end{equation}
We compute the weighted average loss costs and the binned exposures 
\begin{equation*}
\overline{Y}^a_k  = \frac{1}{V^a_k}\,\sum_{i \in \widehat{\cal I}^{a}_k} V_iY_i \qquad \text{ and }\qquad
V^a_k  = \sum_{i \in \widehat{\cal I}^{a}_k} V_i.
\end{equation*}
Analogously, this yields the weighted $\Pi^a$-binned unit premiums
\begin{equation*}
\overline{\Pi}^a_k  = \frac{1}{V^a_k}\,\sum_{i \in \widehat{\cal I}^{a}_k} V_i\Pi^a_i
\qquad \text{ and }\qquad 
\overline{\Pi}^{b/a}_k  = \frac{1}{V^a_k}\,\sum_{i \in \widehat{\cal I}^{a}_k} V_i\Pi^b_i.
\end{equation*}
This then motivates the (asymmetrically) {\it $\Pi^a$-binned Murphy diagrams}
\begin{eqnarray}\label{binned Pi a}
\theta & \mapsto &\frac{1}{ \sum_{k=1}^K V^a_k}\, \sum_{k=1}^K
V^a_k \left(\overline{Y}^a_k-\theta\right)\left(\1_{\{\overline{Y}^a_k>\theta\}}-\1_{\{\overline{\Pi}^{\textcolor{red}{a}}_k>\theta\}}\right),
\\\label{binned Pi b}
\theta & \mapsto &\frac{1}{ \sum_{k=1}^K V^a_k}\, \sum_{k=1}^K
V^a_k \left(\overline{Y}^a_k-\theta\right)\left(\1_{\{\overline{Y}^a_k>\theta\}}-\1_{\{\overline{\Pi}^{\textcolor{red}{b/a}}_k>\theta\}}\right).
\end{eqnarray}
The difference in \eqref{binned Pi a}-\eqref{binned Pi b} stems from the last indicator considering $\overline{\Pi}^{\textcolor{red}{a}}_k$ and $\overline{\Pi}^{\textcolor{red}{b/a}}_k$, respectively, this is highlighted in red color. 

\begin{remark}\normalfont
\begin{itemize}
\item We emphasize that generally
\begin{equation}\label{AvE inequality}
L_\theta(\overline{Y}^a_k, \overline{\Pi}^a_k)\,\neq\,
\frac{1}{ V^a_k}\, \sum_{i \in \widehat{\cal I}^{a}_k} V_i \,L_\theta(Y_i,\Pi_i).
\end{equation}
The binned version is only used as a graphical tool to reduce the complexity and the noise in the plots, but it does not serve at computing the expected elementary loss.
\item We can exchange the role of the two premium rules in \eqref{binned Pi a}-\eqref{binned Pi b} which gives the {\it $\Pi^b$-binned Murphy diagram}. We present both binning versions to understand the impact of binning.
\item If the two premium rules $\Pi^a$ and $\Pi^b$ provide the same risk ranking (in terms of the unit premium), they will result in the identical binning, thus, the binning is rank based w.r.t.~the premium rule.
\end{itemize}
\end{remark}

\begin{figure}[htb!]
\begin{center}
\begin{minipage}[t]{0.45\textwidth}
\begin{center}
\includegraphics[width=\textwidth]{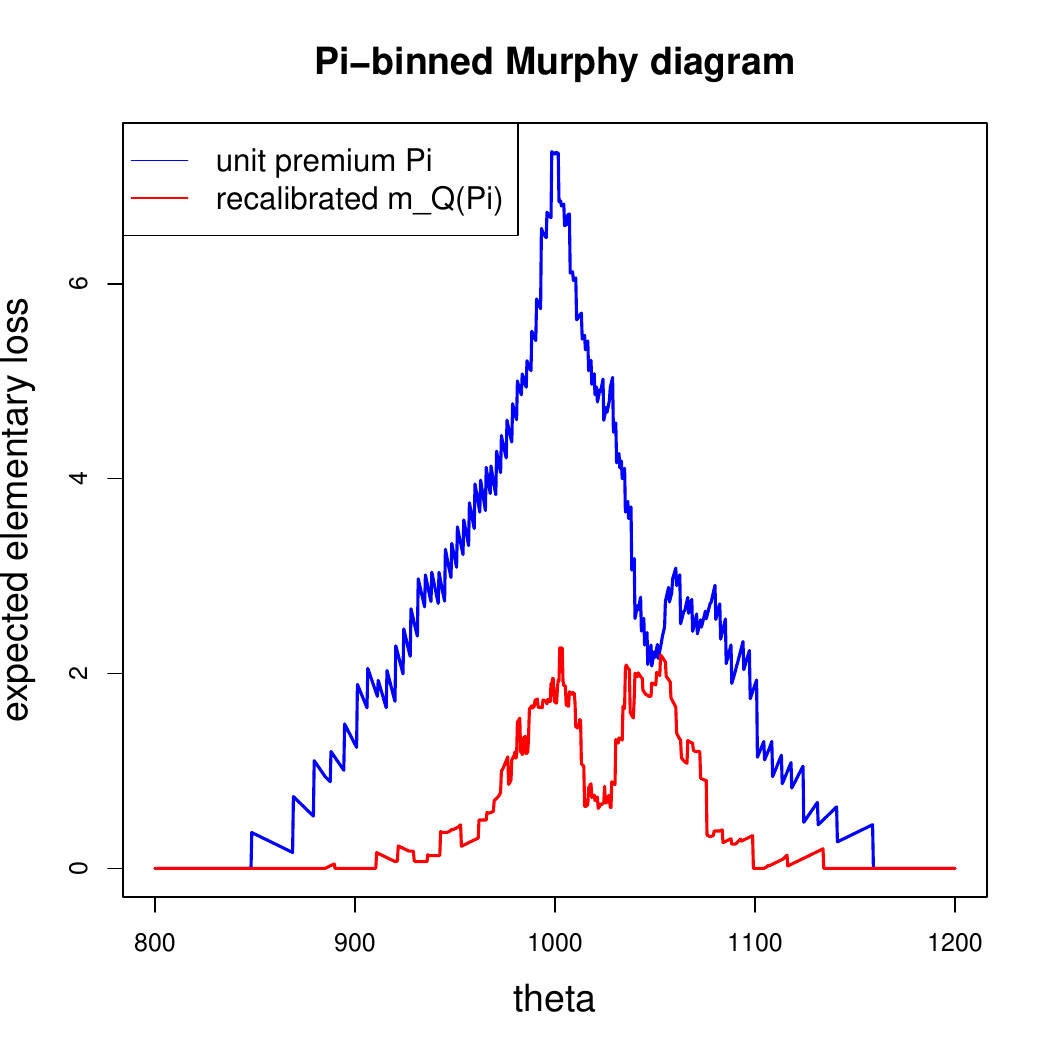}
\end{center}
\end{minipage}
\begin{minipage}[t]{0.45\textwidth}
\begin{center}
\includegraphics[width=\textwidth]{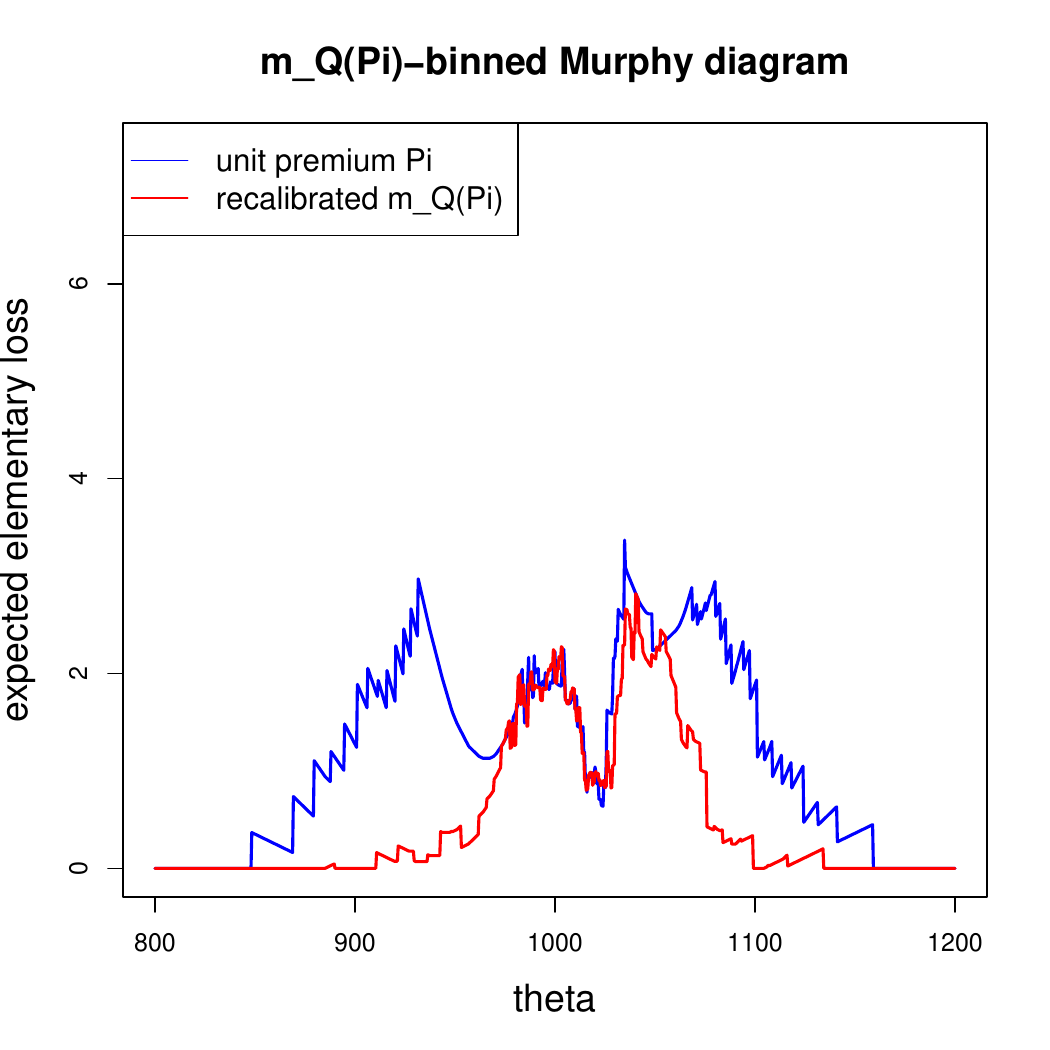}
\end{center}
\end{minipage}
\end{center}
\vspace{-.7cm}
\caption{Asymmetrically binned Murphy diagrams of the unit premium rules $\Pi$ and $m_{\q}(\Pi)$: (lhs) $\Pi$-binned and (rhs) $m_{\q}(\Pi)$-binned for percentile binning $K=100$; the $y$-scale is identical in the two plots.}
\label{binned Murphy Pi 1}
\end{figure}

Figure \ref{binned Murphy Pi 1} shows the asymmetrically binned Murphy diagrams with percentile binning $K=100$. The left-hand side uses the unit premium rule $\Pi$ for binning and the right-hand side the recalibrated version $m_{\q}(\Pi)$.
In this example, the binning choice only marginally affects the binned Murphy diagram of the recalibrated premiums $m_{\q}(\Pi)$, but it impacts the binned Murphy diagram of the premium $\Pi$ quite significantly; recall that $\Pi$ does not provide the correct risk ranking, see Remark \ref{wrong risk ranking}. Based on these binned Murphy diagrams, there is a clear preference of $m_{\q}(\Pi)$ over $\Pi$, which empirically verifies once more the results of Example \ref{Recalibration step and forecast dominance}.

\medskip

\begin{figure}[htb!]
\begin{center}
\begin{minipage}[t]{0.9\textwidth}
\begin{center}
\includegraphics[width=\textwidth]{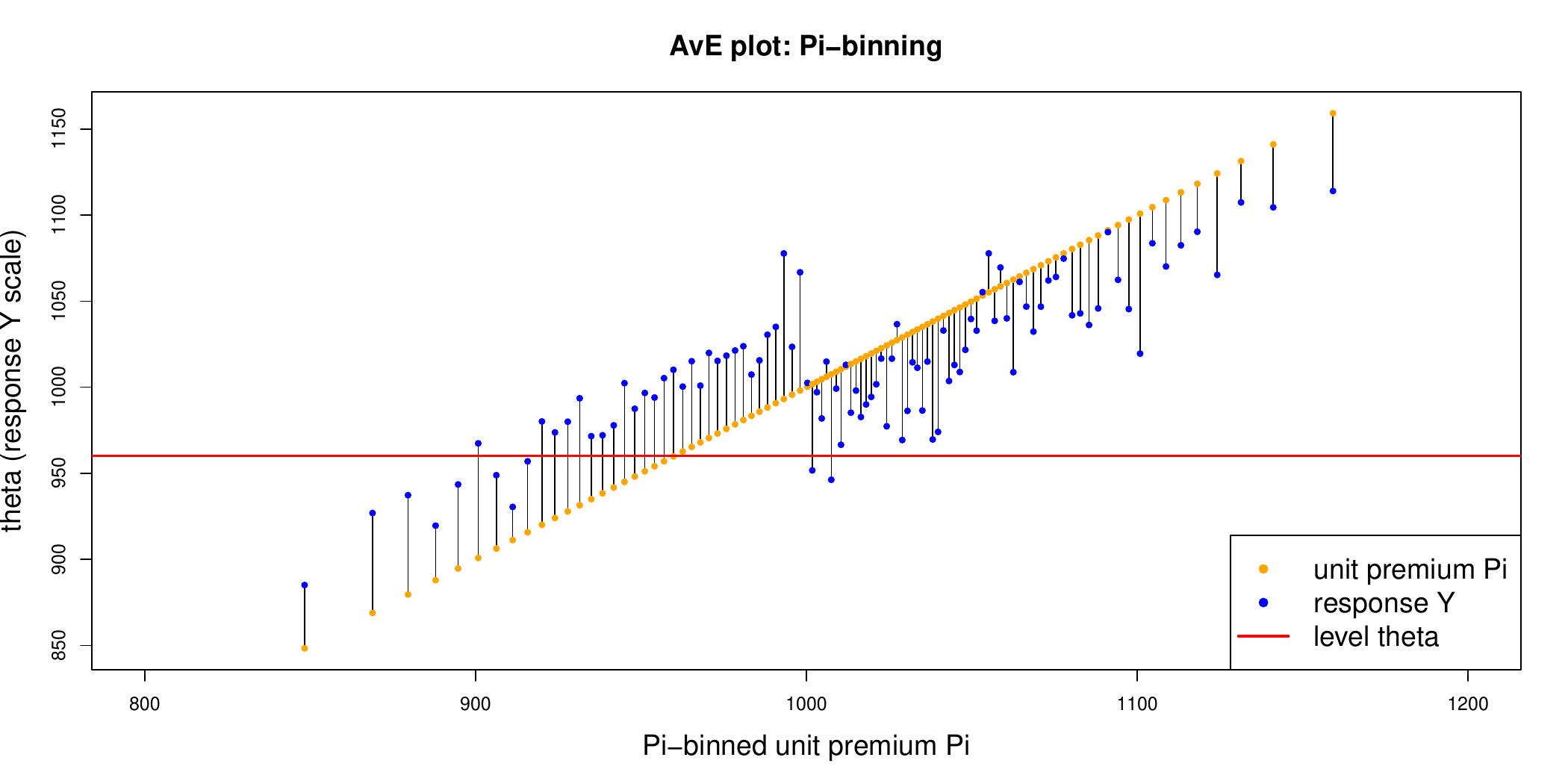}
\end{center}
\end{minipage}
\begin{minipage}[t]{0.9\textwidth}
\begin{center}
\includegraphics[width=\textwidth]{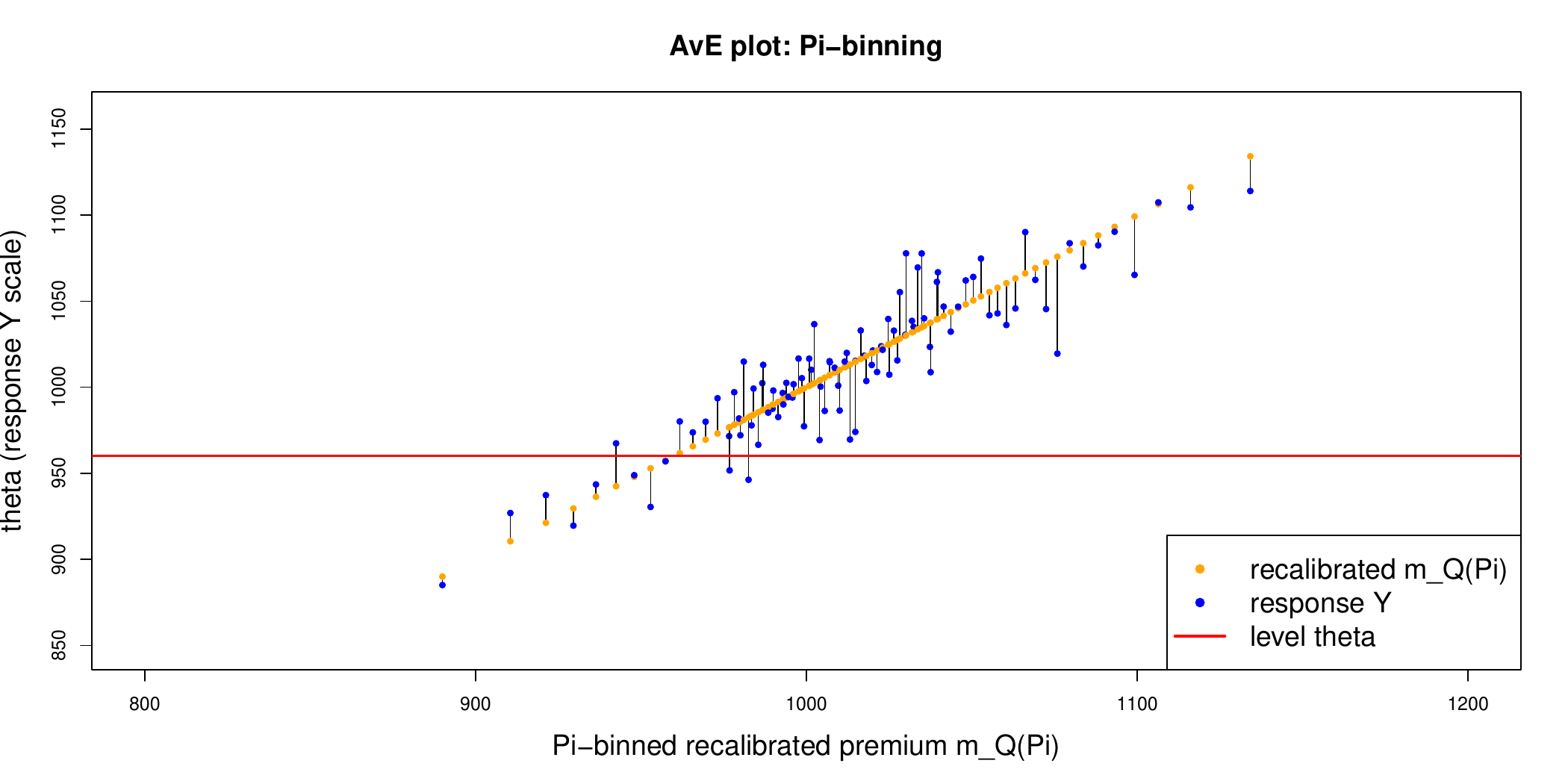}
\end{center}
\end{minipage}
\end{center}
\vspace{-.7cm}
\caption{AvE plot with $\Pi$-percentile binning: upper panel $\overline{Y}^a_k$ vs.~$\overline{\Pi}^{\textcolor{red}{a}}_k$ and lower panel
$\overline{Y}^a_k$ vs.~$\overline{\Pi}^{\textcolor{red}{b/a}}_k$ with premium rules $\Pi^a=\Pi$ and $\Pi^b=m_{\q}(\Pi)$.}
\label{AvE Murphy 1}
\end{figure}

Figure \ref{AvE Murphy 1} shows the AvE plot using percentile binning w.r.t.~$\Pi^a=\Pi$, see \eqref{AvE plot definition 1}. We again connect the 
actual $(\overline{\Pi}^a_k, \overline{Y}^a_k)$ with the expected
$(\overline{\Pi}^a_k, \overline{\Pi}^a_k)$ by vertical black segments, and if these segments intersect the horizontal red line at level $\theta$, they contribute to the elementary loss $L_\theta$. The upper panel in Figure 
\ref{AvE Murphy 1} shows the situation of the unit premium $\Pi^a=\Pi$ and the lower panel gives the recalibrated case $\Pi^b=m_{\q}(\Pi)$. This recalibrated version is binned w.r.t.~$\Pi^a=\Pi$, and it considers the actual
$(\overline{\Pi}^{b/a}_k, \overline{Y}^a_k)$ against the expected
$(\overline{\Pi}^{b/a}_k, \overline{\Pi}^{b/a}_k)$. Generally, in the upper panel the black segments are longer and the unit premium range that contributes to the elementary loss $L_\theta$ for a fixed $\theta$ is bigger for $\Pi$ than $m_{\q}(\Pi)$. This indicates the calibration issue of $\Pi$. Note that the $y$-levels of the actuals $\overline{Y}^a_k$ are identical in both panels, only their $x$-coordinates 
$\overline{\Pi}^a_k$ and $\overline{\Pi}^{b/a}_k=\overline{m_{\q}(\Pi)}^a_k$ differ.

We conclude that the asymmetrically $\Pi^a$-binning Murphy diagram and the resulting AvE plot of Figure \ref{AvE Murphy 1} is mainly a tool to graphically understand which premium levels contribute to the elementary losses.

\begin{remark}\normalfont
\label{binned Murphy}
There is a stronger mathematical justification of the $\Pi$-binned version of the Murphy diagram that is related to the discussion in Section \ref{sec: Murphy's decomposition}. Fix the number of bins $K \in \N$.
Denote by $B$ the bin allocation of a randomly selected instance $(Y,V,\Pi)$ among the $K$ quantile bins under the $\q$-distribution $G_{\Pi}^{\q}$. This allows us to define the recalibration step w.r.t.~the bin indicator $B$
\begin{equation*}
m_{\q}(B) = \E_{\q}\left[ Y \mid B \right]
\qquad \text{ and } \qquad 
\pi_B = \E_{\q}\left[ \Pi \mid B \right].
\end{equation*}
This gives us the binned Murphy diagram
\begin{equation*}
\theta ~ \mapsto ~ \E_{\q}\left[ L_\theta(m_{\q}(B), \pi_B)\right].
\end{equation*}
The graph \eqref{binned Pi a} is a sample version of this binned Murphy diagram.
In fact, we have the following Murphy's decomposition, see
Section \ref{Illustration of resolution and miscalibration} below for a proper treatment,
\begin{equation}\label{binning to be proved}
\E_{\q}\left[ L_\theta(Y, \pi_B)\right]
= \E_{\q}\left[ L_\theta(Y,m_{\q}(B))\right]
+\E_{\q}\left[ L_\theta(m_{\q}(B), \pi_B)\right].
\end{equation}
The first term on the right-hand side is the irreducible within-bin variation produced by $Y$, and the second term gives a miscalibration error, by using $\pi_B$ instead of the calibrated version $m_{\q}(B)$. If $\pi_B$ is $\q$-calibrated, this term vanishes, see Section \ref{sec: Murphy's decomposition}, below.
\end{remark}

\subsection{From elementary losses to Bregman divergences}
\label{Elementary loss representation of the empirical Bregman loss}
From \eqref{elementary generator}-\eqref{elementary scores} it is immediately clear that elementary losses are Bregman losses, and aggregating these elementary losses will preserve the Bregman loss property. This leads to the following mathematical result.
Under mild regularity conditions, Bregman losses can be written as Lebesgue--Stieltjes integrals over the elementary losses, see Ehm et al.~\cite[Theorem 1b]{Ehm},
\begin{equation}\label{Bregman loss vs elementary loss}
L_\varphi(y,x)=\int L_\theta(y,x)\,\dd\varphi'(\theta).
\end{equation}
Thus, Bregman losses consist of mixtures of elementary losses $L_\theta$ with mixing measure $\dd\varphi'(\theta)$. Inserting the specific form of the elementary loss, we can rewrite \eqref{Bregman loss vs elementary loss} as
\begin{equation*}
L_\varphi(y,x)=\1_{\{y> x\}}\int_x^y (y-\theta)\,\dd\varphi'(\theta)+
\1_{\{y<x\}}\int_y^x (\theta-y)\,\dd\varphi'(\theta).
\end{equation*}
This now directly connects to the AvE plots of 
Figures \ref{Murphy Pi 2} and \ref{AvE Murphy 1}. These plots show for a given unit premium, say $\Pi_i$, the segments from $\Pi_i$ to $Y_i$ by the vertical black lines (for a fixed policy $i$). To compute the elementary loss $L_\theta(Y_i,\Pi_i)$ we obtain the length of the segment $|Y_i - \theta|$, supposed that $\theta$ is between $Y_i$ and $\Pi_i$, see \eqref{elementary scores 2}. The Bregman loss adds a scaling over the segment from $\Pi_i$ to $Y_i$ which is determined by $\dd \varphi'(\theta)$. Assume that $\varphi$ is twice differentiable, then the scaling is given by $\varphi''(\theta)$, and we compute 
\begin{equation}\label{Bregman loss vs elementary loss 2}
L_\varphi(y,x)=\1_{\{y> x\}}\int_x^y (y-\theta)\varphi''(\theta)\,\dd \theta+
\1_{\{y<x\}}\int_y^x (\theta-y)\varphi''(\theta)\,\dd \theta.
\end{equation}
In the AvE plot of Figure \ref{Murphy Pi 2}, this scaling $\varphi''(\theta)$ acts on the $y$-axis, and it allows us to downgrade, for example, large loss costs by selecting a monotonically decreasing function 
for $\theta \mapsto \varphi''(\theta)$. Such a downgrading of large losses may make sense, namely, we do not generally expect the mean to capture large losses (tail losses), and henceforth, such large losses should not impact the mean model selection too much (if the mean is not dominated by these large losses, i.e., calibration is given). Remark that this downgrading {\it only} applies to the model validation part, but not to the model estimation part. I.e., if we would censor claims during model fitting, we would underestimate the true average losses (and likely calibration is violated).

\medskip
\begin{graybox}
A popular choice for the convex generator $\varphi$ is the Patton family \cite{Patton} obtained by selecting a fixed $p\in \R$ and setting
\begin{equation}\label{Tweedie choices}
\varphi''(\theta) = \theta^{-p},
\end{equation}
on $\theta > 0$. In fact, one can verify that
\begin{eqnarray*}
p=0 && \text{ relates the square loss \eqref{Q minimization},}\\
p=1 && \text{ relates the Poisson deviance loss,}\\
p=2 && \text{ relates the gamma deviance loss,}\\
p=3 && \text{ relates the inverse Gaussian deviance loss.}
\end{eqnarray*}
\end{graybox}
\medskip

All these losses belong to the EDF, and they build a model hierarchy with an increasing variance behavior; see W\"uthrich--Merz \cite{WM2023}. We demonstrate the case of the gamma deviance loss $p=2$ in Example \ref{example gamma deviance}, below, and the Poisson case is studied in Section \ref{sec Poisson deviance loss}, below.

\subsection{Sample Bregman loss and currency invariant preferences}
The estimation of the expected Bregman loss $\E_{\q}[L_\varphi(Y,\Pi)]$ is straightforward from an i.i.d.~test sample ${\cal T}=(Y_i,V_i, \Pi_i)_{i=1}^n$.
Namely, we set
\begin{equation}\label{empirical Bregman loss}
\widehat{\E}_{\q}[L_\varphi(Y,\Pi)] = \frac{1}{\sum_{i=1}^n V_i}\, \sum_{i=1}^n V_i\,
L_\varphi(Y_i,\Pi_i).
\end{equation}
This is the loss figure that is typically reported in statistical analysis; see, e.g., W\"uthrich et al.~\cite{AITools}.
Based on this sample Bregman loss, we can perform preference ordering \eqref{preference order 2V} of different unit premium rules. We generally do not expect forecast dominance between comparable unit premium rules that are based on the same set of information. Therefore, we typically select one (single) Bregman loss $L_{\varphi}$, and the discussion around \eqref{Tweedie choices} can support us in this selection.

\medskip

There is one more critical point to be considered, namely, the implied preference order \eqref{preference order 2V} should be scale invariant, i.e., it should not depend on the currency 
\begin{equation}\label{scale invariant}
\widehat{\E}_{\q}[L_\varphi(Y,\Pi^a)] \le \widehat{\E}_{\q}[L_\varphi(Y,\Pi^b)]
\quad \Longleftrightarrow \quad
\widehat{\E}_{\q}[L_\varphi(cY,c\Pi^a)] \le \widehat{\E}_{\q}[L_\varphi(cY,c\Pi^b)],
\end{equation}
for any $c>0$. This is to say, that model selection should be unit-free as we want to select the same forecast model under Euros or Dollars.
The scale invariance \eqref{scale invariant} is not generally given. However, if we select a loss function from the Patton family \cite{Patton}, given by \eqref{Tweedie choices}, we observe by a change of variable
\begin{eqnarray}\nonumber
L_\varphi(cy,cx)&=&\1_{\{cy> cx\}}\int_{cx}^{cy} \frac{cy-\theta}{\theta^p}\,\dd \theta+
\1_{\{cy<cx\}}\int_{cy}^{cx} \frac{\theta-cy}{\theta^p}\,\dd \theta
  \\&=&\label{homogeneous order 1}
        c^{2-p} \,L_\varphi(y,x)
.
\end{eqnarray}
Hence, the Patton family \cite{Patton} is positively homogeneous of degree $2-p$. This guarantees that the preference order \eqref{scale invariant} is preserved under different currencies. In fact, this motivation is rather similar to Patton \cite[Proposition 4]{Patton} and it also gives support to consider Tweedie's dominance of Denuit et al.~\cite{DenuitEtAl2021}; Tweedie's class of deviance losses is a subclass of the Patton family \eqref{Tweedie choices} because Tweedie's class does not exist for $p\in (0,1)$; see J{\o}rgensen \cite[Theorem 2]{Jorgensen2}.

\section{Murphy's decomposition}
\label{sec: Murphy's decomposition}
\subsection{Murphy's decomposition and Bregman score}
Since the (strict) consistency of the Bregman loss \eqref{Bregman loss definition} does not depend on the specific choice of the probability law, it also carries over to conditional probabilities.

\begin{lemma} \label{Pythagoras}
Consider an information set ${\cal A}\subset {\cal F}$ and define the conditional mean 
\begin{equation*}
m_{\q}({\cal A})=\E_{\q}[Y\mid {\cal A}].
\end{equation*}
Assume $X$ is ${\cal A}$-measurable.
There is the conditional Bregman Pythagorean relation, 
\begin{equation}\label{eq:conditional-pythagorean}
 \E_{\q}\left[ L_{\varphi}\left(Y,X\right)\mid{\cal A}\right]
  =
  \E_{\q}\left[ L_{\varphi}\left(Y,m_{\q}({\cal A})\right)\mid{\cal A}\right]
  + L_{\varphi}\left(m_{\q}({\cal A}),X\right) \qquad \text{ a.s.}
\end{equation}
\end{lemma}
This lemma is proved in the appendix.

\medskip

For a strictly convex generator $\varphi$, 
the last term in \eqref{eq:conditional-pythagorean} vanishes if and only if 
$X=m_{\q}({\cal A})$. 
We translate this to the $\q$-calibration considerations \eqref{Q calibration}. Assume that the information set ${\cal A}=\sigma(\Pi)$ is generated by the price $\Pi$. 

\medskip
\begin{graybox}
The conditional Bregman Pythagorean relation \eqref{eq:conditional-pythagorean} yields
\begin{equation}\label{conditional Bregman}
 \E_{\q}\left[ L_{\varphi}\left(Y,\Pi\right)\mid \Pi\right]
  =
  \E_{\q}\left[ L_{\varphi}\left(Y,m_{\q}(\Pi)\right)\mid \Pi\right]
  + L_{\varphi}\left(m_{\q}(\Pi),\Pi\right).
\end{equation}
\end{graybox}
\medskip

This relation is crucial. It measures the accuracy of the price $\Pi$ for predicting the claim $Y$ under a Bregman loss $L_\varphi$ with strictly convex generator $\varphi$. This accuracy is decomposed into two terms: (1) the accuracy of the recalibrated price $m_{\q}(\Pi)$, and (2) the discrepancy between the price $\Pi$ and its recalibrated version $m_{\q}(\Pi)$. The first term (1) reflects discrimination and the second term (2) calibration. This second term vanishes for $\q$-calibrated prices $\Pi$. Note that these considerations have already been used in Remark \ref{binned Murphy} where we studied the $\Pi$-binned version of the Murphy diagram.

\medskip

{\it Murphy's decomposition} \cite{Murphy} modifies identity \eqref{conditional Bregman} in two ways. First, it considers the population version by taking expected values w.r.t.~the exposure-weighted measure $\q$. It calibrates the consideration to the global mean $\mu_{\q}=\E_{\q}[Y]$.

\bigskip
\begin{graybox}
\begin{cor}[Murphy's decomposition] \label{cor Murphy}
Select a convex generator $\varphi$. Murphy's decomposition is given by
\begin{eqnarray}
\E_{\q}[L_\varphi(Y, \Pi)]
&=&\label{Murphy}
\underbrace{\E_{\q}[L_\varphi(Y, \mu_{\q})]}_{=:\mathrm{UNC}_\varphi}
-
\underbrace{\E_{\q}[L_\varphi(m_{\q}(\Pi), \mu_{\q})]}_{=:\mathrm{RES}_\varphi(\Pi)}
+
\underbrace{\E_{\q}[L_\varphi(m_{\q}(\Pi),\Pi)]}_{=:\mathrm{MCB}_\varphi(\Pi)}.
\end{eqnarray}
\end{cor}
\end{graybox}
\medskip

This corollary is proved in the appendix.

\bigskip

The expected Bregman loss on the left-hand side of \eqref{Murphy} is decomposed into an uncertainty term (UNC), a resolution term (RES) and a miscalibration term (MCB). UNC measures the total fluctuations contained in the response $Y$ (relative to its deterministic mean), RES measures the resolution (discrimination) that can be achieved by the calibrated version of the unit premium $\Pi$, and MCB quantifies the calibration error.
All three terms are non-negative, and they vanish for a strictly convex generator $\varphi$ if and only if
\begin{align*}
\mathrm{RES}_\varphi(\Pi)=0 &~\iff~ m_{\q}(\Pi)=\mu_{\q}, \qquad \text{a.s.},\\
\mathrm{MCB}_\varphi(\Pi)=0 &~\iff~ m_{\q}(\Pi)=\Pi, \qquad \text{a.s.}
\end{align*}

We turn Murphy's decomposition into a {\it score} where bigger means better, and so that it is calibrated to the global mean $\mu_{\q}$.

\medskip
\begin{graybox}
\begin{defi}[Bregman score]
We define the Bregman score by
\begin{eqnarray}\nonumber
S_\varphi(Y, \Pi)&=&
\E_{\q}[L_\varphi(Y, \mu_{\q})]- 
\E_{\q}[L_\varphi(Y, \Pi)]
\\&=&\label{Bregman score decomposition}
\E_{\q}[L_\varphi(m_{\q}(\Pi), \mu_{\q})]
-\E_{\q}[L_\varphi(m_{\q}(\Pi), \Pi)]
\\&=&\nonumber
\mathrm{RES}_\varphi(\Pi)-\mathrm{MCB}_\varphi(\Pi).
\end{eqnarray}
\end{defi}
\end{graybox}
\medskip

The resulting score system is anchored at $S_\varphi(Y, \mu_{\q})=0$, and it quantifies the two important terms of forecast accuracy:
\begin{itemize}
\item \underline{Resolution:} The resolution term measures how well a premium rule $\Pi$ can discriminate the response $Y$ by quantifying the accuracy gain of the $\q$-calibrated version $m_{\q}(\Pi)$ relative to the global mean $\mu_{\q}$. 
\item \underline{Calibration:} The calibration term measures how well the premium rule $\Pi$ is calibrated to the response $Y$ by quantifying the miscalibration. 
\end{itemize}

\medskip

We come back to the preference order \eqref{preference order 2V}.

\medskip
\begin{graybox}
Assume we have two pricing schemes $(Y, V, \Pi^a)$ and $(Y, V, \Pi^b)$ and a given convex generator $\varphi$. This gives the preference order
\begin{eqnarray}\nonumber
\text{prefer $\Pi^a$ over $\Pi^b$ for $Y$} &\quad \Longleftrightarrow
\quad& S_\varphi(Y, \Pi^a) \ge S_\varphi(Y, \Pi^b)
\\\label{preference order 1}
&\quad \Longleftrightarrow \quad&
\E_{\q}[L_\varphi(Y,\Pi^a)] \le \E_{\q}[L_\varphi(Y,\Pi^b)]
\\\nonumber
&\quad \Longleftrightarrow \quad&
\E[V\,L_\varphi(Y,\Pi^a)] \le \E[V\,L_\varphi(Y,\Pi^b)].
\end{eqnarray}
\end{graybox}
\medskip

The computation of this preference order can be done by its sample version
\eqref{empirical Bregman loss} on the test sample ${\cal T}$. This yields
\begin{eqnarray}\label{empirical scoring}
\widehat{S}_\varphi(Y, \Pi)
&=&\widehat{\E}_{\q}[L_\varphi(Y,\mu_{\q})]-\widehat{\E}_{\q}[L_\varphi(Y,\Pi)]
\\&=&\nonumber
\frac{1}{\sum_{i=1}^n V_i}\, \sum_{i=1}^n V_i\,\Big(
L_\varphi(Y_i,\mu_{\q})
-L_\varphi(Y_i,\Pi_i)\Big).
\end{eqnarray}
We assume that the premium rule $\Pi$ and the global mean $\mu_{\q}$ have been determined on an independent learning sample, and \eqref{empirical scoring} reflects an out-of-sample score.

\begin{example}[Gamma deviance scoring]\normalfont
\label{example gamma deviance}
We compute the preference order \eqref{preference order 1} for our synthetic data example introduced in Section \ref{sec: stylized example}. We compute the sample version \eqref{empirical scoring} based on the test sample ${\cal T}$ of sample size $n=100,000$; this is the identical test data that has been used in Section \ref{Graphical tools: empirical versions}. The responses of this data have been generated by conditional gamma distributions \eqref{gamma responses}. This makes it natural to use the gamma deviance loss from the Bregman loss family. The gamma deviance loss corresponds to the inverse quadratic scaling $\varphi''(\theta)=2/\theta^{2}$ in \eqref{Tweedie choices} and it respects 
the currency invariance \eqref{scale invariant} in preference ordering.

We first derive the gamma deviance loss before providing the numerical results of the premium rules $\Pi$ and $m_{\q}(\Pi)$.
We select for $x>0$
\begin{equation}\label{gamma deviance choice}
\varphi(x) = -2\log x.
\end{equation}
This yields first and second derivatives on $\R_+$
\begin{equation*}
\varphi'(x) = -\frac{2}{x} \qquad \text{ and } \qquad
\varphi''(x) = \frac{2}{x^2}>0.
\end{equation*}
Thus, we have a strictly convex generator $\varphi$ on $\R_+$ and its Bregman loss is given by
\begin{equation}\label{gamma deviance loss}
L_\varphi(y,x)= -2\log y + 2\log x + \frac{2}{x}\left(y-x\right)
= 2 \left( \frac{y-x}{x} - \log\left(\frac{y}{x}\right) \right).
\end{equation}
This Bregman loss \eqref{gamma deviance loss} is the gamma deviance loss. The gamma deviance loss is obtained from the EDF by selecting the cumulant function $\kappa(\vartheta)=-\log(-\vartheta)$, for $\vartheta<0$. This cumulant function generates the gamma distribution, it has canonical link $(\kappa')^{-1}(m)=-1/m$, for $m>0$, and variance function ${\cal V}(m)=\kappa''((\kappa')^{-1}(m))=m^2$; see W\"uthrich--Merz \cite[Section 2.2.2]{WM2023}. The power parameter of $p=2$ of the variance function ${\cal V}$ exactly refers to parameter $p$ of the Patton family \eqref{Tweedie choices}.

\begin{table}[htb!]
\begin{center}
{\small
\begin{tabular}{|l|rr|}
\hline
& \multicolumn{1}{|c}{Bregman} & \multicolumn{1}{c|}{Bregman score}  \\
premium rules & loss \eqref{gamma deviance loss} & \eqref{empirical scoring} in $10^{-3}$\\
\hline\hline
global mean model $\mu_{\q}$ & 0.31928 & -- \\
premium rule $\Pi$ & 0.31901 & 0.2743 \\
premium rule $m_{\q}(\Pi)$ & 0.31778 & 1.5017\\
\hline
\end{tabular}}
\end{center}
\caption{Gamma Bregman loss $\widehat{\E}_{\q}[L_\varphi(Y,\cdot)]$ and Bregman score $\widehat{S}_\varphi(Y, \cdot)$ of the considered premium rules under the gamma deviance loss choice \eqref{gamma deviance choice} for the generator $\varphi$.}
\label{gamma deviance example table}
\end{table}

Table \ref{gamma deviance example table} presents the resulting scores, and we give clear preference to the recalibrated premium rule $m_{\q}(\Pi)$ ($1.5017\cdot 10^{-3}$) over the original premium rule $\Pi$ ($0.2743\cdot 10^{-3}$). In view of the previous graphs and results this is not surprising, as $\Pi$ has a serious calibration issue. In fact, the miscalibration term is zero for the recalibrated price $m_{\q}(\Pi)$, this follows from \eqref{conditional Bregman}. This motivates the estimation
\begin{equation}\label{miscalibration estimation 00}
\widehat{\mathrm{MCB}}_\varphi(\Pi)=
\widehat{S}_\varphi(Y, m_{\q}(\Pi))-
\widehat{S}_\varphi(Y, \Pi).
\end{equation}
This 
immediately yields the miscalibration error estimate of $\Pi$ for $Y$
\begin{equation*}
\widehat{\mathrm{MCB}}_\varphi(\Pi) ~=~1.2274 \cdot 10^{-3}.
\end{equation*}
This is the $\q$-calibration defect of the unit premium rule $\Pi$ for $Y$ measured under the gamma deviance loss.
\EndExample
\end{example}

\medskip

The previous example leaves us with two open questions:
\begin{itemize}
\item[(1)] The previous example essentially benefits from the fact that we can explicitly compute the recalibrated unit premium $m_{\q}(\Pi)$. Only this allows us to obtain Murphy's decomposition of the Bregman score into the resolution term and the miscalibration term, in particular, this allows us to compute \eqref{miscalibration estimation 00}.  What can we do in a real-world situation where the recalibrated unit premium cannot be computed explicitly?
\item[(2)] Table \ref{gamma deviance example table} gives a clear preference to one of the two premium rules in terms of \eqref{preference order 1}. Can we understand this numerical result graphically, e.g., identifying the weaknesses of the premium rules?
\end{itemize}

The first question (1) needs estimation of the corresponding split. The second question (2) is already partly answered by the AvE plots of Figure \ref{AvE Murphy 1} and we will provide a different perspective on the same results.

\subsection{Graphical illustration of the Bregman score}
Having two unit premium rules $\Pi^a$ and $\Pi^b$, we aim at better understanding how they form the Bregman score estimate \eqref{empirical scoring}. We again select one of the two unit premium rules, say $\Pi^a$, for an ordering and, thus, the following plots will be asymmetric if applied to both unit premium rules. Let $([i]_a)_{i=1}^n$ denote the ordered sequence w.r.t.~$(\Pi_i^{a})_{i=1}^n$, that is, 
$\Pi_{[i]_a}^{a} \le  \Pi_{[i+1]_a}^{a}$ for all $i=1,\ldots, n-1$.
This motivates to consider the paired iterative Bregman score aggregation graphs
\begin{eqnarray}\label{iterative Bregman 1}
m \in \{1,\ldots, n\} & \mapsto &
\frac{1}{\sum_{i=1}^n V_i}\, \sum_{i=1}^m V_{[i]_a}\,\Big(
L_\varphi(Y_{[i]_a},\mu_{\q})
-L_\varphi(Y_{[i]_a},\Pi_{[i]_a}^{\textcolor{red}{a}})\Big),
\\\label{iterative Bregman 2}
m \in \{1,\ldots, n\} & \mapsto &
\frac{1}{\sum_{i=1}^n V_i}\, \sum_{i=1}^m V_{[i]_a}\,\Big(
L_\varphi(Y_{[i]_a},\mu_{\q})
-L_\varphi(Y_{[i]_a},\Pi_{[i]_a}^{\textcolor{red}{b}})\Big).
\end{eqnarray}
The only difference is the last premium rule indicated by red color. Note that we consider a {\it paired} graph in \eqref{iterative Bregman 1}-\eqref{iterative Bregman 2}, meaning that the aggregation order is the same for both unit premium rules.

\medskip

\begin{figure}[htb!]
\begin{center}
\begin{minipage}[t]{0.45\textwidth}
\begin{center}
\includegraphics[width=\textwidth]{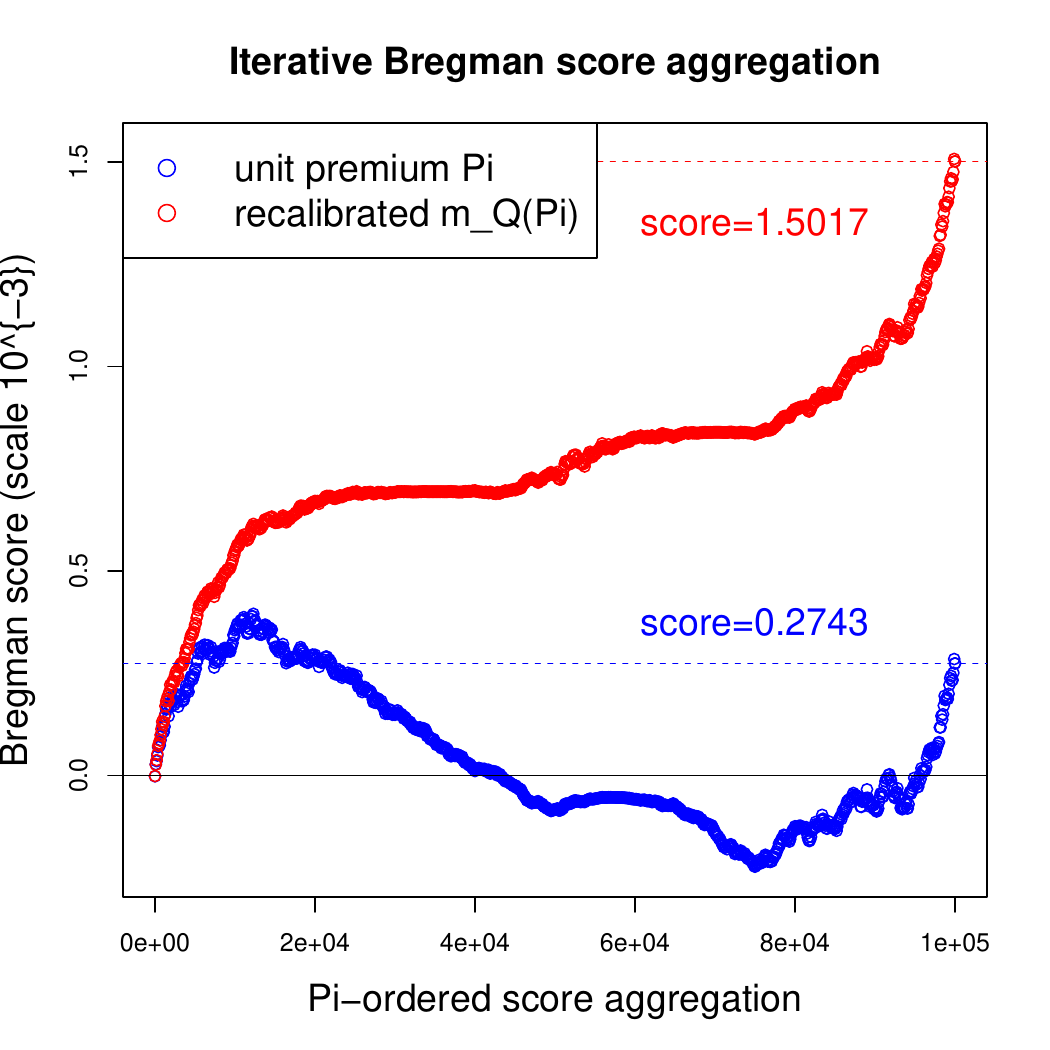}
\end{center}
\end{minipage}
\begin{minipage}[t]{0.45\textwidth}
\begin{center}
\includegraphics[width=\textwidth]{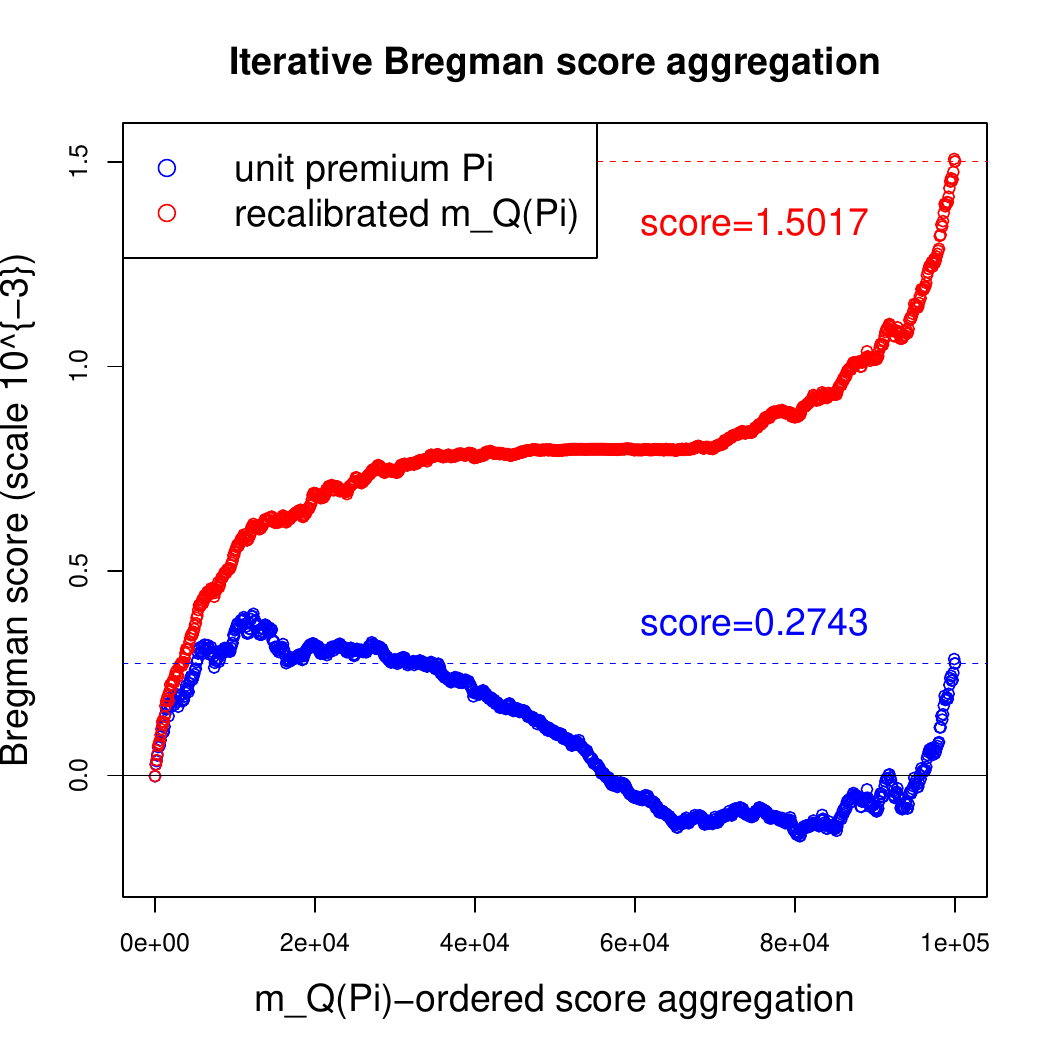}
\end{center}
\end{minipage}
\end{center}
\vspace{-.7cm}
\caption{Paired iterative Bregman score aggregation from smallest to biggest unit premium: (lhs) ordered w.r.t.~unit premium $(\Pi_i)_{i=1}^n$ and (rhs) ordered w.r.t.~recalibrated unit premium $(m_{\q}(\Pi_i))_{i=1}^n$.}
\label{Ordered Bregman aggregation}
\end{figure}

Figure \ref{Ordered Bregman aggregation} shows the paired iterative Bregman score aggregation \eqref{iterative Bregman 1}-\eqref{iterative Bregman 2}, on the left-hand side ordered w.r.t.~the unit premium $(\Pi_i)_{i=1}^n$ and on the right-hand side w.r.t.~the recalibrated unit premium $(m_{\q}(\Pi_i))_{i=1}^n$. The final value for $m=n$ precisely gives the Bregman scores of Table \ref{gamma deviance example table}.

These curves are increasing if the unit premium prediction is more accurate than the global mean $m_{\q}$, and decreasing otherwise. We observe that for $m_{\q}(\Pi)$ the curve is generally increasing with flat pieces where the recalibrated unit premium $m_{\q}(\Pi)$ takes roughly the same value as the global mean $m_{\q}$. On the right-hand side this flat piece is in the middle of the graph, on the left-hand side the flat piece is partitioned into two parts by the wrong risk ordering of $\Pi$ on the $x$-axis. On the other hand, the original unit premium $\Pi$ has many predictions that perform worse than the global mean $m_{\q}$ (negative slopes in blue curves), which clearly shows that the unit premium $\Pi$ is not an accurate predictor. This gives a graphical illustration how the Bregman score of Table \ref{gamma deviance example table} is composed across its unit premium range, and the downward slopes allow us to identify premium ranges with inaccurate predictions.

\subsection{Sample computation of resolution and miscalibration}
The next problem that we consider is Murphy's decomposition into resolution and miscalibration terms of the Bregman score in the case the recalibrated premium $m_{\q}(\Pi)$ is unknown, which is the typical situation in applications. Recall \eqref{Bregman score decomposition}. This gives us two different ways of expressing the miscalibration term
\begin{eqnarray*}
\mathrm{MCB}_\varphi(\Pi)&=&\E_{\q}[L_\varphi(m_{\q}(\Pi), \Pi)],
\\
\mathrm{MCB}_\varphi(\Pi)&=&S_\varphi(Y, m_{\q}(\Pi))-S_\varphi(Y, \Pi)
.
\end{eqnarray*}
The goal is to use these two representations for deriving an estimation.
Assume we have estimates $\widehat{m}_{\q}(\Pi_i)$, we can use both of the two identities to receive a miscalibration term estimate
\begin{eqnarray}
\label{MCB1}
\widehat{\mathrm{MCB}}^{(1)}_\varphi(\Pi)&=&\frac{1}{\sum_{i=1}^n V_i}\, \sum_{i=1}^n V_i\,
L_\varphi(\widehat{m}_{\q}(\Pi_i),\Pi_i),
\\\label{MCB2}
\widehat{\mathrm{MCB}}^{(2)}_\varphi(\Pi)&=&\frac{1}{\sum_{i=1}^n V_i}\, \sum_{i=1}^n V_i\,\Big(
L_\varphi(Y_i,\Pi)
-L_\varphi(Y_i,\widehat{m}_{\q}(\Pi_i))\Big).
\end{eqnarray}
We present these two different formulas \eqref{MCB1}-\eqref{MCB2} because they present different viewpoints, and they give different results on finite samples (they are identical in the population version). The first one \eqref{MCB1} does not use the responses for given $\widehat{m}_{\q}(\Pi_i)$, whereas the second one \eqref{MCB2} does. Consequently, the two estimates are expected to differ. From this viewpoint, one might always prefer the first one because it does not involve the noisy part of the responses. However, the accuracy of \eqref{MCB1} crucially depends on the accuracy that we can get in the (empirical) recalibration step yielding the estimates $\widehat{m}_{\q}(\Pi_i)$. This is the critical step in the empirical version of Murphy's decomposition.  Usually, one uses the isotonic regression
\eqref{isotonic regression 0} on the test sample ${\cal T}$ for the recalibration step, which yields estimates
\begin{equation*}
\widehat{m}_{\q}(\Pi_i)= (\widehat{\bm}_{\q}^{\rm iso})_i 
\qquad \text{ for $i=1,\ldots, n$.}
\end{equation*}
The results in the following list illustrate that this isotonic regression step underestimates the miscalibration error (0.9934 vs.~1.2274) in the first version \eqref{MCB1}, because the isotonic regression estimate is comparably crude and tends to be too close to the unit premium $\Pi_i$:
\begin{eqnarray}\nonumber
\mathrm{MCB}_\varphi(\Pi)&=& 1.2274,
\\\label{MCB numbers}
\widehat{\mathrm{MCB}}^{(1)}_\varphi(\Pi)&=& 0.9934,
\\\nonumber
\widehat{\mathrm{MCB}}^{(2)}_\varphi(\Pi)&=& 1.2591.
\end{eqnarray}
On the other hand, the second version \eqref{MCB2} overestimates the miscalibration error (1.2591 vs.~1.2274). This is due to an in-sample bias, because we use the test sample ${\cal T}$ to fit the isotonic regression \eqref{isotonic regression 0} and we use the same observations to evaluate \eqref{MCB2}. Of course, we could mitigate the last difficulty if we had additional (independent) observations or by cross-validation. In our numerical example \eqref{MCB numbers}, the true value is in between the two estimates, however, we do not know whether this holds more generally or only in this example.

\subsection{Illustration of resolution and miscalibration}
\label{Illustration of resolution and miscalibration}
We come back to Remark \ref{binned Murphy}. Consider $K \in \N$ bins, and
denote by $B$ the bin allocation of a randomly selected instance $(Y,V,\Pi)$ among the $K$ quantile bins under the $\q$-distribution $G_{\Pi}^{\q}$. Recall  
\begin{equation*}
m_{\q}(B) = \E_{\q}\left[ Y \mid B \right]
\qquad \text{ and } \qquad 
\pi_B = \E_{\q}\left[ \Pi \mid B \right].
\end{equation*}
The conditional Bregman Pythagorean relation yields, see Lemma \ref{Pythagoras},
\begin{equation*}
 \E_{\q}\left[ L_{\varphi}\left(Y, \pi_B\right)\mid B\right]
  =
  \E_{\q}\left[ L_{\varphi}\left(Y,m_{\q}(B)\right)\mid B\right]
  + L_{\varphi}\left(m_{\q}(B), \pi_B\right) \qquad \text{ a.s.}
\end{equation*}
Taking the $\E_{\q}$-expectation proves \eqref{binning to be proved}, and it yields the following Murphy's decomposition, the proof is identical to the one of Corollary \ref{cor Murphy},
\begin{eqnarray*}
\E_{\q}\left[ L_{\varphi}(Y, \pi_B)\right]
&=& \E_{\q}\left[ L_{\varphi}(Y,m_{\q}(B))\right]
    +\E_{\q}\left[ L_{\varphi}(m_{\q}(B), \pi_B)\right]
  \\&=&\E_{\q}\left[ L_{\varphi}(Y,\mu_{\q})\right]
-\E_{\q}\left[ L_{\varphi}(m_{\q}(B), \mu_{\q})\right]
+\E_{\q}\left[ L_{\varphi}(m_{\q}(B), \pi_B)\right]        
        .
\end{eqnarray*}
The first term on the first line on the right-hand side is the irreducible within-bin variation produced by the responses $Y$ around $m_{\q}(B)$, and the second term gives {\it a} miscalibration error, by using $\pi_B$ instead of the calibrated version $m_{\q}(B)$. There is a subtle difference here to Murphy's decomposition \eqref{Murphy}, namely, we consider $B$-partitioning and we do not recalibrate $\pi_B$ resulting in the calibrated version $m_{\q}(B)$. Therefore, the second line does not present the classic Murphy's decomposition \eqref{Murphy}, but a variant that uses $B$-binning.

\medskip

\begin{figure}[htb!]
\begin{center}
\begin{minipage}[t]{0.45\textwidth}
\begin{center}
\includegraphics[width=\textwidth]{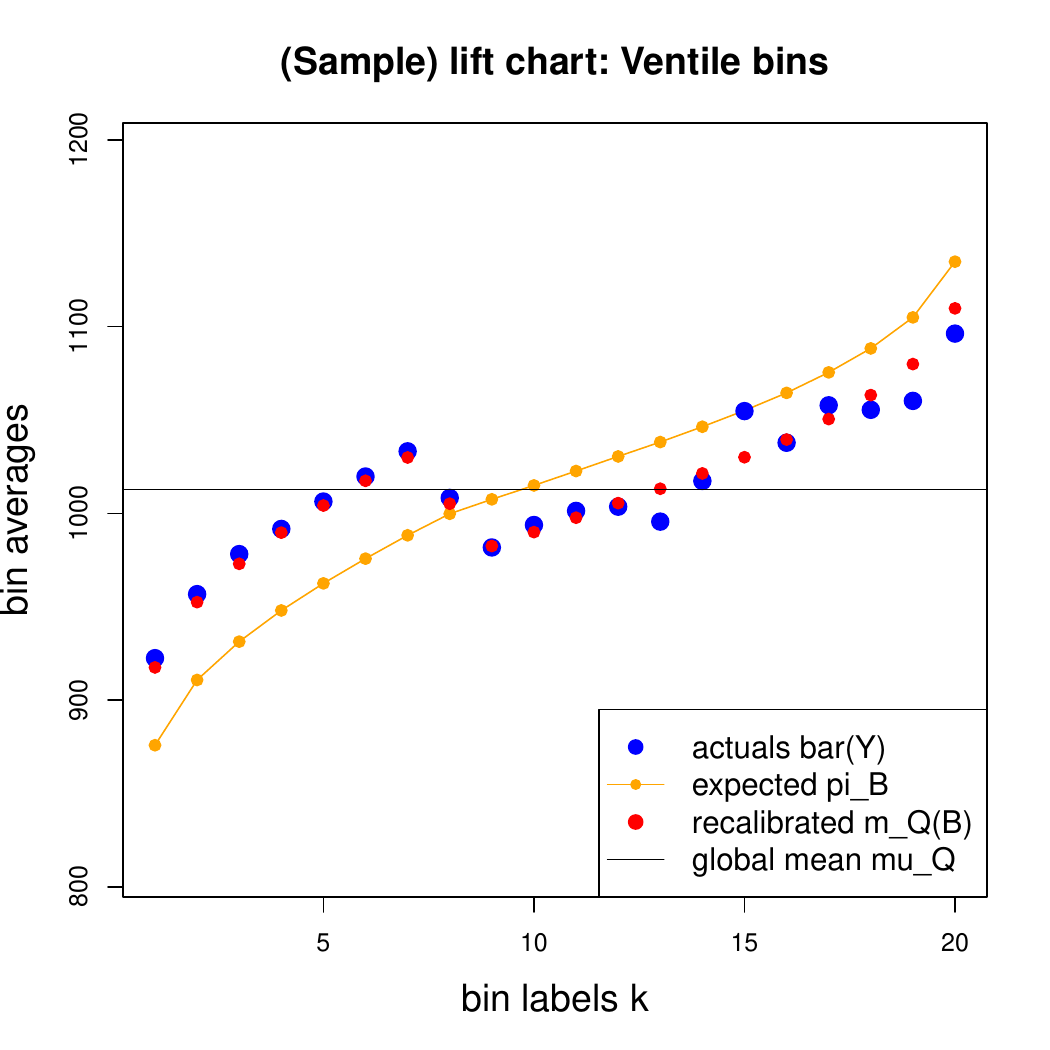}
\end{center}
\end{minipage}
\begin{minipage}[t]{0.45\textwidth}
\begin{center}
\includegraphics[width=\textwidth]{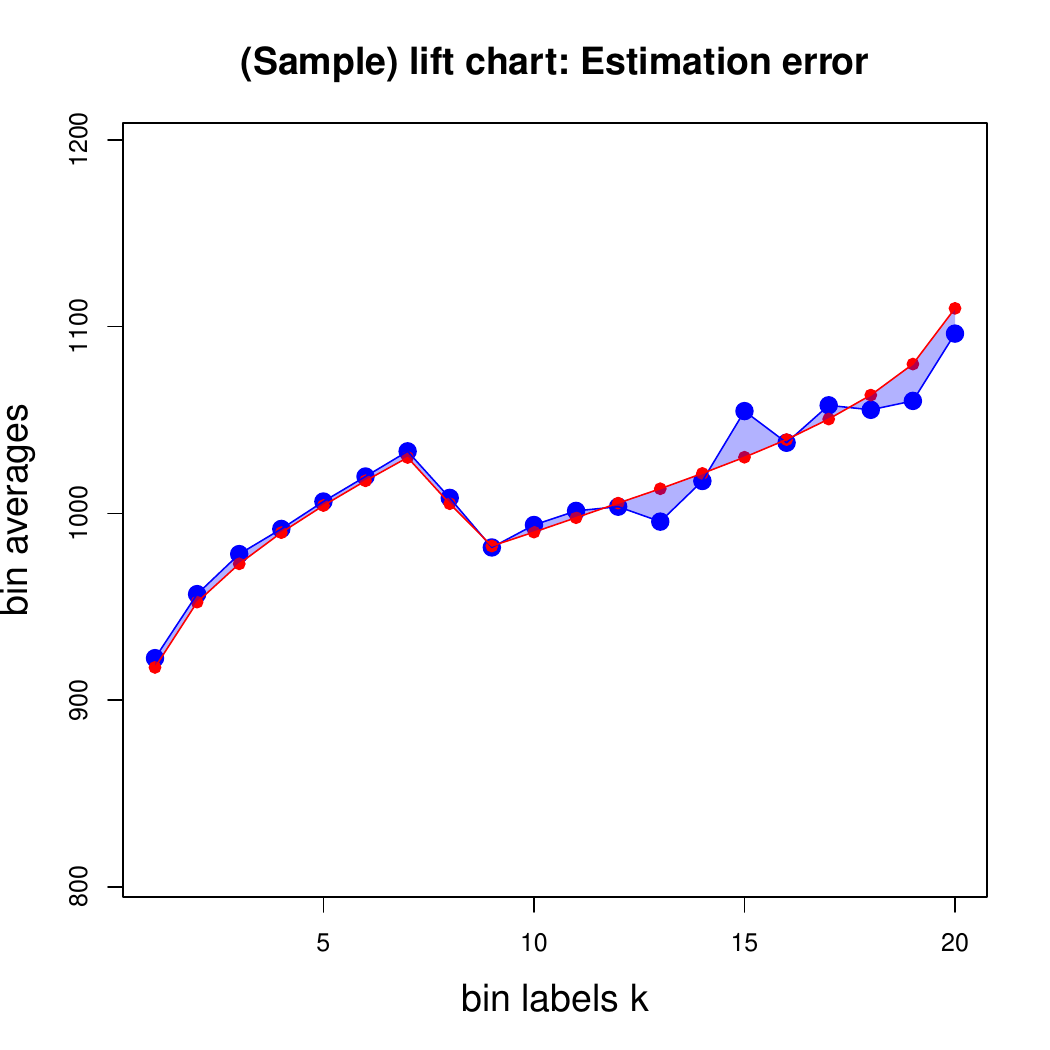}
\end{center}
\end{minipage}
\begin{minipage}[t]{0.45\textwidth}
\begin{center}
\includegraphics[width=\textwidth]{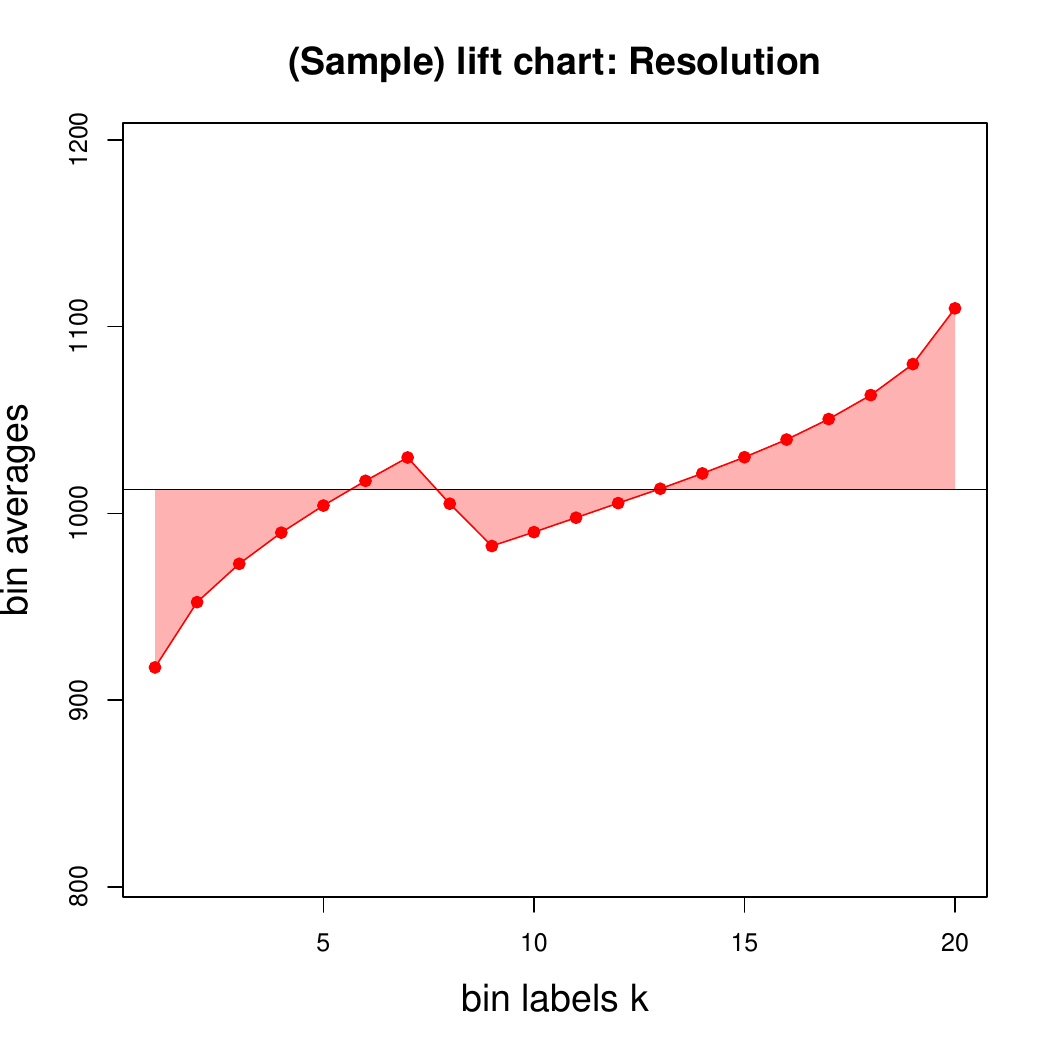}
\end{center}
\end{minipage}
\begin{minipage}[t]{0.45\textwidth}
\begin{center}
\includegraphics[width=\textwidth]{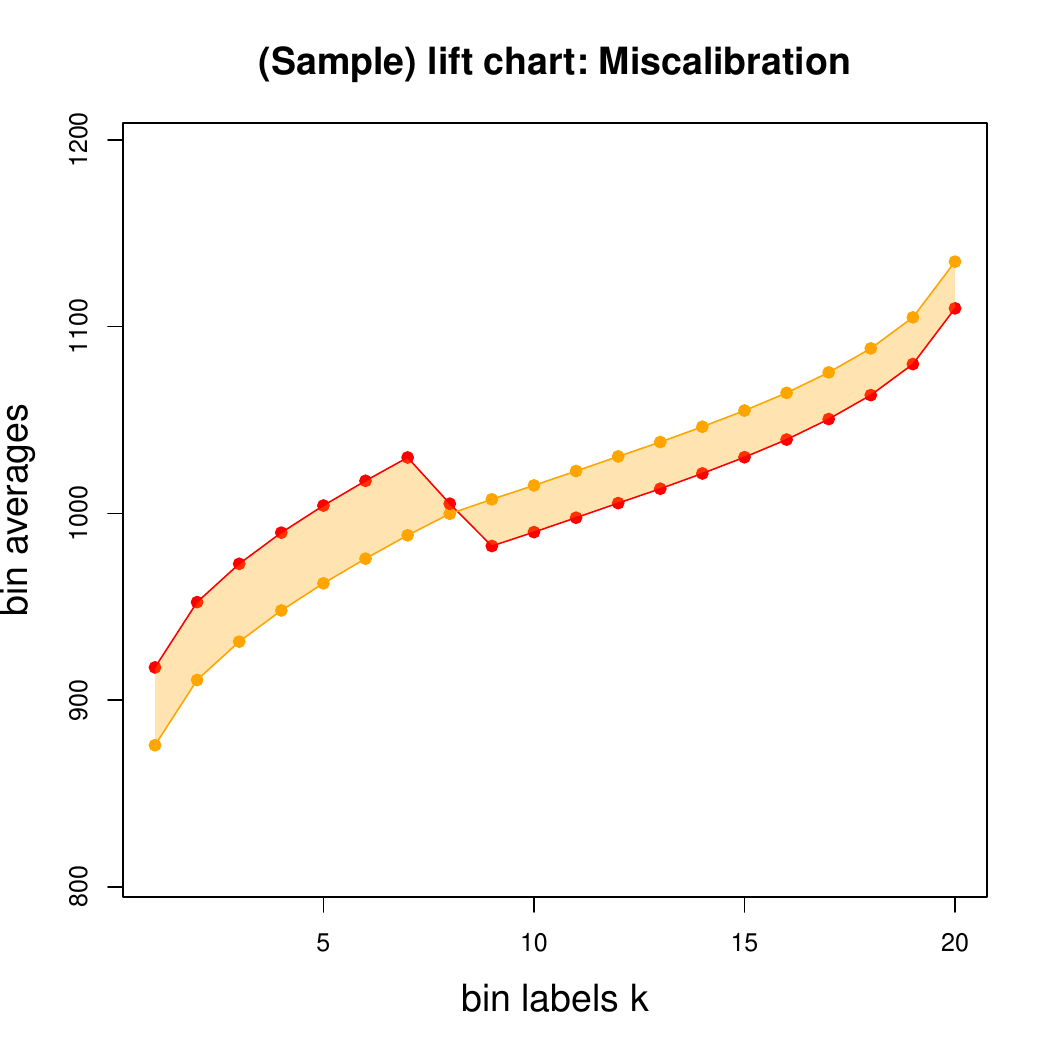}
\end{center}
\end{minipage}
\end{center}
\vspace{-.7cm}
\caption{(top-lhs) Lift chart using ventile binning $K=20$, (top-rhs) illustration of estimation error
  $\overline{Y}_B - m_{\q}(B)$, (bottom-lhs) visualization of resolution gain over the global mean model $\mu_{\q}$, and (bottom-rhs) visualization of miscalibration error (raw differences before they enter the Bregman loss).}
\label{Resolution plot}
\end{figure}

This then inspires the following plots of Semenovich--Dolman \cite{SemenovichDolman2020}:
\begin{itemize}
\item Figure \ref{Resolution plot} (top-lhs) is precisely the sample lift chart illustrated in Figure \ref{LiftChart Bins 1}; for better visualization we use ventile binning $K=20$.
Moreover, we adapt the notation to the present section using the binning indicator $B$, and we add to this lift chart the global mean $\mu_{\q}=\E_{\q}[Y]$ (black horizontal line) as well as the true calibrated bin means $m_{\q}(B)$ (red dots).
\item
Figure \ref{Resolution plot} (top-rhs) shows the estimation error $\overline{Y}_B - m_{\q}(B)$ comparing the sample means $\overline{Y}_B$ on the bins $B\in \{1,\ldots, K\}$, see \eqref{binning procedure 1}, to their recalibrated population counterparts $m_{\q}(B)$. This is the unavoidable estimation error of the recalibration step on the bins that results from the irreducible risk, i.e., this is the error we collect under an unknown population distribution.
\item
Figure \ref{Resolution plot} (bottom-lhs) shows the resolution term (raw differences) comparing the calibrated means $m_{\q}(B)$ to the global mean $\mu_{\q}$. This is the (calibrated) price granularity that we can get on the selected bins, resulting in the corresponding discrimination and risk classification. The red area shows the cross-subsidy  if one would use the egalitarian price $\mu_{\q}$ instead.
\item
Figure \ref{Resolution plot} (bottom-rhs) shows the miscalibration error (raw differences) if using $\pi_B$ instead of the calibrated version $m_{\q}(B)$ obtained on the $B$-binning granularity.
\end{itemize}

The illustration in Figure \ref{Resolution plot} shows resolution and calibration in the lift chart. The binning was performed w.r.t.~the premium scheme $\Pi$. However, the same analysis applies to any binnning of the insurance portfolio, and this relates to the original work in the 1960's about an optimal risk classification; see Bailey--Simon \cite{BaileySimon}, Bailey \cite{Bailey} and Jung \cite{Jung}. In a more modern view, such a risk classification should also be void of unfair discrimination; see Lindholm et al.~\cite{DFIP}.

\section{Gini score}
\label{section Gini score}
\subsection{Introduction: Risk ranking}
In the previous section, the discussion has been centered around {\it calibration} and {\it discrimination}. A couple of times, we made some statements about {\it risk ranking}, in particular, the isotonic regression is based on such a (correct) risk ranking. The present section discusses the question of accurate risk rankings. A popular risk ranking measure is the {\it Gini score} which goes back to the Gini index \cite{Gini0, Gini}  in economics that was used to study disparity of wealth distributions within different populations; the Gini score is based on the Lorenz curve \cite{Lorenz}. This Gini concept has been adopted in statistical modeling for risk ranking assessments, and it has also become a popular model validation tool in credit scoring and actuarial modeling; see Gourieroux--Jasiak \cite{GourierouxJasiak}, Frees et al.~\cite{Frees1, Frees2} and Denuit et al.~\cite{DenuitSznajderTrufin, DenuitTrufin}.
The following outline is based on Brauer--W\"uthrich \cite{Brauer1}, who discussed the Gini score under exposures (case weights) and ties in the data, and we directly focus on the sample version of the Gini score.

\medskip

The main question we study in this section is whether the order statistics of the unit premiums $(\Pi_i)_{i=1}^n$ and the observed loss costs $(Y_i)_{i=1}^n$ are aligned. This consideration is {\it purely rank based}, and any strictly monotonically increasing transformation of the unit premium provides the same result. Thus, the following considerations do not assess calibration and they also are not based on strictly consistent loss functions for mean estimation; see W\"uthrich \cite{WGini} for more discussion.

\subsection{Cumulative accuracy profile}
We start by solely considering the observations $(Y_i)_{i=1}^n$, and the unit premiums $(\Pi_i)_{i=1}^n$ will be integrated in a second step. This first step constructs the {\it Leimkuhler curve}, which is a mirrored version of the {\it Lorenz curve}. Both curves consider an order statistics of the observations $(Y_i)_{i=1}^n$, the former in decreasing order and the latter in increasing order. 
Consider the decreasing order statistics of the unit loss costs
\begin{equation}\label{suborder 0}
Y_{(1)} \ge Y_{(2)} \ge \ldots \ge Y_{(n)},
\end{equation}
with a deterministic rule if there are ties in the observations $(Y_i)_{i=1}^n$. The lower round brackets $_{(i)}$ in \eqref{suborder 0} indicate that we ordered the responses from biggest to smallest. We map precisely this order to the exposures $(V_{[i]})_{i=1}^n$ and we use square brackets $_{[i]}$ to indicate that this is the implied order from the responses $(Y_{(i)})_{i=1}^n$. For example, $V_{[5]}$ is the exposure of the fifth biggest unit loss costs in \eqref{suborder 0}.
 
We define the increasing sequences (running totals)
\begin{eqnarray}\label{Leimkuhler 1}
m \in \{1,\ldots, n\} & \quad \mapsto \quad& u_m=\frac{1}{\sum_{i=1}^n V_i}\,
\sum_{i=1}^m V_{[i]},
\\\label{Leimkuhler 2}
m \in \{1,\ldots, n\} &  \quad \mapsto \quad & L_m= \frac{1}{\sum_{i=1}^n V_i Y_i}\,
\sum_{i=1}^m V_{[i]} Y_{(i)},
\end{eqnarray}
and we initialize $u_0=L_0=0$ for $m=0$. Note that $(u_m)_{m=0}^n$ 
is strictly increasing and $(L_m)_{m=0}^n$ is non-decreasing, both live in the unit interval.

The first line \eqref{Leimkuhler 1} considers the exposure-weighted sample distribution of the unit loss costs because the exposures are ordered w.r.t.~$(Y_i)_{i=1}^n$; see \eqref{empirical distribution 2} for the premium counterpart. There is one difference though in \eqref{Leimkuhler 1} compared to \eqref{empirical distribution 2}, namely, we disaggregate the ties of 
$(Y_i)_{i=1}^n$ to retain the original cardinality $n$ of the training sample ${\cal T}$. Since below we are going to linearly interpolate, the selected suborder in the ties \eqref{suborder 0} will not impact the results. The second line 
\eqref{Leimkuhler 2} is the exposure-weighted Leimkuhler curve that measures the weighted contributions of the decreasing unit loss costs $(Y_{(i)})_{i=1}^n$ to the total loss costs $\sum_{i=1}^n Z_i=\sum_{i=1}^n V_i Y_i$.

\medskip
\begin{graybox}
The {\it Leimkuhler curve} is obtained by linearly interpolating between the points
\begin{equation}\label{def Leimkuhler curve}
\left(u_m, L_m\right) \qquad \text{ for $m=0,\ldots, n$.}
\end{equation}
\end{graybox}
\medskip

The Leimkuhler curve is a concave curve in the unit square $[0,1]^2$ that connects the two corners $(0,0)$ and $(1,1)$. This Leimkuhler curve is illustrated in cyan color in Figure \ref{figure Gini plots} (lhs). This is the upper benchmark (upper bound) of the Gini score because it considers the perfect ordering of the unit loss costs $(Y_{(i)})_{i=1}^n$, and our goal is to see whether the unit premiums $(\Pi_i)_{i=1}^n$ and $(m_{\q}(\Pi_i))_{i=1}^n$ align with this order.

\medskip

To compute the Gini score, we construct a second curve called the {\it cumulative accuracy profile} (CAP); its mirrored version is also called {\it concentration curve}, see Denuit et al.~\cite{DenuitSznajderTrufin}.  The construction of the CAP slightly differs from the Leimkuhler curve, because now the suborder in the ties of the unit premiums $(\Pi_i)_{i=1}^n$ matters. 

The decreasing order statistics of the unit premiums $(\Pi_i)_{i=1}^n$ is constructed in two steps. In the first step, we order the unit premiums in decreasing order $\Pi_{(1)} \ge \Pi_{(2)} \ge \ldots \ge \Pi_{(n)}$. This order statistics may have ties, say, we may have $\Pi_{(k)}=
\Pi_{(k+1)}= \ldots = \Pi_{(k+l)}$. In such ties we consider two suborders implied by the corresponding responses. The first suborder
\begin{equation}\label{suborder a}
\Pi_{(1 \downarrow)} \ge \Pi_{(2\downarrow)} \ge \ldots \ge \Pi_{(n\downarrow)},
\end{equation}
is implied by ordering the ties of the unit premiums in a decreasing order w.r.t.~the responses $(Y_i)_{i=1}^n$, and equivalently in increasing order w.r.t.~the responses $(Y_i)_{i=1}^n$ denoted by
\begin{equation}\label{suborder b}
\Pi_{(1 \uparrow)} \ge \Pi_{(2\uparrow)} \ge \ldots \ge \Pi_{(n\uparrow)}.
\end{equation}
The first suborder \eqref{suborder a} in the ties considers the most favorable suborder to align the ordering of the unit premiums with the responses, and the second suborder
\eqref{suborder b} is the least favorable one. In absence of ties in $(\Pi_i)_{i=1}^n$, the order statistics \eqref{suborder a} and \eqref{suborder b} are identical.

We then map these two orderings to the exposures $(V_{[i\downarrow]})_{i=1}^n$ and $(V_{[i\uparrow]})_{i=1}^n$, and to the unit loss costs $(Y_{[i\downarrow]})_{i=1}^n$ and $(Y_{[i\uparrow]})_{i=1}^n$.
This gives us the running total sequences for $m \in \{1,\ldots, n\}$
\begin{eqnarray}\label{Gini 1b}
 u^{\downarrow}_m=\frac{1}{\sum_{i=1}^n V_i}\,
\sum_{i=1}^m V_{[i\downarrow]} & \text{ and }&
u^{\uparrow}_m=\frac{1}{\sum_{i=1}^n V_i}\,
\sum_{i=1}^m V_{[i\uparrow]},
\\\label{Gini 2b}
 C^{\downarrow}_m= \frac{1}{\sum_{i=1}^n V_i Y_i}\,
\sum_{i=1}^m V_{[i\downarrow]} Y_{[i\downarrow]}
& \text{ and }&
C^{\uparrow}_m= \frac{1}{\sum_{i=1}^n V_i Y_i}\,
\sum_{i=1}^m V_{[i\uparrow]} Y_{[i\uparrow]},
\end{eqnarray}
and we initialize $u^{\downarrow}_0=u^{\uparrow}_0=C^{\downarrow}_0=C^{\uparrow}_0=0$ for $m=0$. In absence of ties in the unit premiums
$(\Pi_i)_{i=1}^n$, the two constructions in \eqref{Gini 1b} and in \eqref{Gini 2b} coincide because no subordering is necessary.

The main difference between the Leimkuhler curve \eqref{Leimkuhler 2} and the CAPs
\eqref{Gini 1b}-\eqref{Gini 2b} is that we compute the running totals in different orders. The former is ordered w.r.t.~the responses $(Y_i)_{i=1}^n$ and the latter w.r.t.~the unit premiums $(\Pi_i)_{i=1}^n$.
We interpret \eqref{Gini 2b} as a {\it concordance measure} that assesses how well the order (ranking) of $(\Pi_i)_{i=1}^n$ is aligned with the one of $(Y_i)_{i=1}^n$. If they have the same order, we obtain the Leimkuhler curve \eqref{Leimkuhler 2}, and otherwise \eqref{Gini 2b} is dominated by the Leimkuhler curve \eqref{Leimkuhler 2}. The motivation behind the Gini score precisely is to measure this discrepancy. This consideration is fully rank based -- by \eqref{suborder a} and \eqref{suborder b} -- and it is asymmetric in the treatment of $(\Pi_i)_{i=1}^n$ and $(Y_i)_{i=1}^n$.

\medskip
\begin{graybox}
The {\it cumulative accuracy profiles} (CAPs) are obtained by linear interpolation between the points
\begin{equation}\label{def CAPs 1}
\left(u^{\downarrow}_m, C^{\downarrow}_m\right) \qquad \text{ for $m=0,\ldots, n$,}
\end{equation}
respectively, 
\begin{equation}\label{def CAPs 2}
\left(u^{\uparrow}_m, C^{\uparrow}_m\right) \qquad \text{ for $m=0,\ldots, n$.}
\end{equation}
\end{graybox}
\medskip

Again, these two curves are identical in absence of ties in the unit premiums.

\medskip

\begin{figure}[htb!]
\begin{center}
\begin{minipage}[t]{0.45\textwidth}
\begin{center}
\includegraphics[width=\textwidth]{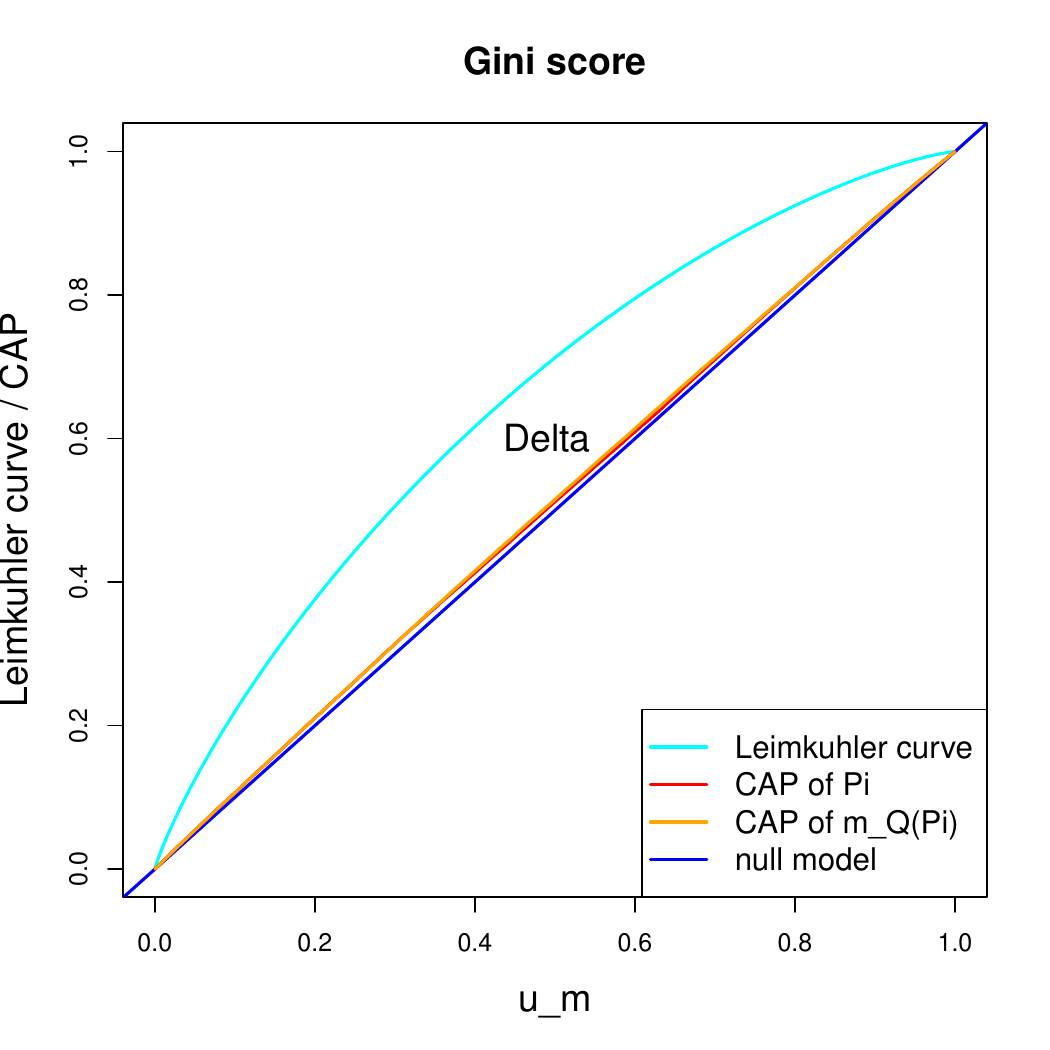}
\end{center}
\end{minipage}
\begin{minipage}[t]{0.45\textwidth}
\begin{center}
\includegraphics[width=\textwidth]{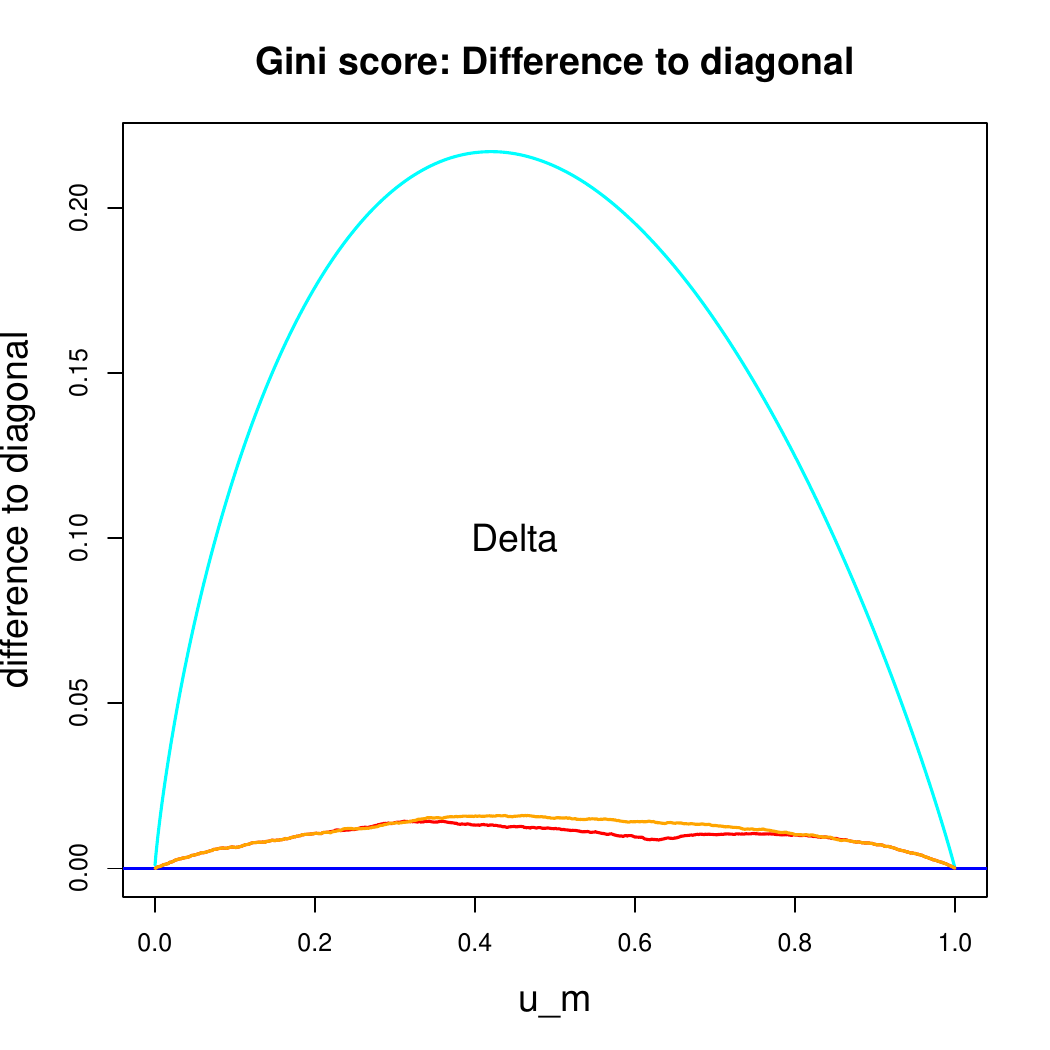}
\end{center}
\end{minipage}
\end{center}
\vspace{-.7cm}
\caption{Gini score plots: (lhs) CAPs of $(\Pi_i)_{i=1}^n$ and $(m_{\q}(\Pi_i))_{i=1}^n$ and Leimkuhler curve of $(Y_i)_{i=1}^n$; and (rhs) with subtracted diagonal for better visualization.}
\label{figure Gini plots}
\end{figure}

Figure \ref{figure Gini plots} (lhs) shows the CAPs \eqref{def CAPs 1}-\eqref{def CAPs 2} of the two premium rules $(\Pi_i)_{i=1}^n$ and $(m_{\q}(\Pi_i))_{i=1}^n$ in red and orange color (the two premium rules do not have ties). We see that the red and orange CAPs are almost indistinguishable in Figure \ref{figure Gini plots} (lhs)
and they are dominated by the Leimkuhler curve that considers the perfect ordering. The diagonal line in blue reflects the null model predictor $\mu_{\q}$ not considering any covariates.  To better visualize the results, we subtract these diagonal values from 
the Leimkuhler and CAP curves, this yields the plot on the right-hand side of Figure \ref{figure Gini plots}. We can now see that the CAP of $(m_{\q}(\Pi_i))_{i=1}^n$ in orange dominates the one of $(\Pi_i)_{i=1}^n$ in red, reflecting that there is an issue in the risk ranking of $(\Pi_i)_{i=1}^n$. That is, we correctly conclude from this plot that the recalibrated unit premiums provide the better risk ranking than the original unit premiums. From these plots, it seems that the differences are comparably small, however, the magnitudes of the differences in this analysis heavily suffer from the low signal-to-noise ratio in the data, which was already reported in Figure \ref{Murphy Pi 2}.

\subsection{Gini score}
The final step is to map the graphs of Figure \ref{figure Gini plots} to a score.
The Gini score is obtained by computing the ratio of the area of the CAP curve enclosed with the diagonal, and the area of the Leimkuhler curve enclosed with the diagonal -- the difference is highlighted by Delta in Figure \ref{figure Gini plots}. If Delta is zero we have a perfect ordering. This step also requires taking care of the suborders in the ties
\eqref{Gini 1b}-\eqref{Gini 2b}, and we simply average over the areas of the most favorable and least favorable suborders. We define from the Leimkuhler curve \eqref{def Leimkuhler curve} the area
\begin{equation*}
B = \sum_{m=1}^n \left(u_{m}- u_{m-1}\right) \frac{L_{m-1}+L_{m}}{2} - 1/2,
\end{equation*}
and for the CAPs \eqref{def CAPs 1}-\eqref{def CAPs 2} the areas
\begin{eqnarray*}
A^{\downarrow} &=& \sum_{m=1}^n \left(u^{\downarrow}_{m}- u^{\downarrow}_{m-1}\right) \frac{C^{\downarrow}_{m-1}+C^{\downarrow}_{m}}{2} - 1/2,
\\
A^{\uparrow} &=& \sum_{m=1}^n \left(u^{\uparrow}_{m}- u^{\uparrow}_{m-1}\right) \frac{C^{\uparrow}_{m-1}+C^{\uparrow}_{m}}{2} - 1/2
\quad \le \quad A^{\downarrow}.
\end{eqnarray*}
These formulas reflect a linear interpolation between the points in
\eqref{def Leimkuhler curve}, \eqref{def CAPs 1} and \eqref{def CAPs 2}, respectively.
The area Delta in Figure \ref{figure Gini plots} is the difference of $B$ and $A$, i.e., 
${\rm Delta}=B-A$. Remark that $B>0$ as soon as we have at least two different observations.

\bigskip
\begin{graybox}
The {\it Gini score} for the risk ranking of $\Pi$ for $Y$ on the test data ${\cal T}$ is defined by
\begin{equation}\label{def Gini 1}
\operatorname{Gini}({\cal T}) = \frac{(A^{\downarrow}+A^{\uparrow})/2}{B} ~\le ~ 1.
\end{equation}
\end{graybox}
\medskip

The implementation of the computation of the Gini score \eqref{def Gini 1} is straightforward and it can be found in the appendix of Brauer--W\"uthrich \cite{Brauer1}.

\medskip

\begin{table}[htb!]
\begin{center}
{\small
\begin{tabular}{|l|r|}
\hline
premium rules & Gini score \\
\hline\hline
global mean model $\mu_{\q}$ & 0.0000 \\
premium rule $\Pi$ & 0.0627  \\
premium rule $m_{\q}(\Pi)$ & 0.0719 \\
\hline
\end{tabular}}
\end{center}
\caption{Gini score assessing the risk rankings.}
\label{Gini score table}
\end{table}

Table \ref{Gini score table} shows the resulting Gini scores (bigger is better), and we give preference to the recalibrated unit premium $m_{\q}(\Pi)$ for providing the better risk ranking. As already mentioned, these numbers are comparably small (as can also be seen from Figure \ref{figure Gini plots}) because we have a low signal-to-noise ratio, but we obtain the correct preference order for the risk ranking.

\section{French motor third party liability example}

The running example above was a rather stylized one. In this section, we consider the popular French motor third party liability (MTPL) claims frequency dataset of Dutang--Charpentier \cite{Dutang} as our second example. We first apply the data cleaning procedure as described in W\"uthrich--Merz \cite{WM2023}.\footnote{The cleaned data is available from \url{https://aitools4actuaries.com/}.} The dataset is then partitioned into a learning dataset ${\cal L}$ consisting of 610,206 insurance policies and an independent training dataset ${\cal T}$ that contains $n=67,801$ insurance policies. We fit two unit premium rules $\Pi^{\rm PIN}$ and $\Pi^{\rm GBM}$ on the same learning dataset ${\cal L}$, and our goal is to validate these two premium rules on the independent test sample ${\cal T}$.

The dataset contains nine covariates, a claim counts response $Z \in \N_0$ and a time exposure $V>0$ in yearly units. For the subsequent considerations, we compute the unit losses $Y=Z/V$ -- which are {\it claims frequencies} in the case of claim counts $Z$ -- and our aim is to find accurate {\it expected claim frequency estimates} $\Pi=\Pi(\bX)$ which are measurable functions of the covariates $\bX$, we also refer to Remarks \ref{Remark covariate}. Thus, in this example we consider claim frequencies, but to maintain linguistic consistency with the previous sections, we call $\Pi$ a {\it unit premium rule}.

Our numerical analysis considers two unit premium rules (expected claim frequency regression estimates):
\begin{enumerate}
\item Pair-wise interaction network (PIN) forecast model $\Pi^{\rm PIN}$. The PIN model was developed in Richman et al.~\cite{PIN}, it is a neural network that specifically models pair-wise interactions of covariates. This model was trained on the learning dataset ${\cal L}$ mentioned above. 
 We use precisely the PIN parametrization that was obtained in Richman et al.~\cite[last line of Table 2]{PIN}.\footnote{The average Poisson deviance losses in Richman et al.~\cite[Table 2]{PIN} are scaled with the sample size $n$ and not with the aggregated exposure $\sum_{i=1}^n V_i$. The conversion factor on the test sample to compare the numbers is $n/\sum_{i=1}^n V_i=1.88511$.}
 \item As a second competing forecast model, we fit a gradient boosting machine (GBM) on the same learning dataset ${\cal L}$. We use the LightGBM version of {\sf R} with the hyper-parameter specification as given in the appendix, Listing \ref{LightGBMRCode}. This provides us with a second unit premium rule $\Pi^{\rm GBM}$.
\end{enumerate}
Our goal is to validate these two unit premium rules 
$\Pi^{\rm PIN}$ and $\Pi^{\rm GBM}$
on the independent test sample ${\cal T}$. We highlight that  both of the two unit premium rules are strong claim frequency forecast models on the French MTPL dataset. 

\subsection{Poisson deviance loss}
\label{sec Poisson deviance loss}
We compare and validate the two unit premium rules $\Pi^{\rm PIN}$ and $\Pi^{\rm GBM}$ in terms of calibration, discrimination and risk ranking. This is done on the independent test sample ${\cal T}$ that consists of $n=67,801$ instances. Since we deal with claim frequencies, it is natural to use the Poisson deviance loss for scoring. The Poisson deviance loss has a strictly convex generator with  $\varphi''(\theta)=2/\theta>0$ on the positive real line. That is, we have an example from the Patton family \eqref{Tweedie choices} with parameter $p=1$. We select for $x>0$
\begin{equation*}
\varphi(x) = 2\left( x \log(x)-x\right),
\end{equation*}
and we extend it to $x=0$ by setting $\varphi(0)=0$.
This yields first and second derivatives on $\R_+$
\begin{equation*}
\varphi'(x) = 2 \log(x)
\qquad \text{ and }\qquad
\varphi''(x) = 2/x>0.
\end{equation*}
This gives the Poisson deviance loss
\begin{equation}\label{Poisson loss}
L_\varphi(y,x) = 2\left(x - y + y \log(y/x)\right),
\end{equation}
with $L_\varphi(0,x)=2x$ in $y=0$ (i.e., for insurance policies without claims).

\begin{remark}[Exposures in Poisson deviance losses]\normalfont
The Poisson deviance loss is special in terms of the positive homogeneity property \eqref{homogeneous order 1}, namely, it is homogeneous of order 1.
The consequence of this property is that we obtain the {\it identical scoring} by either considering the total loss $Z$ (claim counts in our example) under the $\p$-measure or the unit loss costs $Y$ (claim frequency) under the $\q$-measure because of the identity
\begin{eqnarray*}
  L_\varphi(Z,P) ~=~ L_\varphi(VY,V\Pi )
  &=&
      2\left(V\Pi - VY + VY \log(Y/\Pi)\right)
  \\&=&
2\,V\left(\Pi - Y + Y \log(Y/\Pi)\right) =V\, L_\varphi(Y,\Pi).
\end{eqnarray*}
Thus, under the Poisson deviance loss, we can either use claim counts $Z$ to validate $P$ under $\p$ (left-hand side of the previous identity) or we can use the claim frequency $Y=Z/V$ to validate $\Pi=P/V$ under the $\q$-measure (on the right-hand side of the previous identity). This sometimes leads to confusion in Poisson model fitting, e.g., in {\sf R} there is the family {\tt poisson} which requires claim counts $Z$, and there is the family {\tt quasipoisson} which allows one to use claim frequencies $Y$ resulting in the same forecast model. This equivalence uses the homogeneity of order 1, and it does {\it not} carry over to other Bregman losses.

The quasi-Poisson estimation is also a popular method to ensure the (in-sample) balance property under a log-link generalized linear model (GLM); see Lindholm--W\"uthrich \cite{LW, LW0}. It can be used for any non-negative loss costs $Y$, not necessarily being claims frequencies, because the estimation procedure only relies on strictly consistent scoring with the Poisson deviance loss under the log-link choice, but not on any distributional assumptions.
\end{remark}

\medskip

\begin{table}[htb!]
\begin{center}
{\small
\begin{tabular}{|l|rr||r|}
\hline
& Bregman & Bregman & \qquad global \\
premium rules & divergence & score & mean \\
\hline\hline
global mean model $\widehat{\mu}_{\q}=\widehat{\mu}_{\q, {\cal L}}$ & 47.967 & -- & $\overline{Y}_{\q, {\cal T}}=7.35\%$ \\
  GLM unit premium& 45.435 &  2.532 & 7.40\%\\
  PIN unit premium $\Pi^{\rm PIN}$ & 44.615 &  3.352 & 7.31\%\\
  LightGBM unit premium $\Pi^{\rm GBM}$ & 44.372 &  3.595 & 7.38\%\\
\hline
\end{tabular}}
\end{center}
\caption{MTPL Bregman divergences $\widehat{\E}_{\q}[L_\varphi(Y,\cdot)]$ and Bregman scores $\widehat{S}_\varphi(Y, \cdot)$ of the considered premium rules under the Poisson deviance loss choice on the test sample ${\cal T}$; units are shown in $10^{-2}$.}
\label{MTPL deviance losses}
\end{table}

Table \ref{MTPL deviance losses} reports the Poisson deviance losses on the test sample ${\cal T}$ of the (constant) global mean model $\widehat{\mu}_{\q}=\widehat{\mu}_{\q, {\cal L}}$ (estimated by the observed frequency on the learning sample ${\cal L}$), a strong generalized linear model (GLM) taken from W\"uthrich--Merz \cite{WM2023}, and the PIN and LightGBM models described above. This allows us to compute the sample versions of the Poisson Bregman scores \eqref{Bregman score decomposition}. From these sample Bregman scores we conclude that the PIN and the LightGBM are clearly stronger than the selected GLM (2.532), and we give preference to the LightGBM (3.595) over the PIN (3.352) unit premium rule in the terms of the sample Bregman scores of Table \ref{MTPL deviance losses}. Performing a paired non-parametric bootstrap (drawing with replacement from the test sample ${\cal T}$), we obtain a bootstrap standard deviation of 0.057 that quantifies the uncertainty in the difference $3.595-3.352=0.243$, thus, the Bregman score improvement is significant.

The main question that we study below is whether we can find more evidence for this preference and whether it can be understood on a more granular level. Remark that the Bregman score preference of Table \ref{MTPL deviance losses} is mainly based on discrimination.

The last column of Table \ref{MTPL deviance losses} shows that exposure-weighted average unit premiums over the entire portfolio (test sample ${\cal T}$), the first line giving the observed empirical frequency $\overline{Y}_{\q, {\cal T}}=7.35\%$ on the test sample ${\cal T}$. Thus, the last column of the table shows the sample version of \eqref{global unbiasedness definition}. The numbers indicate that the PIN slightly underestimates the observed frequency of 7.35\% and the LightGBM slightly overestimates this observed frequency. Under a Poisson assumption the magnitude of irreducible risk in this frequency estimate is 0.14\%, thus, the deviations from the global observed frequency do not seem to be of a systematic nature.

\subsection{Actual-vs-expected plots and lift charts}
We begin by studying the exposure-weighted distributions of the unit premiums 
$\Pi_i^{\rm PIN}$ and $\Pi_i^{\rm GBM}$ over the entire test sample $i=1,\ldots, n$.

\medskip

\begin{center}
\noindent\fbox{%
  \begin{minipage}{0.9\textwidth}
    \centering
    ~\\
    To receive expressive graphs, most of the figures are plotted on the log-scale.
\\~    
  \end{minipage}%
}
\end{center}

\medskip

\begin{figure}[htb!]
\begin{center}
\begin{minipage}[t]{0.45\textwidth}
\begin{center}
\includegraphics[width=\textwidth]{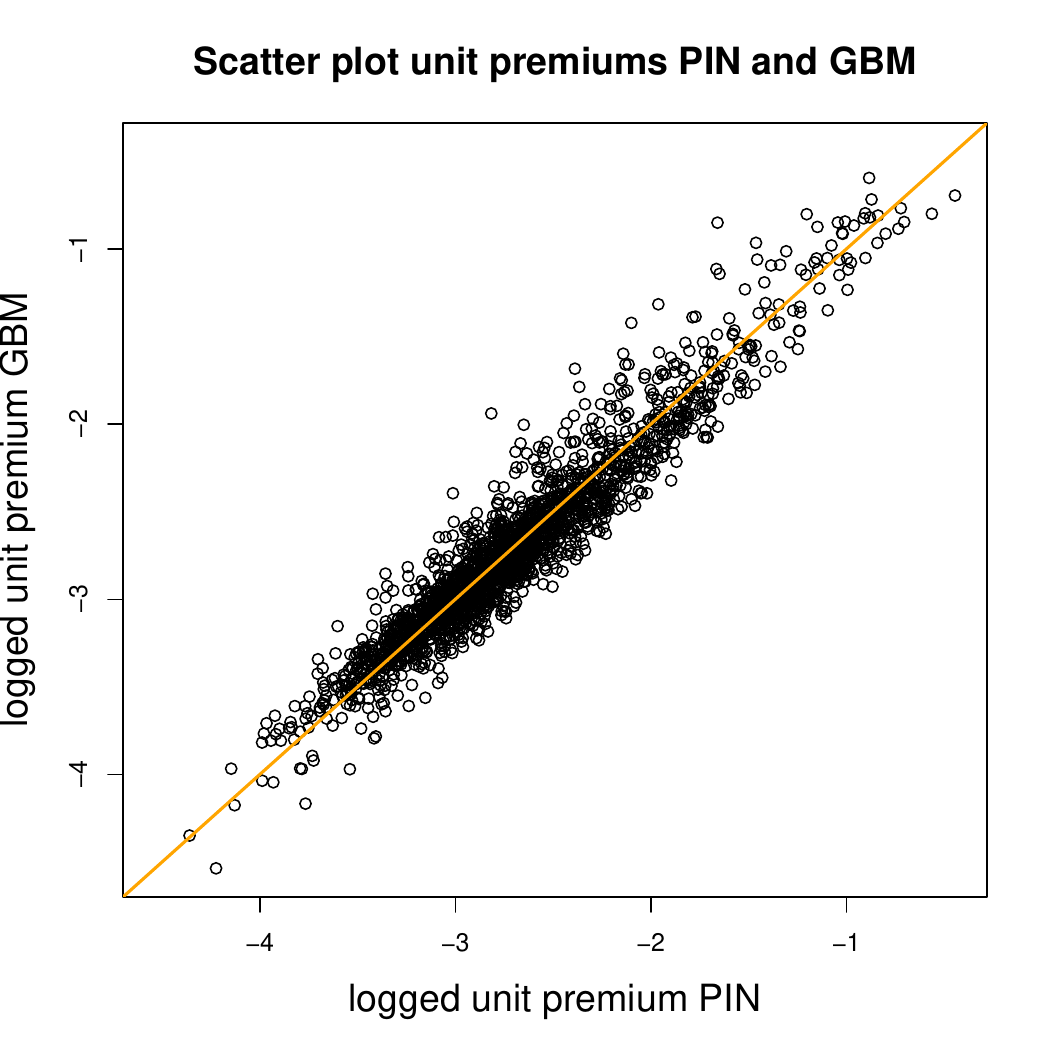}
\end{center}
\end{minipage}
\begin{minipage}[t]{0.45\textwidth}
\begin{center}
\includegraphics[width=\textwidth]{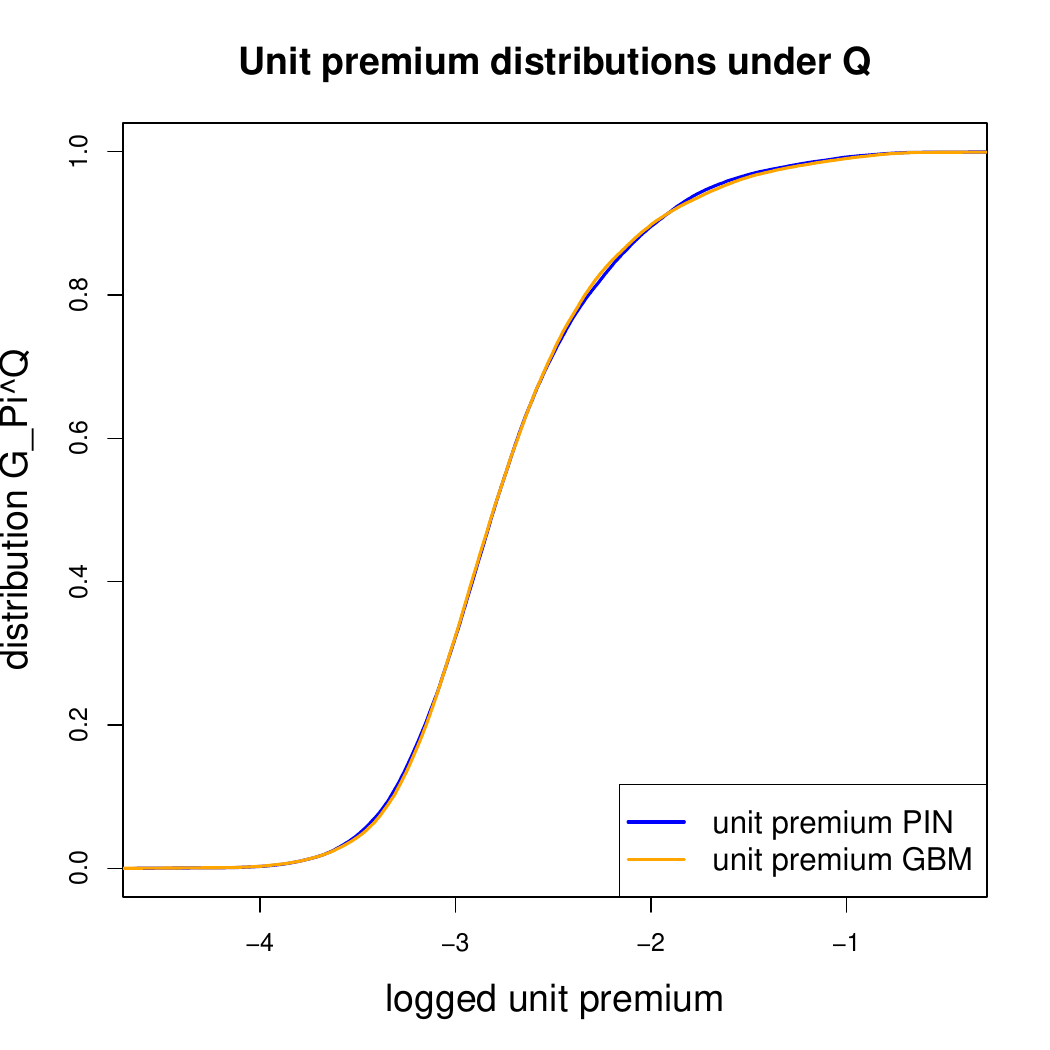}
\end{center}
\end{minipage}
\end{center}
\vspace{-.7cm}
\caption{MTPL unit premiums: (lhs) scatter plot of $\log(\Pi^{\rm GBM})$ against $\log(\Pi^{\rm PIN})$, and (rhs) exposure-weighted  sample distributions $\widehat{G}_{\Pi^{\rm PIN}}^{\q}$ and  $\widehat{G}_{\Pi^{\rm GBM}}^{\q}$.}
\label{MTPL G distributions}
\end{figure}

Figure \ref{MTPL G distributions} (lhs) shows the scatter plot of the LightGBM unit premiums $\Pi_i^{\rm GBM}$ against the PIN unit premiums $\Pi_i^{\rm PIN}$ over the test instances $i=1,\ldots, n$. This scatter plot fluctuates around the orange diagonal line which indicates a large similarity between the two unit premium rules. This is also verified by the
exposure-weighted  sample distributions $\widehat{G}_{\Pi^{\rm PIN}}^{\q}$ and  $\widehat{G}_{\Pi^{\rm GBM}}^{\q}$ on the right-hand side of that figure. The biggest differences are observed in the range of bigger predictions, however, from Figure \ref{MTPL G distributions} it seems that these differences are small.

\medskip

\begin{figure}[htb!]
\begin{center}
\begin{minipage}[t]{0.45\textwidth}
\begin{center}
\includegraphics[width=\textwidth]{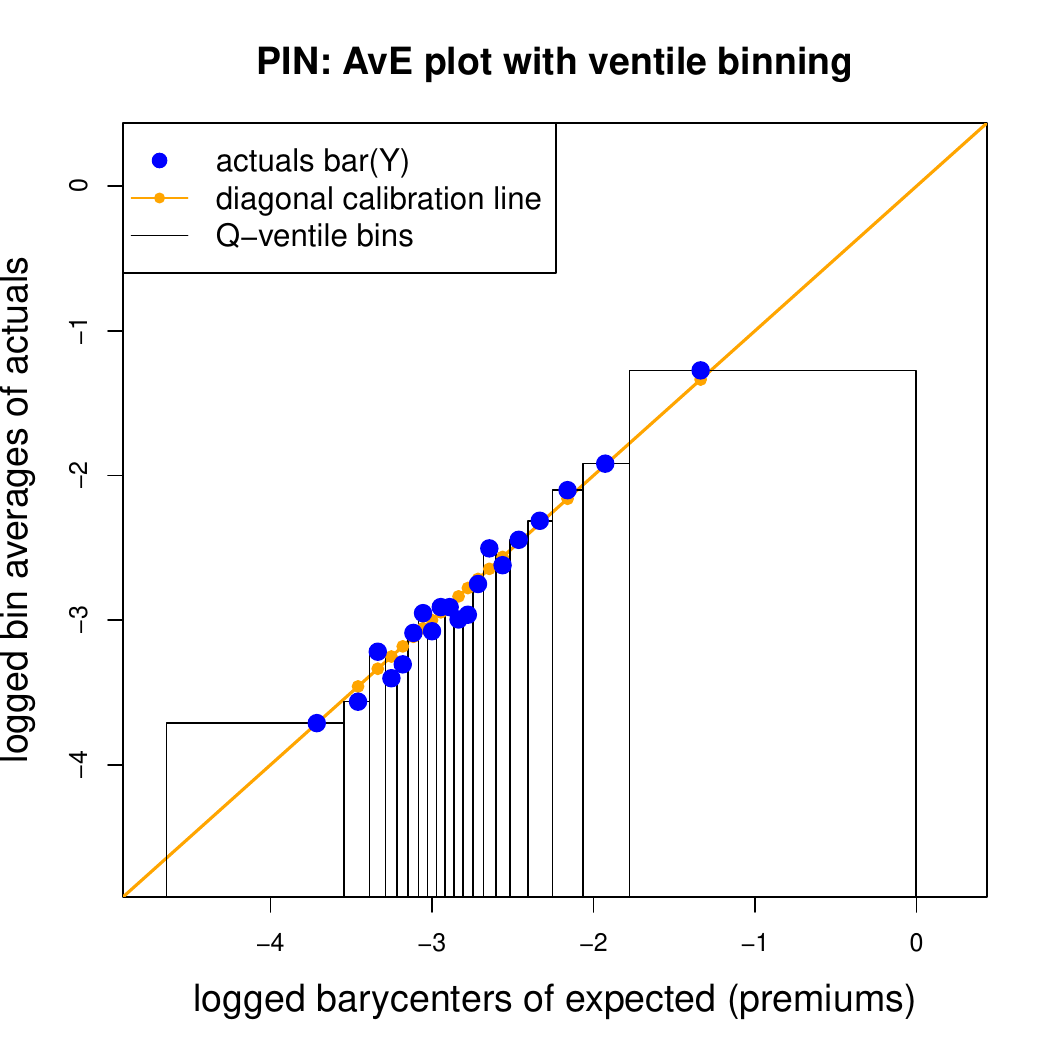}
\end{center}
\end{minipage}
\begin{minipage}[t]{0.45\textwidth}
\begin{center}
\includegraphics[width=\textwidth]{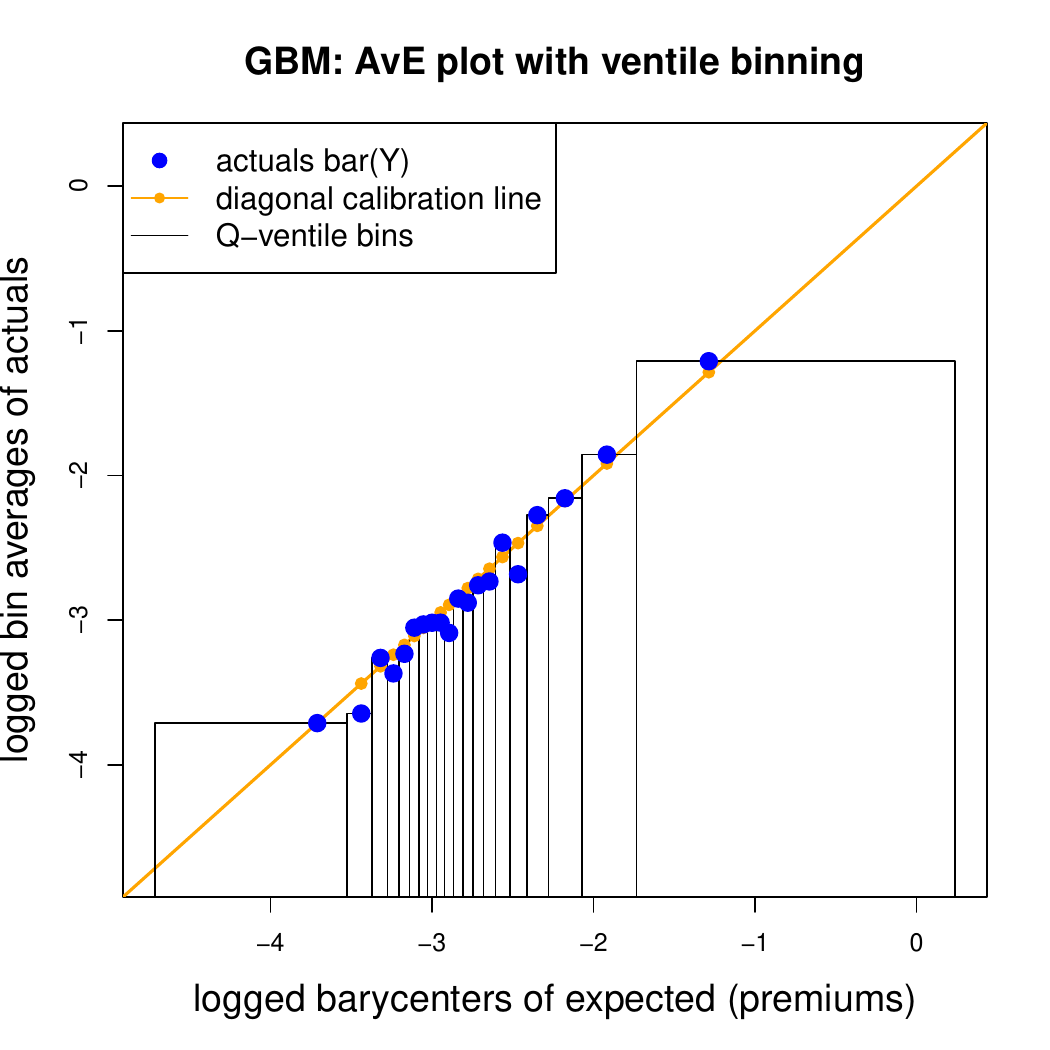}
\end{center}
\end{minipage}
\begin{minipage}[t]{0.45\textwidth}
\begin{center}
\includegraphics[width=\textwidth]{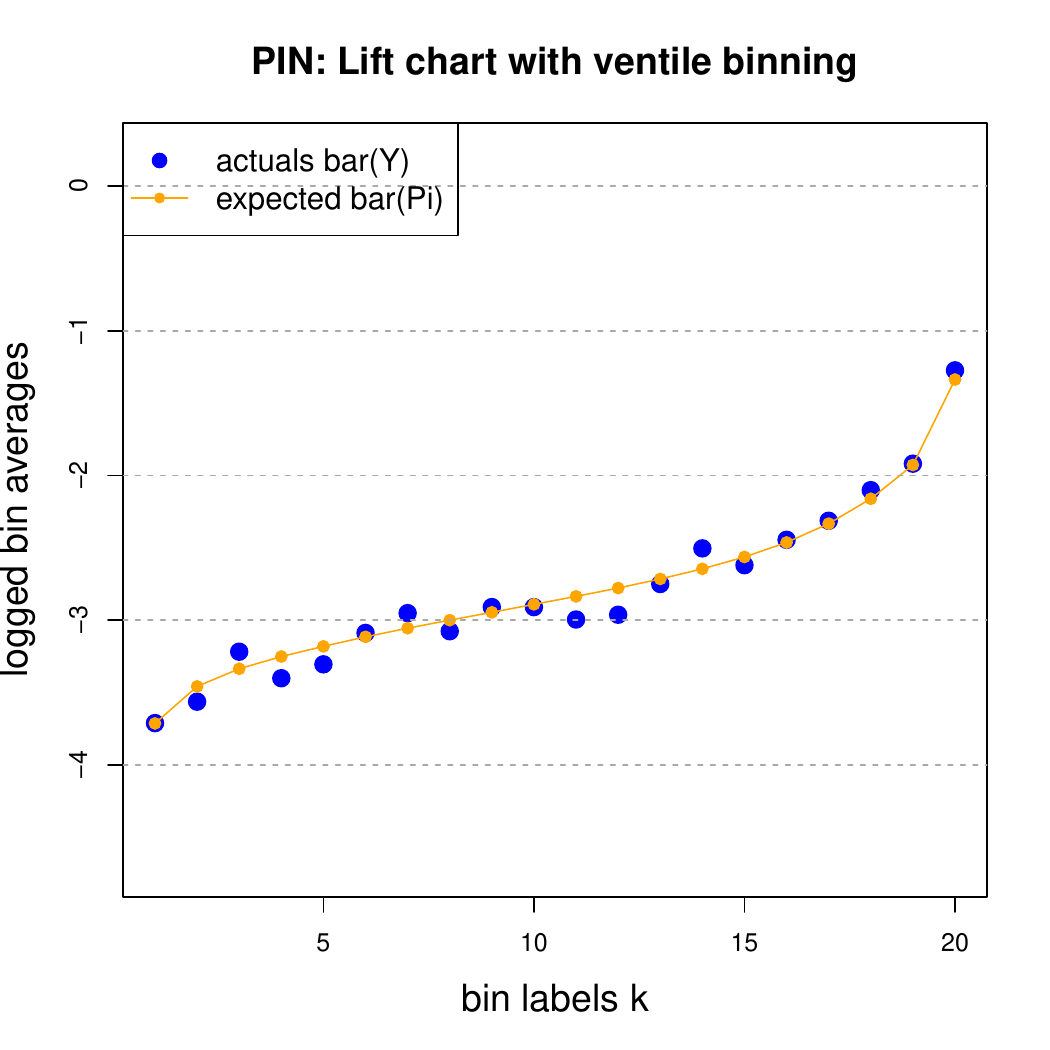}
\end{center}
\end{minipage}
\begin{minipage}[t]{0.45\textwidth}
\begin{center}
\includegraphics[width=\textwidth]{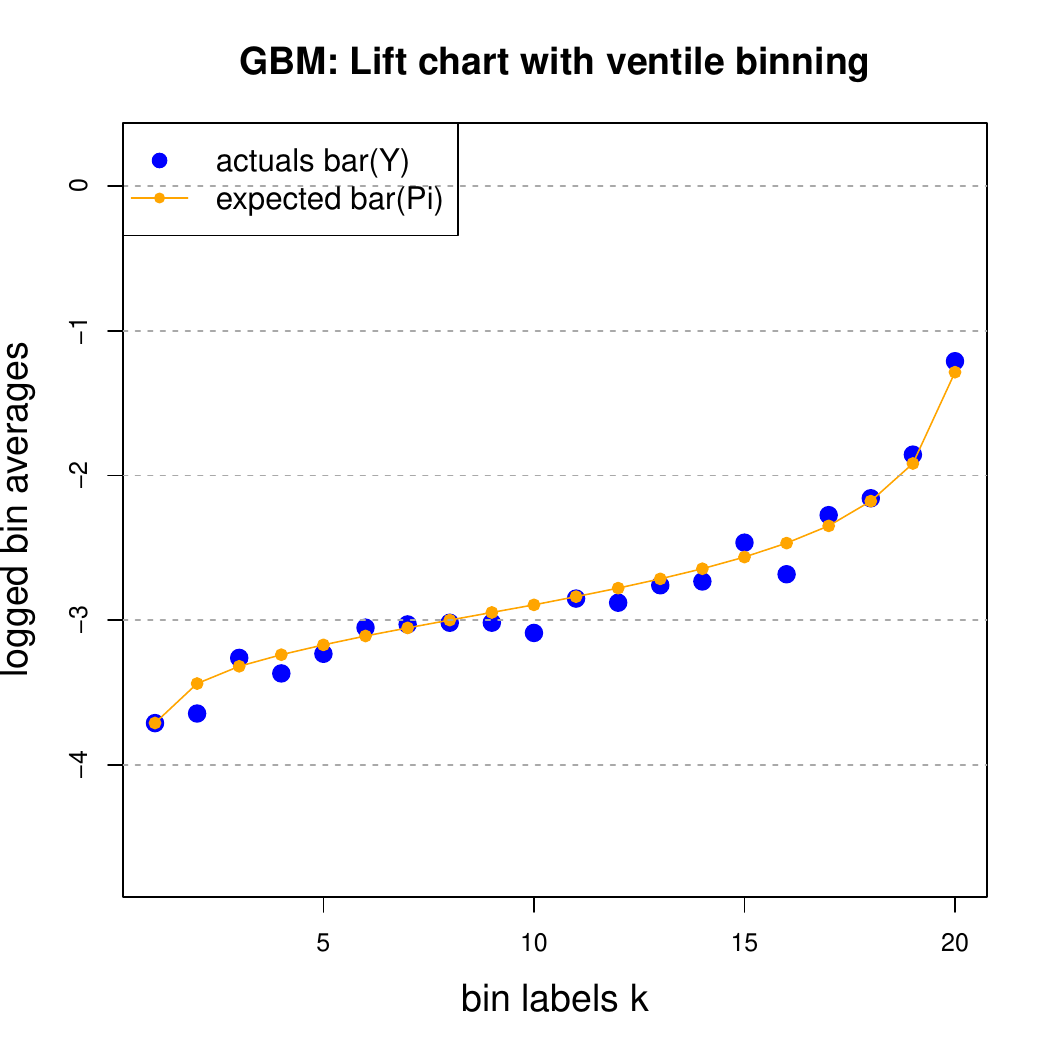}
\end{center}
\end{minipage}
\end{center}
\vspace{-.7cm}
\caption{(top) MTPL AvE plots with ventile binning $K=20$: (lhs) PIN unit premium $\Pi^{\rm PIN}$ and (rhs) LightGBM unit premium $\Pi^{\rm GBM}$;  (bottom)
MTPL lift charts with ventile binning $K=20$: (lhs) PIN unit premium and (rhs) GBM unit premium; the axes are on the log-scale.}
\label{MTPL lift chart plot}
\end{figure}

Figure \ref{MTPL lift chart plot} presents the AvE plots and the lift charts of the two unit premium rules $\Pi^{\rm PIN}$ and $\Pi^{\rm GBM}$. For the binning, we select ventile binning $K=20$. The lift charts are indicating the following properties:
\begin{itemize}
\item Discrimination: Both unit premium rules show a similar lift on the ventile binning scale, see lift charts in the lower panels; for a colored version see Figure \ref{MTPL resolution and miscalibration plots} in the appendix. 
\item Calibration: It seems that PIN shows a slightly better calibration picture because the GBM looks more systematically biased over the bins 9 to 14; see also Figure \ref{MTPL resolution and miscalibration plots} in the appendix.
\item Risk ranking: In both cases the actuals are non-monotone, may be a slight preference is given to the GBM version.
\end{itemize}
In Figure \ref{MTPL resolution and miscalibration plots} in the appendix, we present the same lift charts, but we add the coloring for the (estimated) miscalibration and resolution on this ventile binning granularity.

\medskip

\begin{figure}[htb!]
\begin{center}
\begin{minipage}[t]{0.45\textwidth}
\begin{center}
\includegraphics[width=\textwidth]{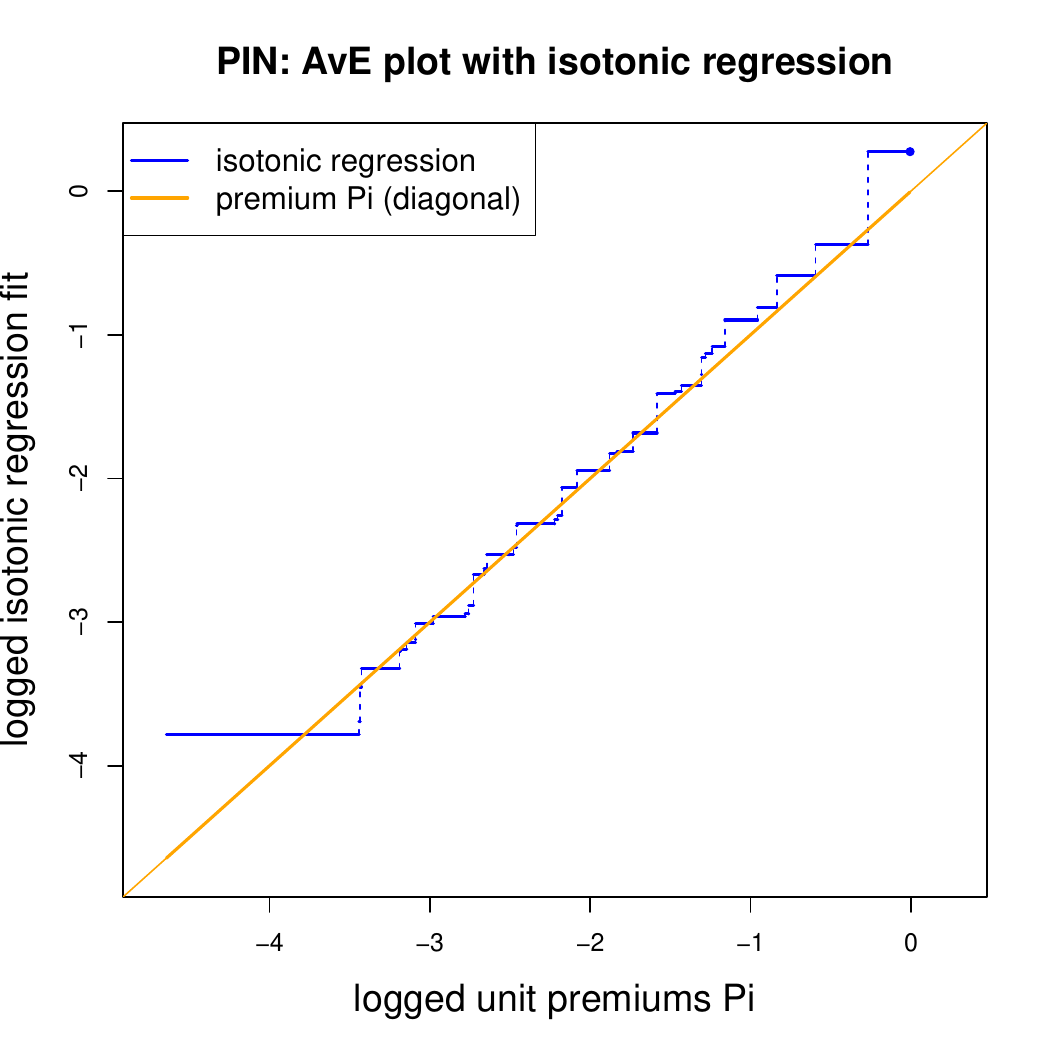}
\end{center}
\end{minipage}
\begin{minipage}[t]{0.45\textwidth}
\begin{center}
\includegraphics[width=\textwidth]{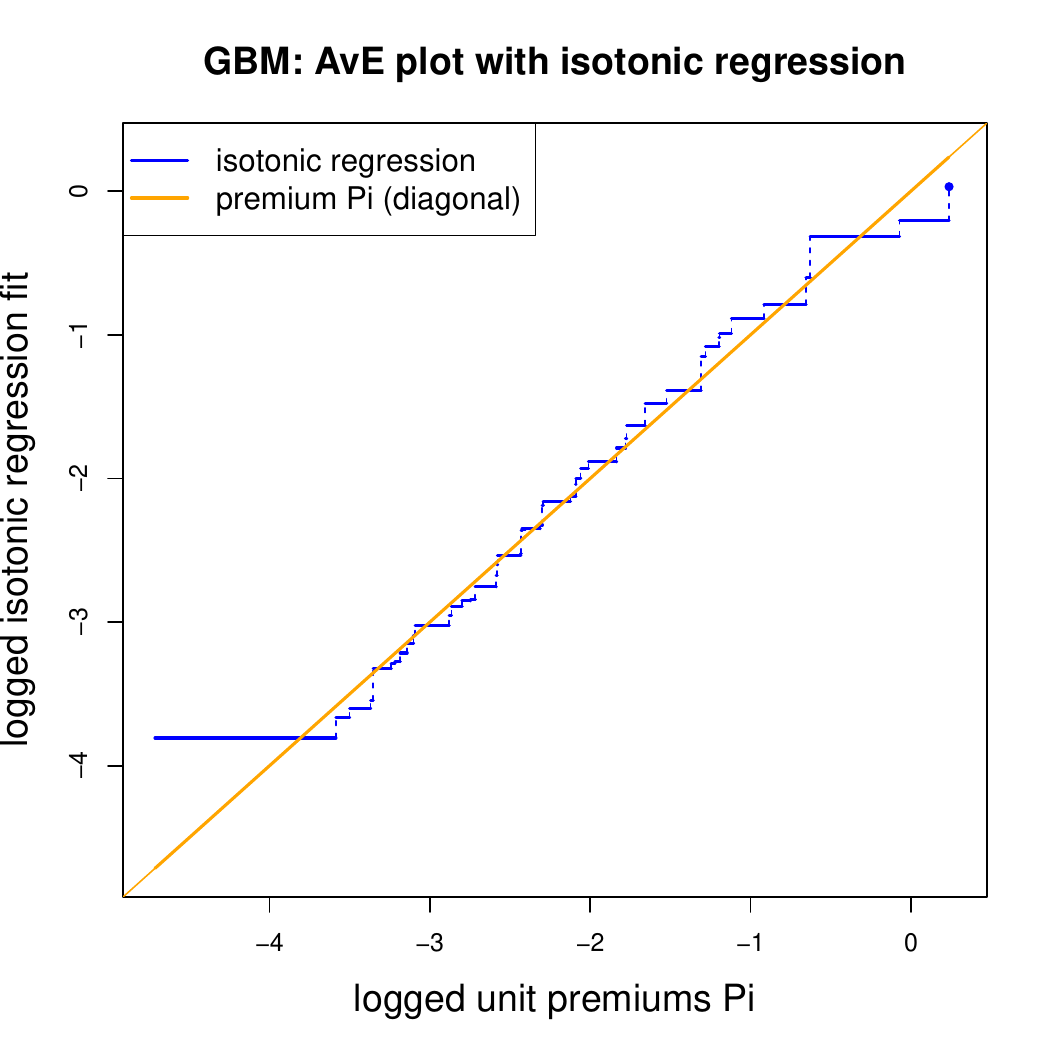}
\end{center}
\end{minipage}
\begin{minipage}[t]{0.45\textwidth}
\begin{center}
\includegraphics[width=\textwidth]{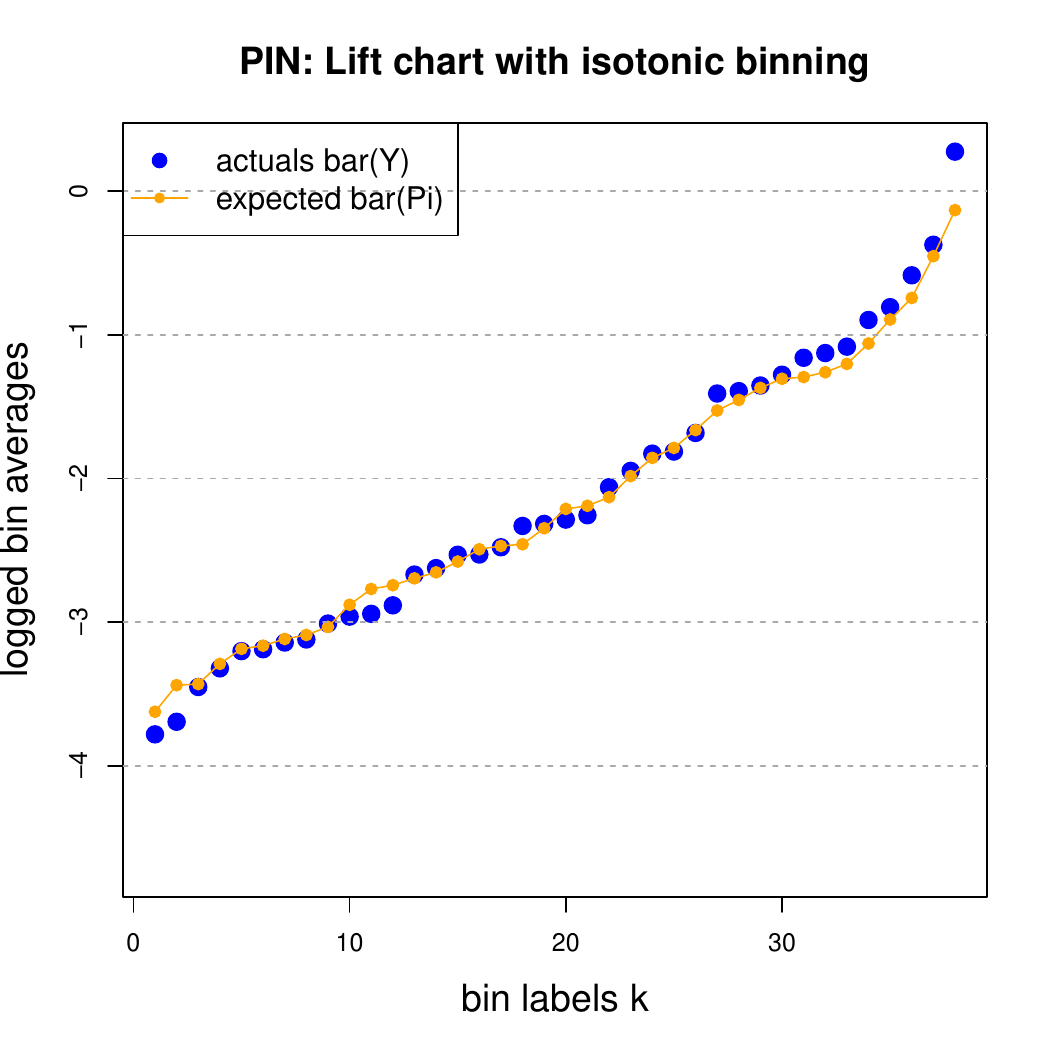}
\end{center}
\end{minipage}
\begin{minipage}[t]{0.45\textwidth}
\begin{center}
\includegraphics[width=\textwidth]{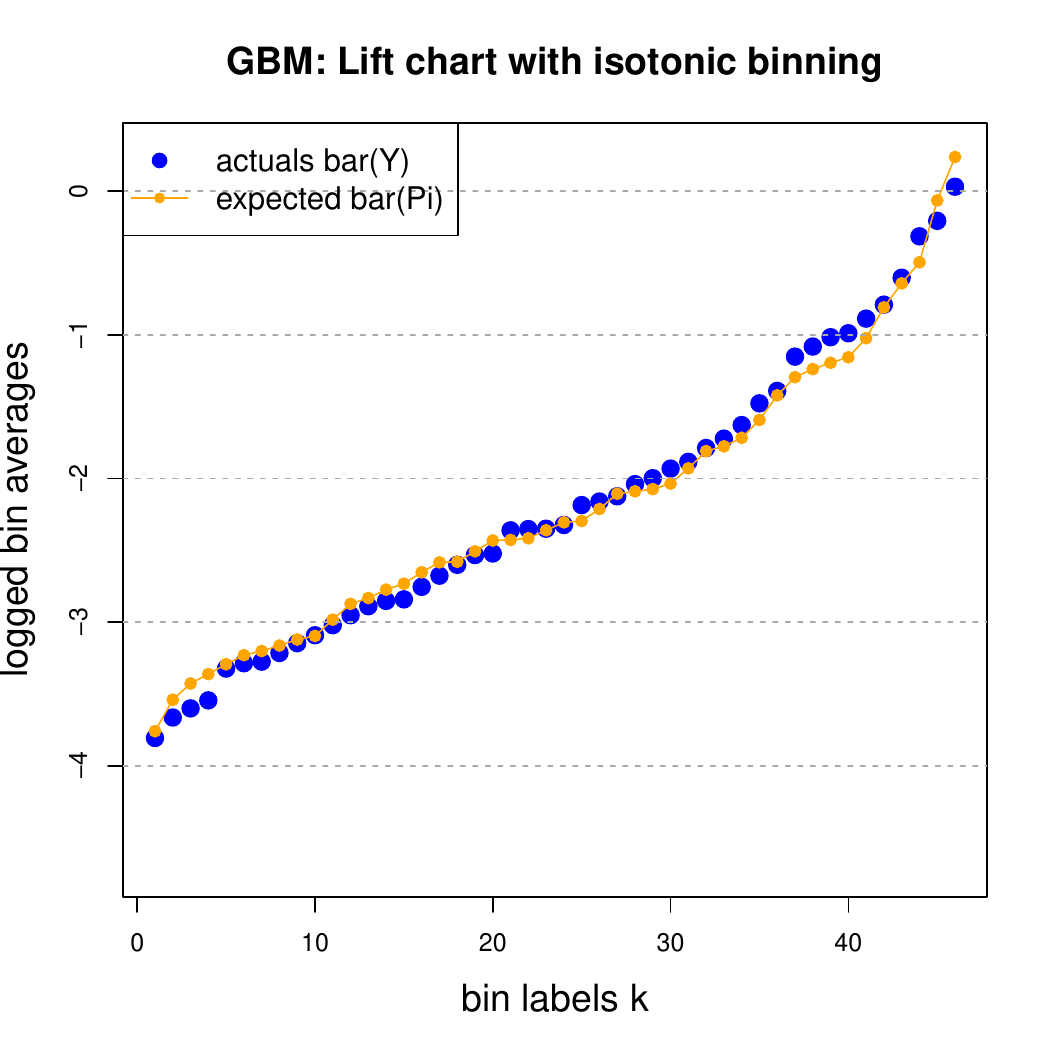}
\end{center}
\end{minipage}
\end{center}
\vspace{-.7cm}
\caption{(top) MTPL AvE plots with isotonic regression: (lhs) PIN unit premium with 38 bins and (rhs) GBM unit premium with 46 bins; (bottom) MTPL lift chart using isotonic regression binning: (lhs) PIN unit premium with 38 bins and (rhs) GBM unit premium with 46 bins.}
\label{MTPL isotonic regression lift chart}
\end{figure}

Figure \ref{MTPL isotonic regression lift chart} gives the AvE plots and the lift charts under isotonic regression binning \eqref{isotonic regression 0}. Recall that this uses the unit premiums as risk rankings (risk ordering). A higher number of bins may indicate a better risk ranking because wrongly ordered premiums typically lead to bigger bin sizes by correspondingly averaging over larger premium ranges; for the explanation of the functioning of the PAV algorithm to perform the isotonic regression, we refer to the appendix of W\"uthrich--Ziegel \cite{WZiegel}. PIN receives 38 isotonic regression bins and GBM 46 bins; these are not reliable statistics, but they give some indication about risk rankings. 

The lift of the two unit premium rules is still similar.
We merged the two smallest bins to ensure strict positivity of the isotonic recalibration solution. We did not merge the two largest bins, and it seems that the isotonic recalibration step slightly overfits for the largest value in the PIN version of Figure \ref{MTPL isotonic regression lift chart}. Calibration seems in both plots similar, maybe on the larger unit premiums the losses are slightly under estimated, and there might be a small preference for PIN in terms of calibration because the alignment of actuals and expected in the lift chart in the lower panel seems slightly better for PIN.

\medskip

\begin{figure}[htb!]
\begin{center}
\begin{minipage}[t]{0.45\textwidth}
\begin{center}
\includegraphics[width=\textwidth]{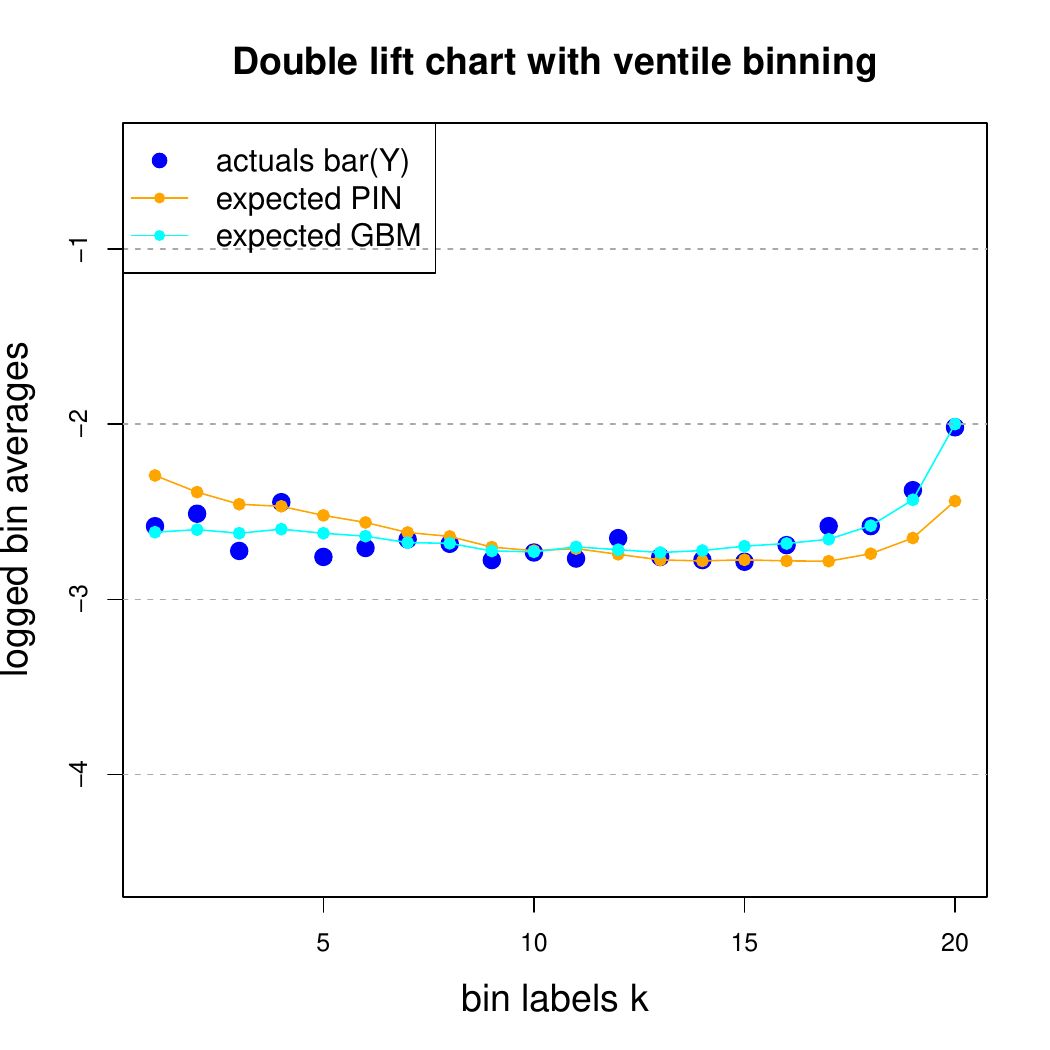}
\end{center}
\end{minipage}
\end{center}
\vspace{-.7cm}
\caption{MTPL double lift chart with ventile binning $K=20$ considering GBM unit premiums over PIN unit premiums for $\kappa$; the $y$-axis is on the log-scale.}
\label{MTPL double lift chart plot}
\end{figure}

Figure \ref{MTPL double lift chart plot} provides the double lift chart using ventile binning $K=20$. It considers the ratio of the LightGBM unit premium divided by the PIN unit premium. From the double lift chart we conclude that the LightGBM is clearly more accurate than the PIN on this granularity, this holds for both tails.  This suggests better discrimination (and/or risk ranking) by the GBM forecast, and this also supports the numerical results found in Table \ref{MTPL deviance losses}.

\subsection{Murphy diagram}
Next, we aim at understanding how the Poisson deviance losses of Table \ref{MTPL deviance losses} are composed across the unit premium ranges. We therefore start by studying the Murphy diagram \eqref{empirical Murphy graph} showing the sample elementary losses as a function of the threshold $\theta$. In particular, this does not apply the Poisson deviance loss weighting $\varphi''(\theta)=2/\theta$, see \eqref{Bregman loss vs elementary loss}.

\medskip

\begin{figure}[htb!]
\begin{center}
\begin{minipage}[t]{0.45\textwidth}
\begin{center}
\includegraphics[width=\textwidth]{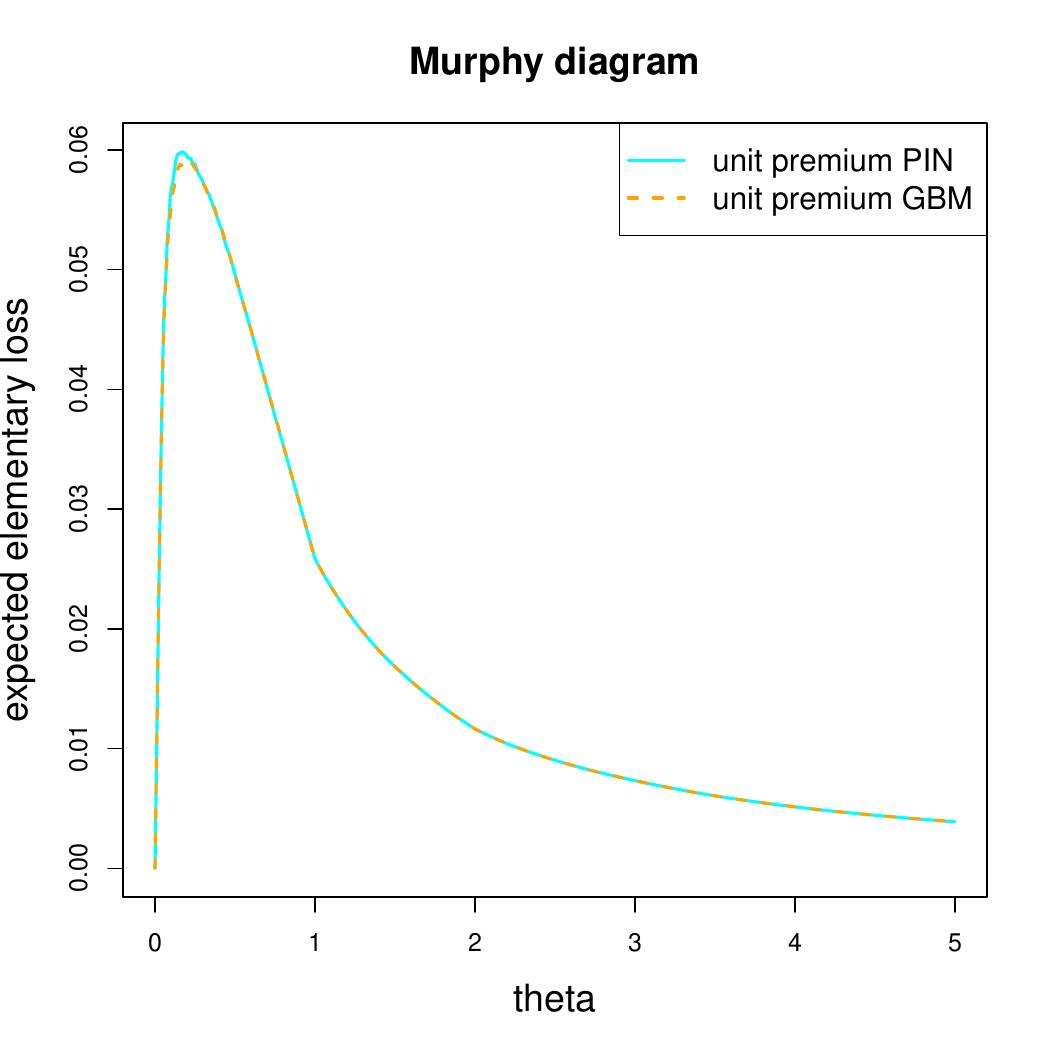}
\end{center}
\end{minipage}
\begin{minipage}[t]{0.45\textwidth}
\begin{center}
\includegraphics[width=\textwidth]{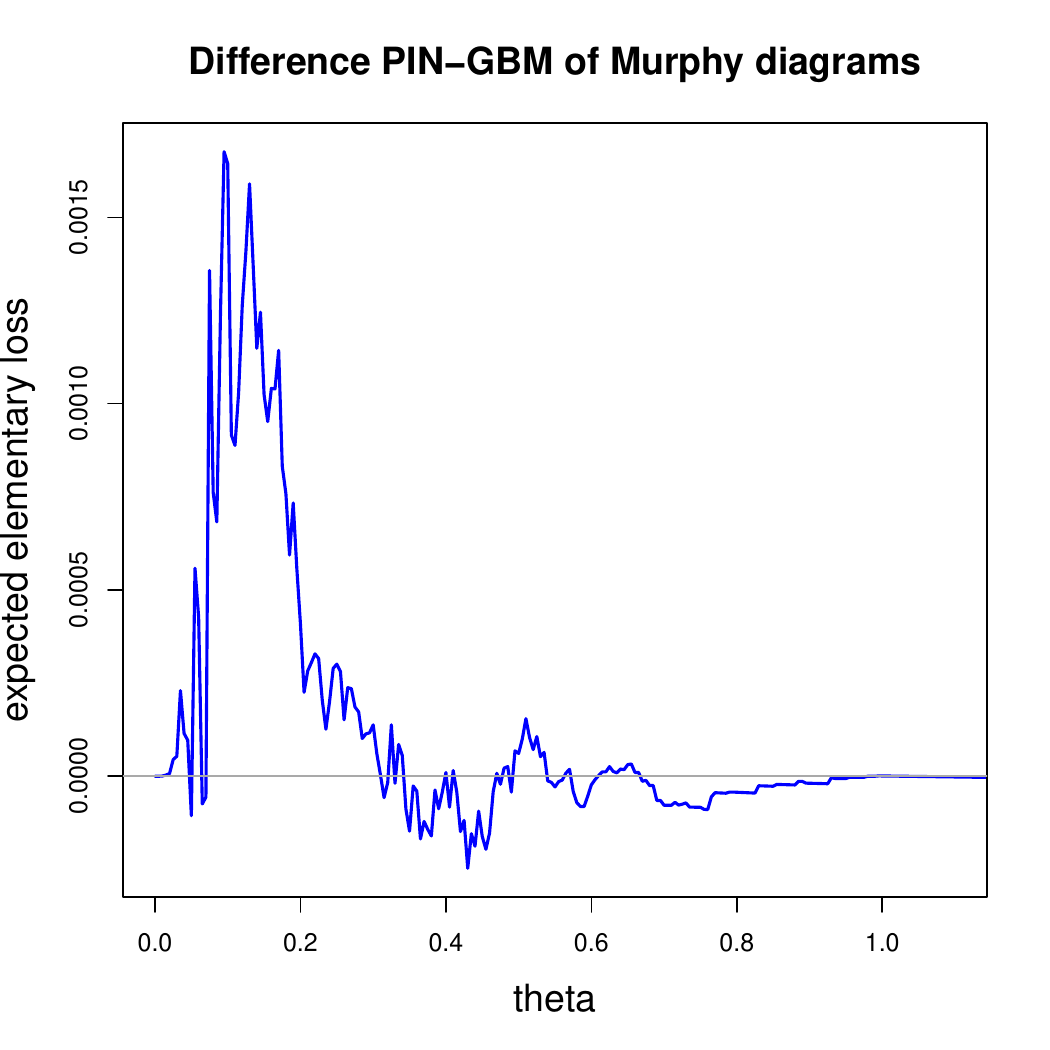}
\end{center}
\end{minipage}
\end{center}
\vspace{-.7cm}
\caption{MTPL sample elementary losses of the PIN and LightGBM unit premiums and their differences $\Delta_{\q}^{\rm Murphy}(\theta; \Pi^{\rm PIN}, \Pi^{\rm GBM})$ as a function of the threshold $\theta$.}
\label{MTPL elementary scores}
\end{figure}

Figure \ref{MTPL elementary scores} shows the Murphy diagram, the left-hand side gives the sample elementary losses \eqref{empirical Murphy graph} of the two premium rules as a function of $\theta$, and the right-hand side considers their differences $\Delta_{\q}^{\rm Murphy}(\theta; \Pi^{\rm PIN}, \Pi^{\rm GBM})$, subtracting the GBM version from the PIN version, see \eqref{Murphy Delta difference}. From the right-hand side we conclude that the LightGBM performs better than the PIN for the majority of $\theta$-values, confirming the better Bregman score of Table \ref{MTPL deviance losses} probably over almost the entire unit premium range. Remark that the Poisson deviance loss integrates over these sample elementary losses using the decreasing weight function $\varphi''(\theta)=2/\theta$, that is, for the Poisson Bregman score difference we consider
\begin{equation*}
\operatorname{PoissonScoreDifference}
= \int_0^\infty \Delta_{\q}^{\rm Murphy}(\theta; \Pi^{\rm PIN}, \Pi^{\rm GBM}) \,\frac{2}{\theta^p}\, \dd \theta
\qquad \qquad \text{ with $p=1$.}
\end{equation*}
In view of Figure \ref{MTPL elementary scores} (rhs), we expect the same preference order within the Patton family \eqref{Tweedie choices} for any $p\ge 0$, since these weightings are all monotonically decreasing.

\medskip

\begin{figure}[htb!]
\begin{center}
\begin{minipage}[t]{0.45\textwidth}
\begin{center}
\includegraphics[width=\textwidth]{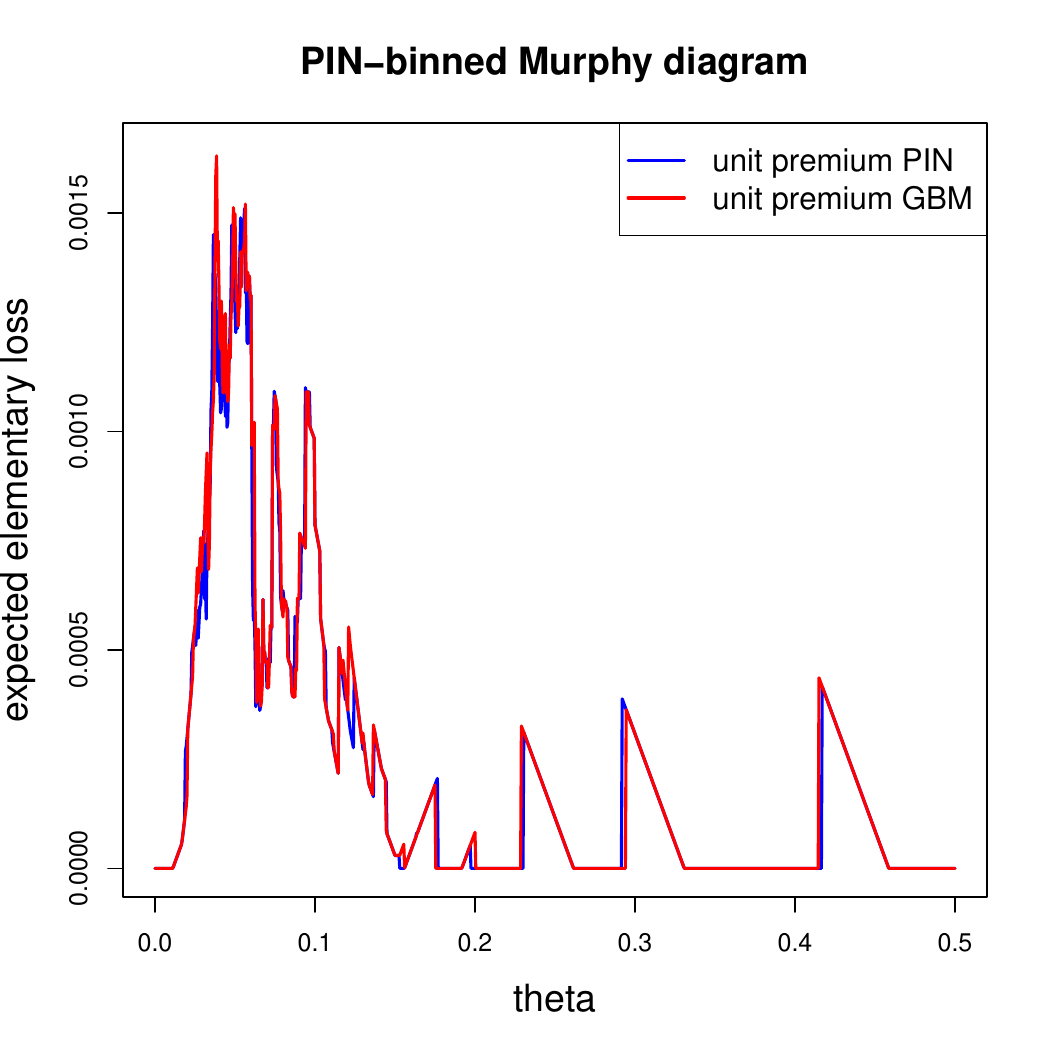}
\end{center}
\end{minipage}
\begin{minipage}[t]{0.45\textwidth}
\begin{center}
\includegraphics[width=\textwidth]{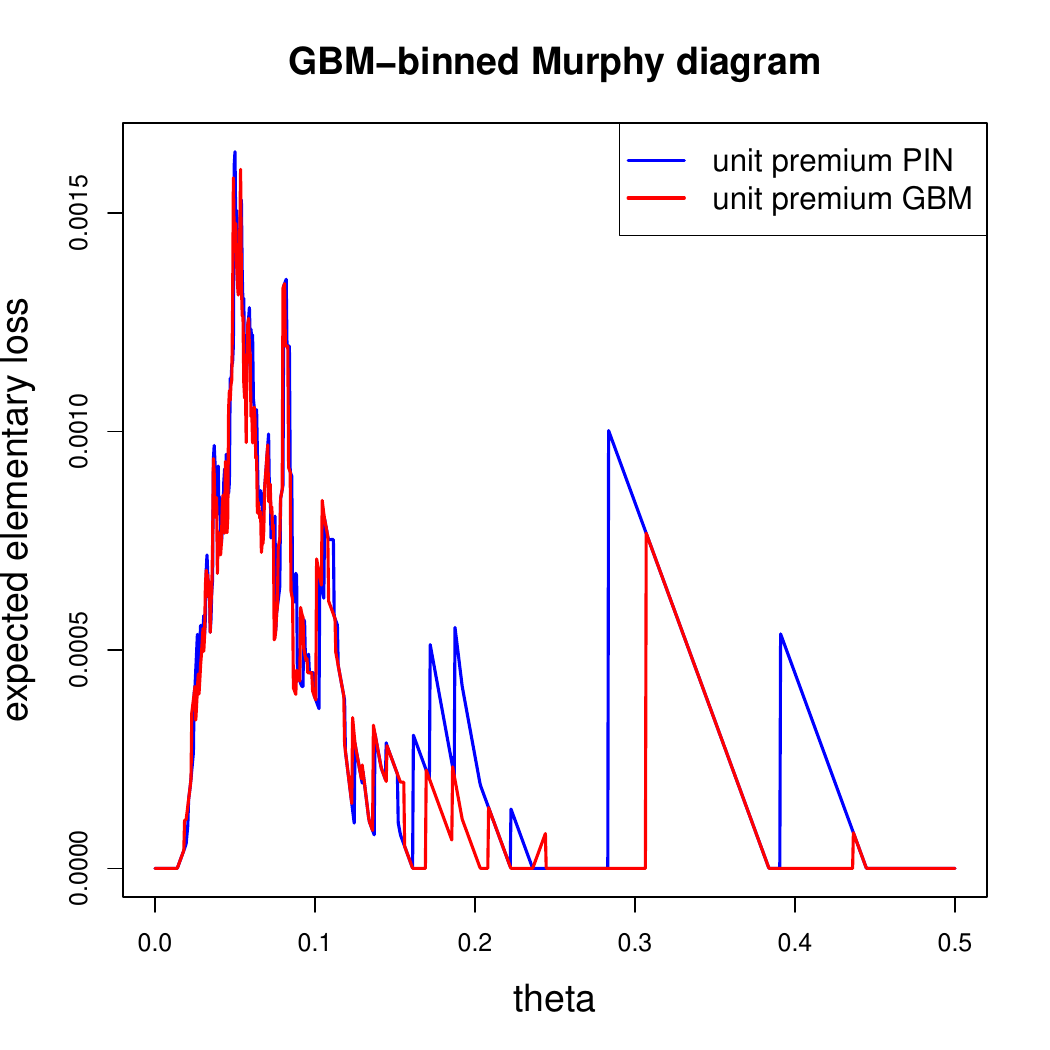}
\end{center}
\end{minipage}
\begin{minipage}[t]{0.45\textwidth}
\begin{center}
\includegraphics[width=\textwidth]{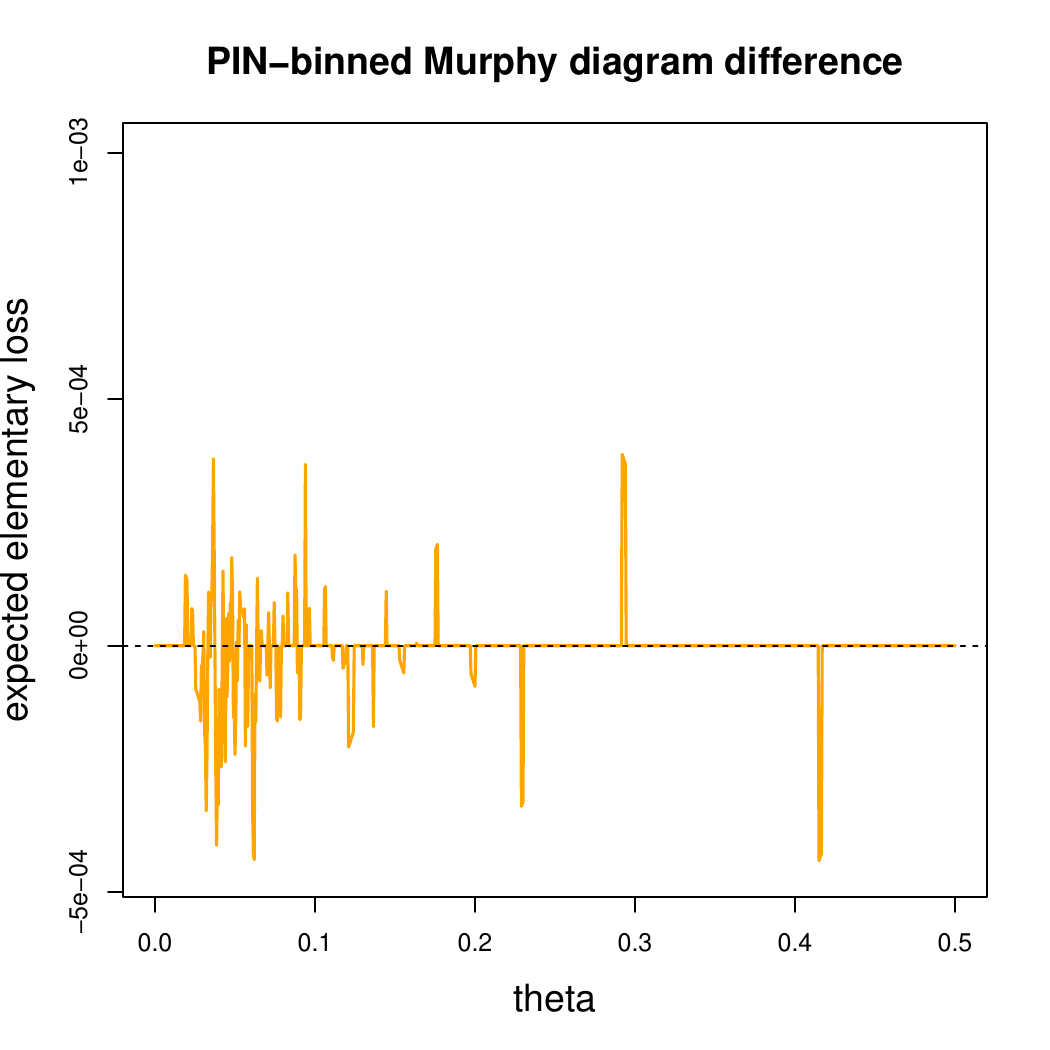}
\end{center}
\end{minipage}
\begin{minipage}[t]{0.45\textwidth}
\begin{center}
\includegraphics[width=\textwidth]{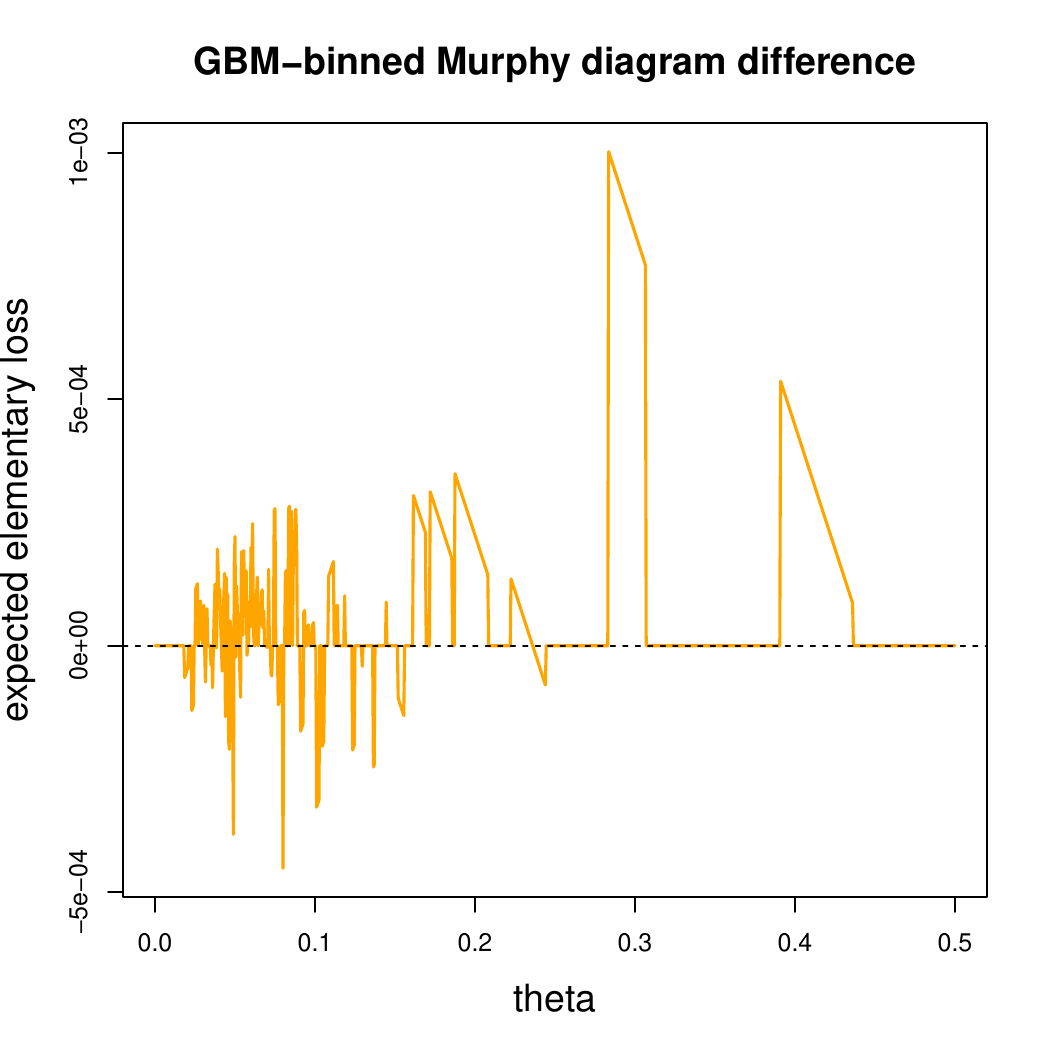}
\end{center}
\end{minipage}
\end{center}
\vspace{-.7cm}
\caption{MTPL asymmetrically binned Murphy diagrams of the PIN and GBM unit premium rules: (lhs) PIN percentile binning and (rhs) GBM percentile binning.}
\label{MTPL aggregated Murphy}
\end{figure}

Figure \ref{MTPL aggregated Murphy} shows the asymmetrically binned Murphy diagrams (top) and the resulting differences (bottom). We perform percentile binning $K=100$ w.r.t.~the PIN premium rule on the left-hand side and the GBM premium rule on the right-hand side. From the PIN binning plots we do not learn very much. But from the GBM binning plots on the right-hand side, it seems that the GBM unit premium has a superior performance especially in the larger range. Moreover, the differences in the two binning versions may indicate that one of the two premium rules provides a better risk ranking, because a different exposure-weighting binning is only impacted by different risk rankings.

\medskip

\begin{figure}[htb!]
\begin{center}
\begin{minipage}[t]{0.9\textwidth}
\begin{center}
\includegraphics[width=\textwidth]{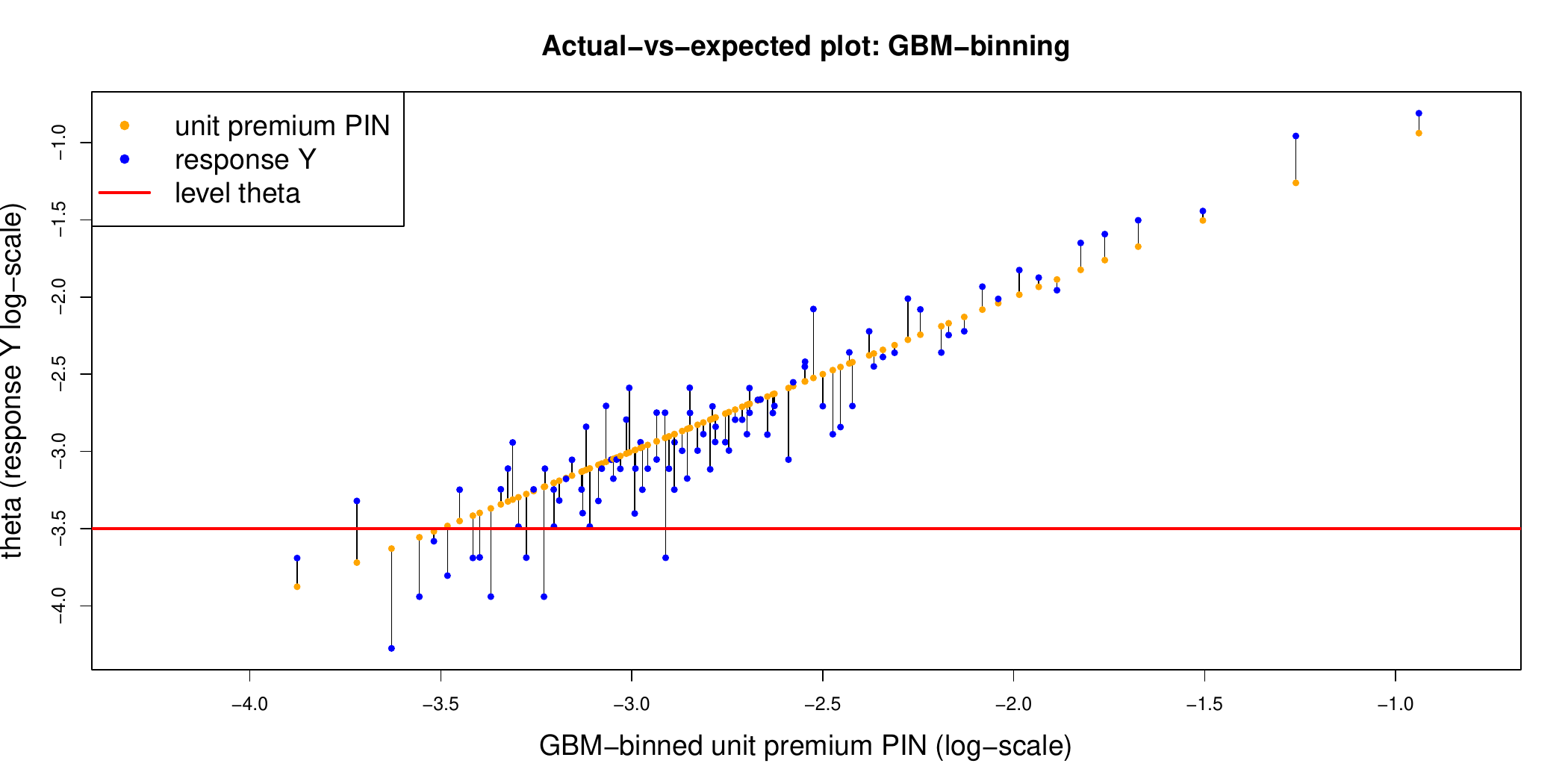}
\end{center}
\end{minipage}
\begin{minipage}[t]{0.9\textwidth}
\begin{center}
\includegraphics[width=\textwidth]{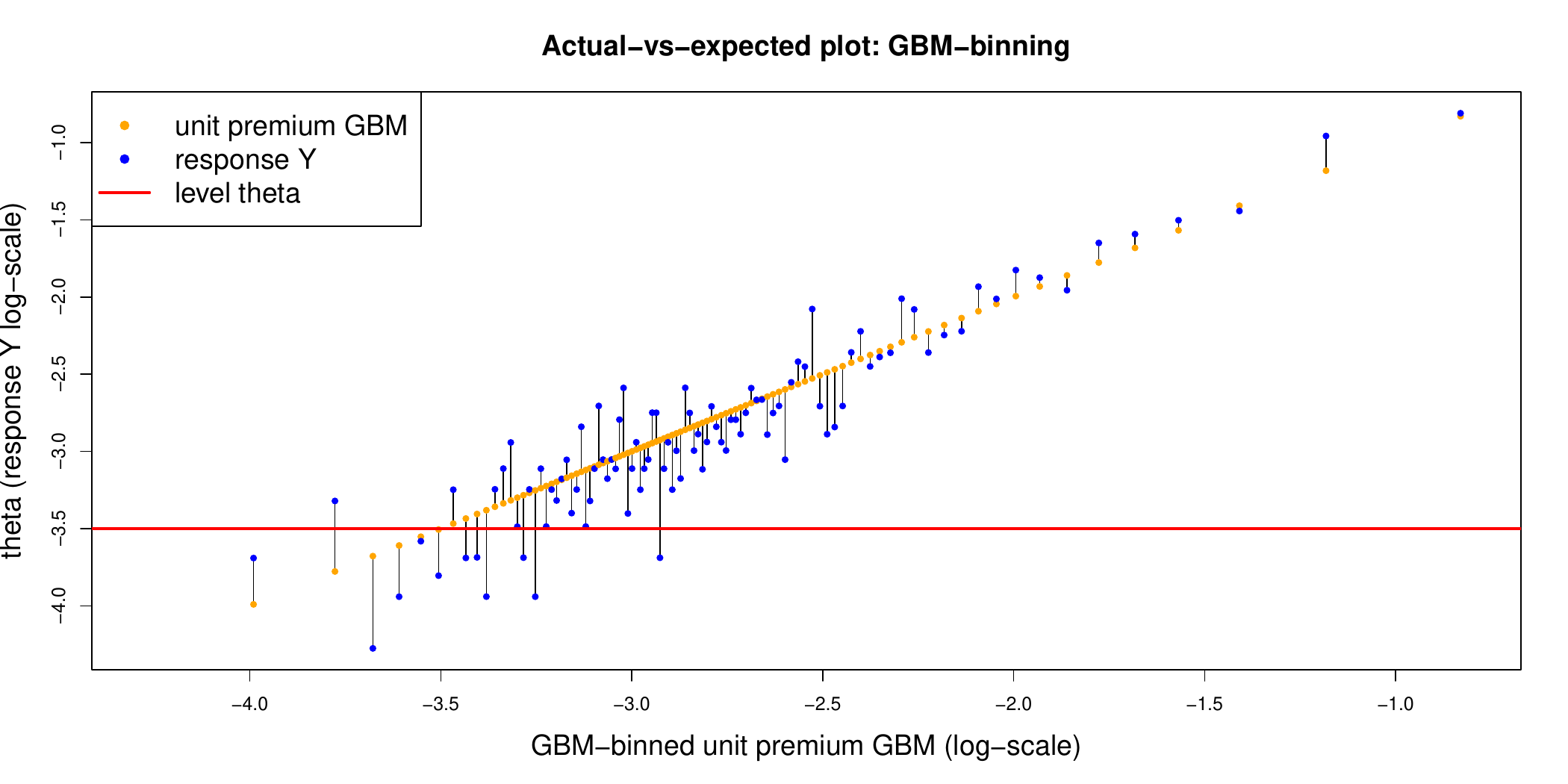}
\end{center}
\end{minipage}
\end{center}
\vspace{-.7cm}
\caption{MTPL AvE plots with GBM percentile binning: upper panel $\overline{Y}^{\rm GBM}_k$ vs.~$\overline{\Pi}^{\textcolor{red}{{\rm PIN}/{\rm GBM}}}_k$ and lower panel
$\overline{Y}^{\rm GBM}_k$ vs.~$\overline{\Pi}^{\textcolor{red}{{\rm GBM}}}_k$; since we work on the log-scale, this graph shows $|\log Y -\log \theta|$ instead of $|Y-\theta|$, but this does not qualitatively change the statements.}
\label{AvE Murphy 1 GBM Binning}
\end{figure}

For the next figure, we only consider the GBM binning version, because from the PIN binning version we cannot learn much.
Figure \ref{AvE Murphy 1 GBM Binning} shows that AvE plots with asymmetric GBM percentile binning $K=100$ on the log-scale for both axes. The figure shows that the vertical segments seem to be generally larger for the PIN version than the GBM one, giving preference to the GBM. The figure may also indicate that the GBM gives a more accurate risk ranking, this is especially true for the upper tail. However, since the binning is asymmetric, these conclusions may not be fully valid.

\subsection{Bregman score}
In the next step, we lift the elementary loss versions to the Bregman scores. For the Poisson deviance loss, we integrate over the elementary losses with weighting function $\varphi''(\theta)=2/\theta$, see \eqref{Bregman loss vs elementary loss}. In view of Figure \ref{MTPL elementary scores}, it is clear that preference is given to the LightGBM unit premiums, because the biggest differences in that figure stem from small values of $\theta$. This is confirmed by Table \ref{MTPL deviance losses}.

\medskip

\begin{figure}[htb!]
\begin{center}
\begin{minipage}[t]{0.45\textwidth}
\begin{center}
\includegraphics[width=\textwidth]{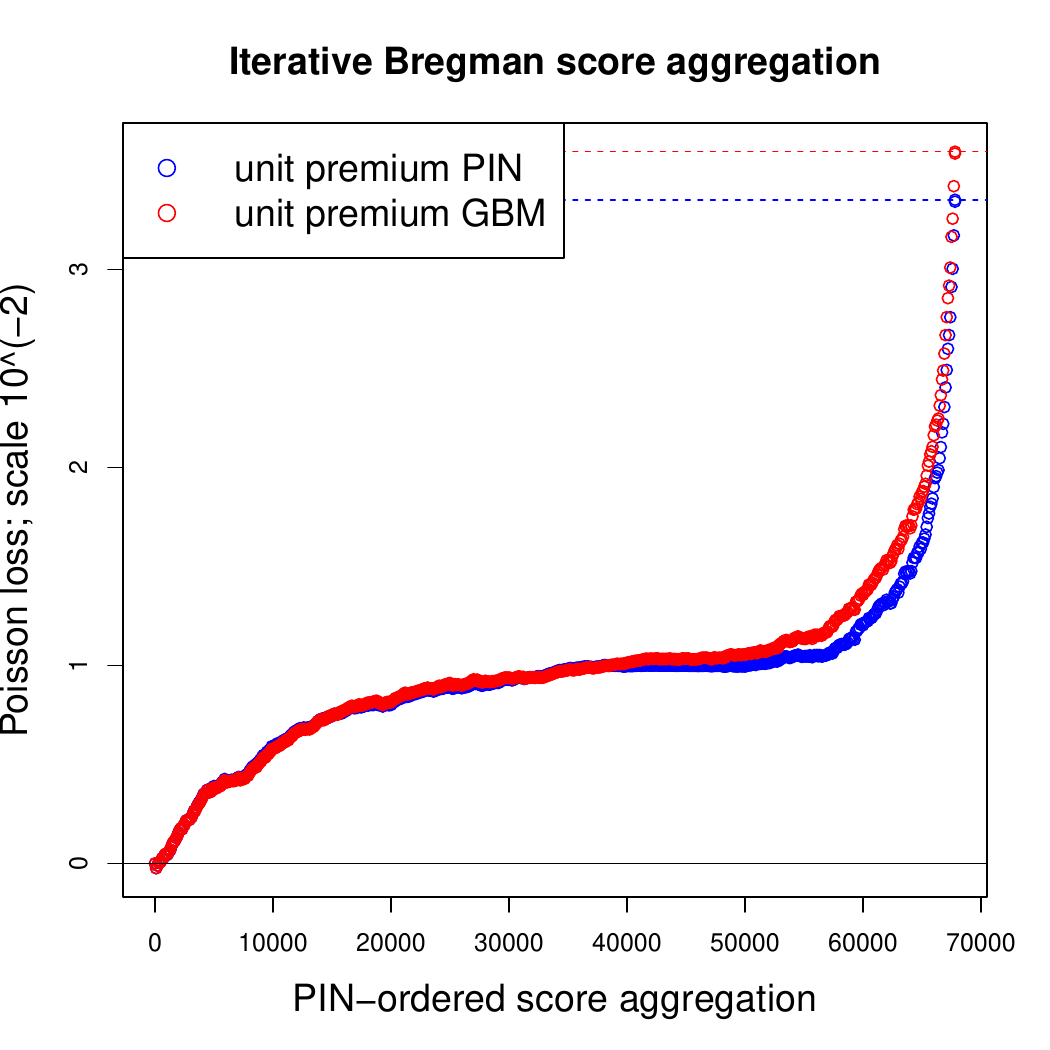}
\end{center}
\end{minipage}
\begin{minipage}[t]{0.45\textwidth}
\begin{center}
\includegraphics[width=\textwidth]{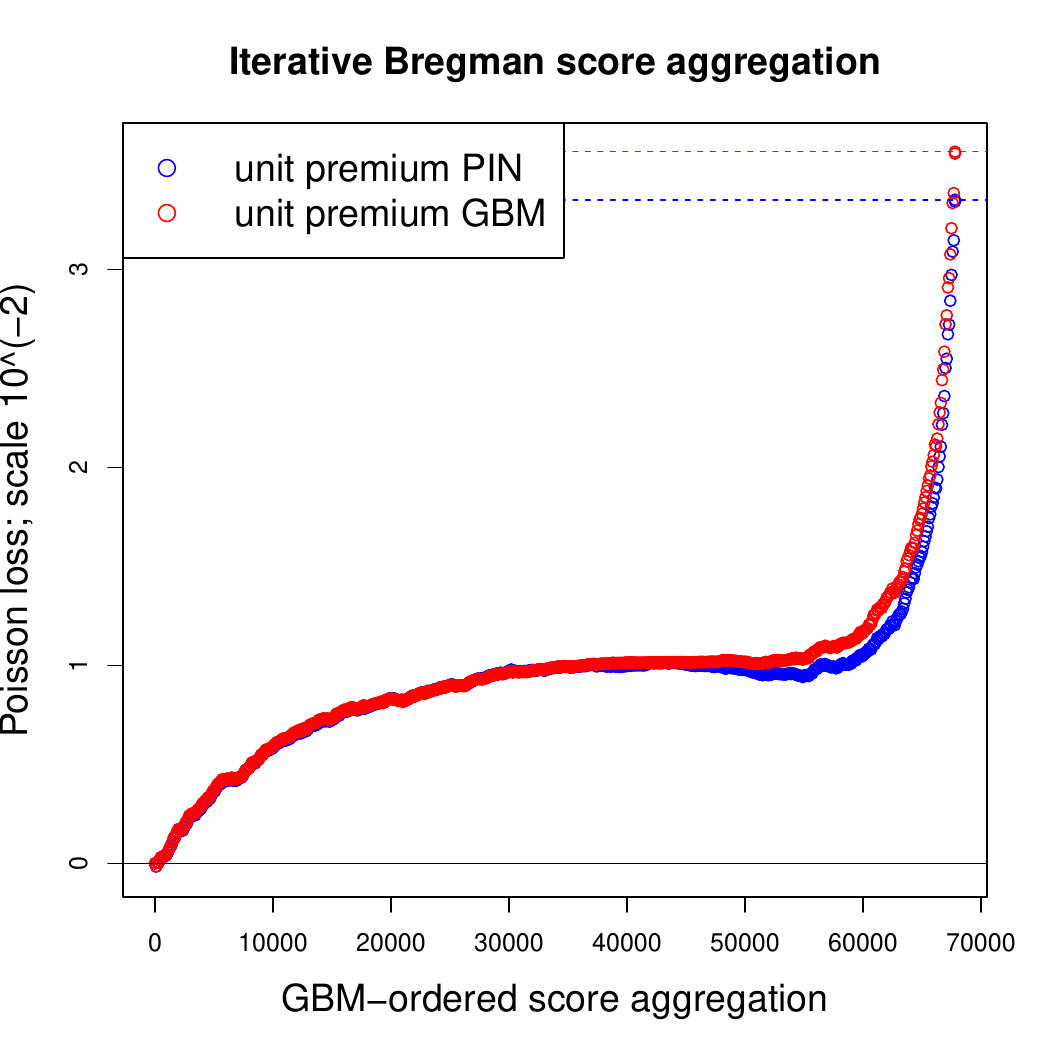}
\end{center}
\end{minipage}
\begin{minipage}[t]{0.45\textwidth}
\begin{center}
\includegraphics[width=\textwidth]{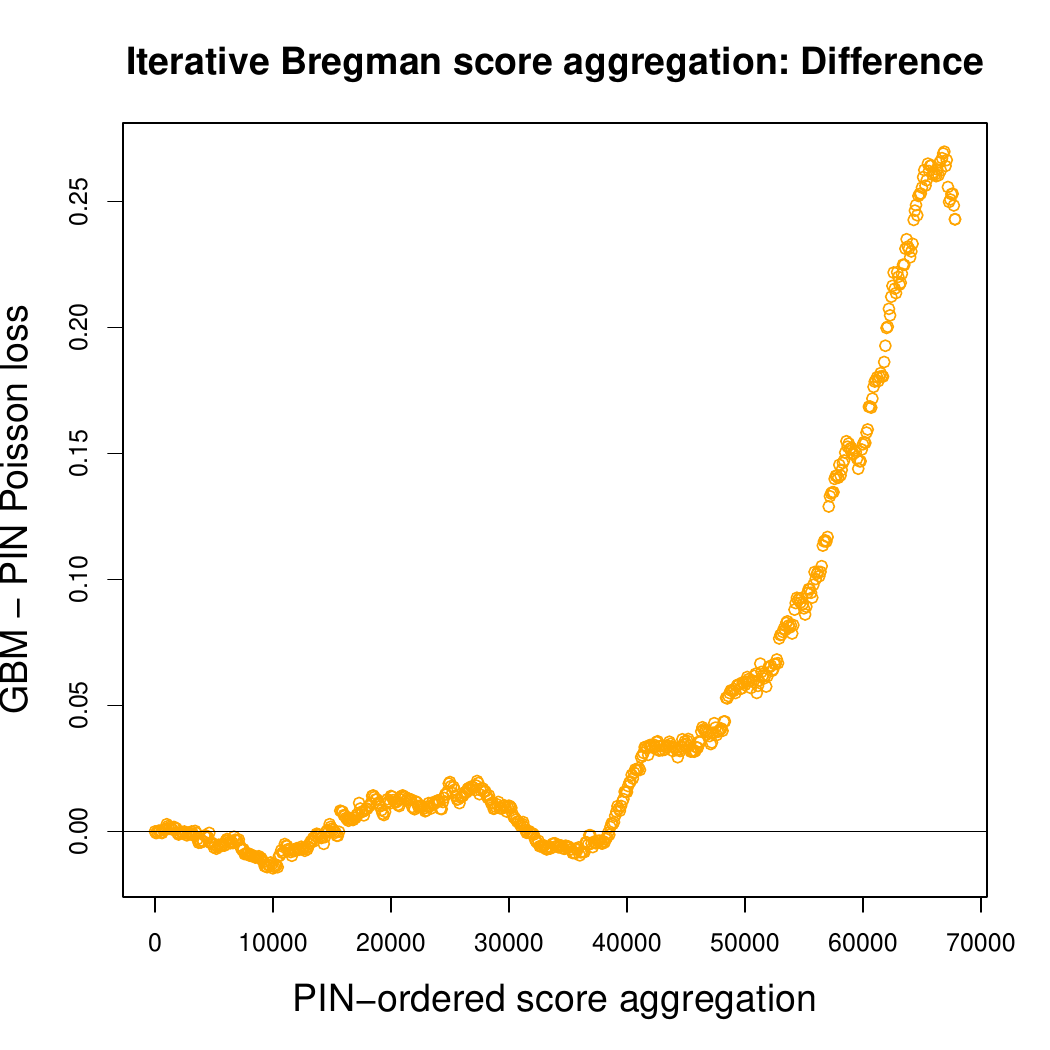}
\end{center}
\end{minipage}
\begin{minipage}[t]{0.45\textwidth}
\begin{center}
\includegraphics[width=\textwidth]{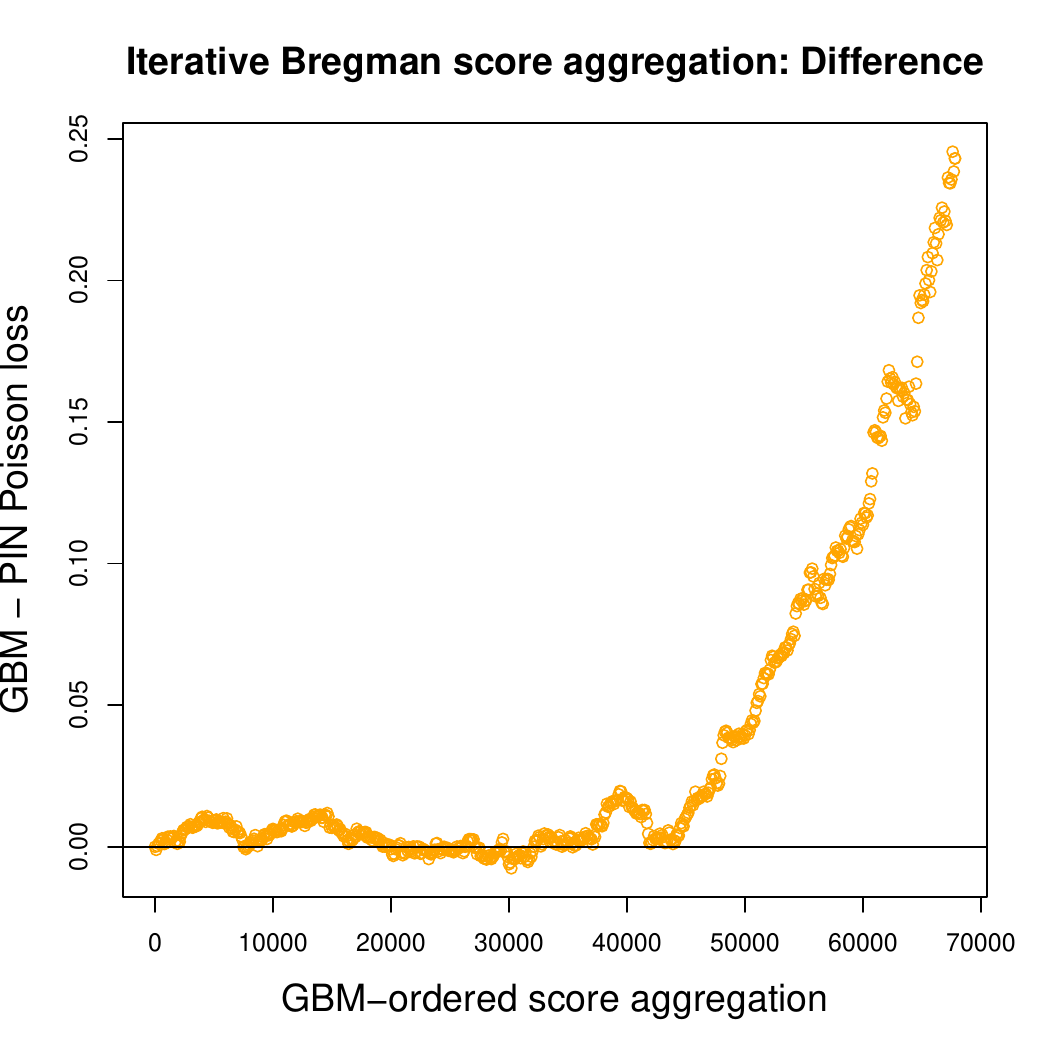}
\end{center}
\end{minipage}
\end{center}
\vspace{-.7cm}
\caption{MTPL paired iterative Bregman score aggregation: (lhs) PIN-ordered aggregation and (rhs) GBM-ordered aggregation. The upper panel shows the score aggregation and the lower panel the differences. The Poisson Bregman scores are 3.352 for PIN (in blue) and 3.595 for GBM (in red), the difference is 0.243 (0.057). }
\label{MTPL ordered Bregman score aggregation}
\end{figure}

To better understand the Poisson Bregman scores, we study the paired iterative Bregman score aggregation from small to large unit premiums \eqref{iterative Bregman 1}-\eqref{iterative Bregman 2}. The left panel of Figure 
\ref{MTPL ordered Bregman score aggregation} considers the PIN ordering and the right panel the GBM ordering, i.e., the two panels aggregate to the same totals, but they differ in the order of aggregation. From these graphs we conclude that the Poisson Bregman scores are very similar for the half of the insurance policies that have lower risks (smaller unit premiums), but the GBM outperforms the PIN especially on the third of the insurance policies with the highest unit premiums. That is, the difference of the Bregman scores of 0.243 in Table \ref{MTPL deviance losses} (3.352 vs.~3.595) can be explained by these biggest unit premium policies.

\subsection{Murphy's decomposition}
In the next step of our analysis we compute Murphy's decomposition that separates the Bregman score into the resolution term and the miscalibration term, see \eqref{Bregman score decomposition}. As mentioned above, the main difficulty in this analysis concerns the recalibration step to receive the calibrated versions of $\Pi^{\rm PIN}$ and 
$\Pi^{\rm GBM}$. We perform this step by isotonic regression, which is precisely the step illustrated in the upper panel of Figure \ref{MTPL isotonic regression lift chart}. Using these isotonically recalibrated unit premiums, we compute the two empirical versions \eqref{MCB1}-\eqref{MCB2} of Murphy's decomposition.

\begin{table}[htb!]
\begin{center}
{\small
\begin{tabular}{|l|rrr|}
\hline
 & Bregman & MCB  & RES\\
premium rules & score & 1 \& 2 & 1 \& 2 \\
\hline\hline
  PIN premium & 3.352 & 0.063 & 3.415\\
  & 3.352 & 0.180 & 3.532\\ \hline
  LightGBM premium &  3.595 & 0.059&3.654\\
   &  3.595 & 0.168&3.763\\
\hline
\end{tabular}}
\end{center}
\caption{Murphy's decomposition \eqref{Bregman score decomposition} using the two empirical versions
\eqref{MCB1}-\eqref{MCB2}; units are shown in $10^{-2}$.}
\label{MTPL deviance losses RES}
\end{table}

Table \ref{MTPL deviance losses RES} shows the results of Murphy's decomposition using the two versions 
\eqref{MCB1}-\eqref{MCB2} under isotonic recalibration steps. Similar to the stylized example, version 1 seems to underestimate miscalibration and version 2 seems to overestimate miscalibration. The first is attributed to missing accuracy in the isotonic regression step and the latter to in-sample overfitting (this is our interpretation but there is no proof in absence of the ground truth). Overall the miscalibration terms seem rather similar between PIN and GBM, with very slight preference for GBM, but this interpretation may not be valid because the isotonic regression uses the premium itself for the recalibration step (i.e., contamination resulting from a wrong risk ranking cannot be controlled).

These similar magnitudes in miscalibration are then mapped to the resolution terms, and the absolute difference between the GBM and the PIN unit premiums is preserved for the resolution validation giving preference to the LightGBM forecast model.

\subsection{Risk ranking}
The final step is to analyze the risk rankings provided by the unit premiums $\Pi^{\rm PIN}$ and 
$\Pi^{\rm GBM}$. For this we show the CAP and the Leimkuhler curves and we compute the Gini scores
\eqref{def Gini 1}.

\medskip

\begin{figure}[htb!]
\begin{center}
\begin{minipage}[t]{0.45\textwidth}
\begin{center}
\includegraphics[width=\textwidth]{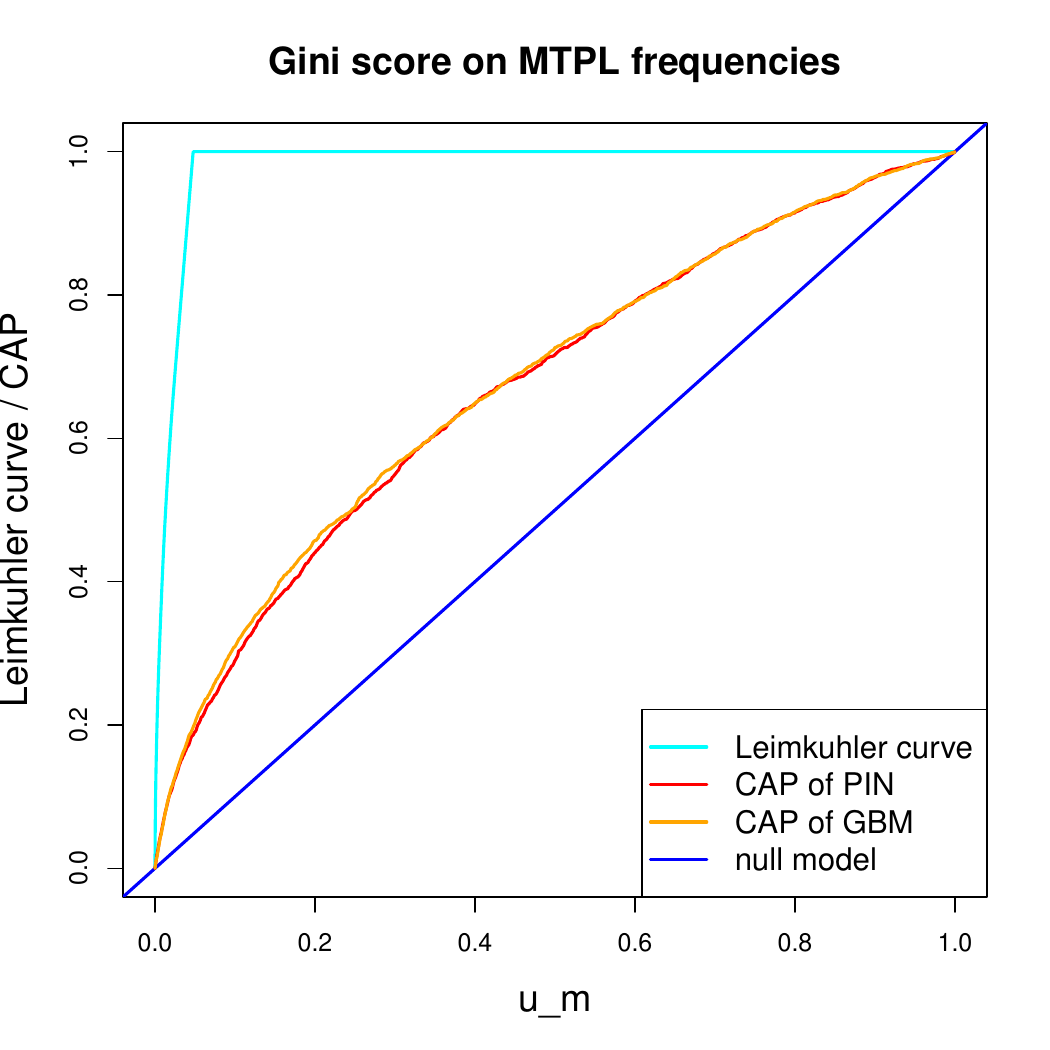}
\end{center}
\end{minipage}
\begin{minipage}[t]{0.45\textwidth}
\begin{center}
\includegraphics[width=\textwidth]{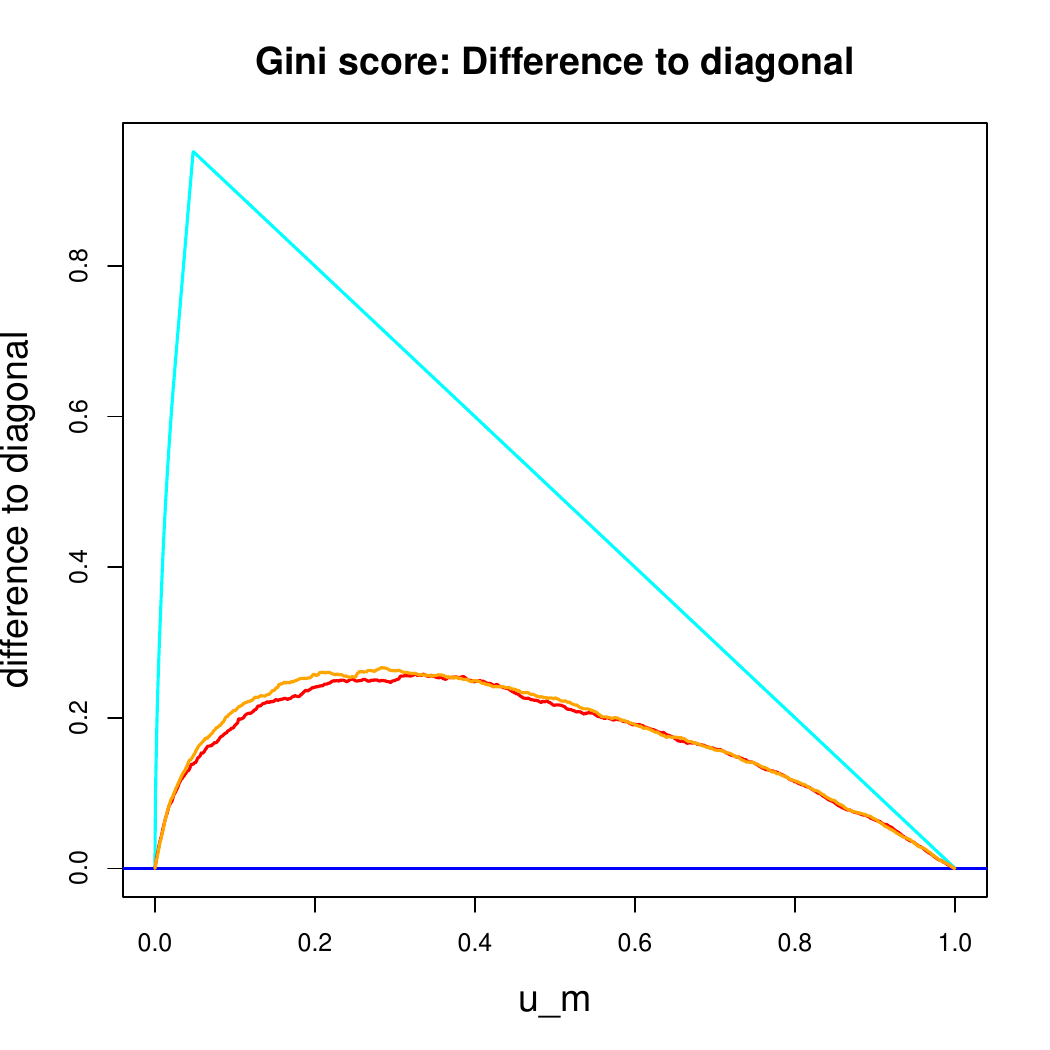}
\end{center}
\end{minipage}
\end{center}
\vspace{-.7cm}
\caption{MTPL Gini plots: (lhs) CAPs of PIN unit premiums $(\Pi^{\rm PIN}_i)_{i=1}^n$ and GBM unit premiums $((\Pi^{\rm GBM}_i)_{i=1}^n$ and Leimkuhler curve of $(Y_i)_{i=1}^n$; and (rhs) with subtracted diagonal for better visualization.}
\label{figure Gini plots MTPL}
\end{figure}

Figure \ref{figure Gini plots MTPL} shows the CAPs of the two unit premium rules $(\Pi^{\rm PIN}_i)_{i=1}^n$ and $(\Pi^{\rm GBM}_i)_{i=1}^n$. From the plot on the right-hand side, it can be seen that the GBM risk ordering dominates the PIN one, and this is also confirmed by the Gini scores in Table \ref{Gini score table MTPL}. That is, the GBM provides the more accurate risk ranking relative to the observed responses $(Y_i)_{i=1}^n$. To verify statistical significance of the Gini score difference of $0.3668-0.3568=0.0100$ we performed a paired non-parametric bootstrap (drawing with replacement on the test sample), and the bootstrap standard deviation is 0.0033.

\begin{table}[htb!]
\begin{center}
{\small
\begin{tabular}{|l|r|}
\hline
premium rules & Gini score \\
\hline\hline
global mean model $\mu_{\q}$ & 0.0000 \\
premium rule PIN & 0.3568  \\
premium rule GBM & 0.3668 \\
\hline
\end{tabular}}
\end{center}
\caption{MTPL Gini scores assessing the risk rankings of the PIN and the GBM unit premium rankings.}
\label{Gini score table MTPL}
\end{table}

\section{Summary}

The main goal of this manuscript was to present graphical and quantitative model validation tools. We presented these in a proper out-of-sample model validation analysis. We emphasize the following points:
\begin{itemize}
\item Actuarial statistical modeling considers exposure-weighted quantities to make losses and premiums comparable across insurance policies with different exposures. Exposure-scaled quantities generally still depend on that exposure, i.e., by an exposure scaling one does not get rid of the exposure-dependence in the scaled responses (unit losses). For example, a mean-independence of the scaled response requires further assumptions. Such assumptions are made, for instance, when working within the exponential dispersion family. However, such assumptions may not generally be satisfied in real-world applications.
\item To properly account for the dependence between the exposure and the exposure-scaled losses, it is important to work under the exposure-weighted distribution. Otherwise one may obtain systematically biased forecasts. This requires a change of measure from the classical policy-weighted population distribution to the exposure-weighted distribution.
\item There are different notions of calibration. Generally, the exposure-weighted calibration, given the unit premium, is the correct calibration view in actuarial pricing. It is based on the available information at contract inception. The available information often only includes the unit premium, but not the exposure, because early termination and suspension of contracts makes the exposure only an ex-post available variable. Moreover, this calibration view accounts for the dependence between the unit losses and the exposure, and it controls (mitigates) systematic cross-subsidy between different unit premium classes.
\item Good forecast models perform well in {\it calibration}, {\it discrimination} ({\it resolution}) and {\it risk ranking}:
\begin{itemize}
\item Calibration can be assessed by studying calibration plots, actual-vs-predicted plots and lift charts. The sample version of Murphy's decomposition (based on isotonic regression) gives a quantitative measure of calibration. A critical point of this sample version is that the isotonic regression step uses the test data, and the subsequent analysis is not truly out-of-sample leading to biases. This critical point requires future research.
  \item Discrimination can be assessed by the (double) lift chart, the Murphy diagram and the paired iterative Bregman score aggregation.
The Bregman score gives a quantitative tool to assess the overall predictive accuracy, and the sample version of Murphy's decomposition gives a discrimination and resolution statistics.
\item Risk ranking can be assessed by the cumulative accuracy profile yielding the rank-based  Gini score.
\end{itemize}
\item Forecast dominance is usually a too strong condition for model selection. The Murphy diagram based on elementary losses may help us to do an informed model selection because it illustrates how the Bregman score is composed for different convex generators. The selection of the specific Bregman score is often done within the Patton family for actuarial pricing problems.

\end{itemize}  

Our outline has been based on graphical tools and point estimates. Naturally, the next step is to derive confidence bounds for these point estimates to perform proper statistical testing. We presented two paired bootstrap analyses for the Bregman score difference and the Gini score difference in the applied motor insurance example. Naturally, much more remains to be done. Many questions and tools are still active research problems, we mention the calibration tests of Denuit et al.~\cite{DenuitEtAl2024}, Gatti \cite{Gatti} or Delong--W\"uthrich \cite{DW2}. A critical issue in many of these developments is the missing accuracy in the isotonic regression step, methods that attempt to address this limitation consider, e.g., a boosting step; see Gatti \cite{Gatti}. Similar things can be said for discrimination testing, and for the Gini score we mentioned the asymptotic results of Frees et al.~\cite{Frees1} and Brauer et al.~\cite{BrauerMenzel}.

\bigskip

{\small 
\renewcommand{\baselinestretch}{.51}
}

\newpage

\appendix

\section{Mathematical proofs}

\bigskip

{\Beweis 
{\bf Proof of Lemma \ref{Pythagoras}.}
Using ${\cal A}$-measurability and the functional form of the Bregman loss
\begin{equation*}
\E_{\q}\left[ L_{\varphi}(Y,X)\mid{\cal A}\right]
=
\E_{\q}\left[\varphi(Y)\mid{\cal A}\right]-\varphi(X)-\varphi'(X)(m_{\q}({\cal A})-X).
\end{equation*}
The last term vanishes for $X=m_{\q}({\cal A})$ which gives us
\begin{equation*}
\E_{\q}\left[ L_{\varphi}(Y,m_{\q}({\cal A}))\mid{\cal A}\right]
=
\E_{\q}\left[\varphi(Y)\mid{\cal A}\right]-\varphi(m_{\q}({\cal A})).
\end{equation*}
Subtracting the two identities proves the claim.
\EndProof}

\bigskip

{\Beweis
{\bf Proof of Corollary \ref{cor Murphy}.} We take the expectation over \eqref{conditional Bregman} providing
\begin{eqnarray*}
 \E_{\q}\left[ L_{\varphi}(Y,\Pi)\right]
  &=&
  \E_{\q}\left[ L_{\varphi}(Y,m_{\q}(\Pi))\right]
  + \E_{\q}\left[L_{\varphi}(m_{\q}(\Pi),\Pi)\right]
\\  &=&\E_{\q}[L_\varphi(Y, \mu_{\q})]
+\left( \E_{\q}\left[ L_{\varphi}(Y,m_{\q}(\Pi))\right]- \E_{\q}[L_\varphi(Y, \mu_{\q})]\right)
  + \E_{\q}\left[L_{\varphi}(m_{\q}(\Pi),\Pi)\right].
\end{eqnarray*}
The round bracket can again be evaluated with \eqref{eq:conditional-pythagorean} by selecting $X=\mu_{\q}$. This proves the claim.
\EndProof}

\bigskip

\section{LightGBM code}

\begin{lstlisting}[float=h!,frame=tb,caption={LightGBM hyper-parameter specification.},label=LightGBMRCode]
library(lightgbm)

# Data interface of LightGBM
dtrain <- lgb.Dataset(
  X_train,
  label = y_train,
  weight = learn$Exposure,
  params = list(feature_pre_filter = FALSE)
)

# Parameters
params <- list(
  learning_rate = 0.05,
  objective = "poisson",
  metric = "poisson",
  num_leaves = 7,
  min_data_in_leaf = 50,
  min_sum_hessian_in_leaf = 0.001,
  colsample_bynode = 0.8,
  bagging_fraction = 0.8,
  lambda_l1 = 3,
  lambda_l2 = 5,
  num_threads = 7
)

fit_lgb <- lgb.train(params = params, data = dtrain, nrounds = 3000)
\end{lstlisting}

\newpage

\section{Miscalibration and resolution plots}

\begin{figure}[htb!]
\begin{center}
\begin{minipage}[t]{0.45\textwidth}
\begin{center}
\includegraphics[width=\textwidth]{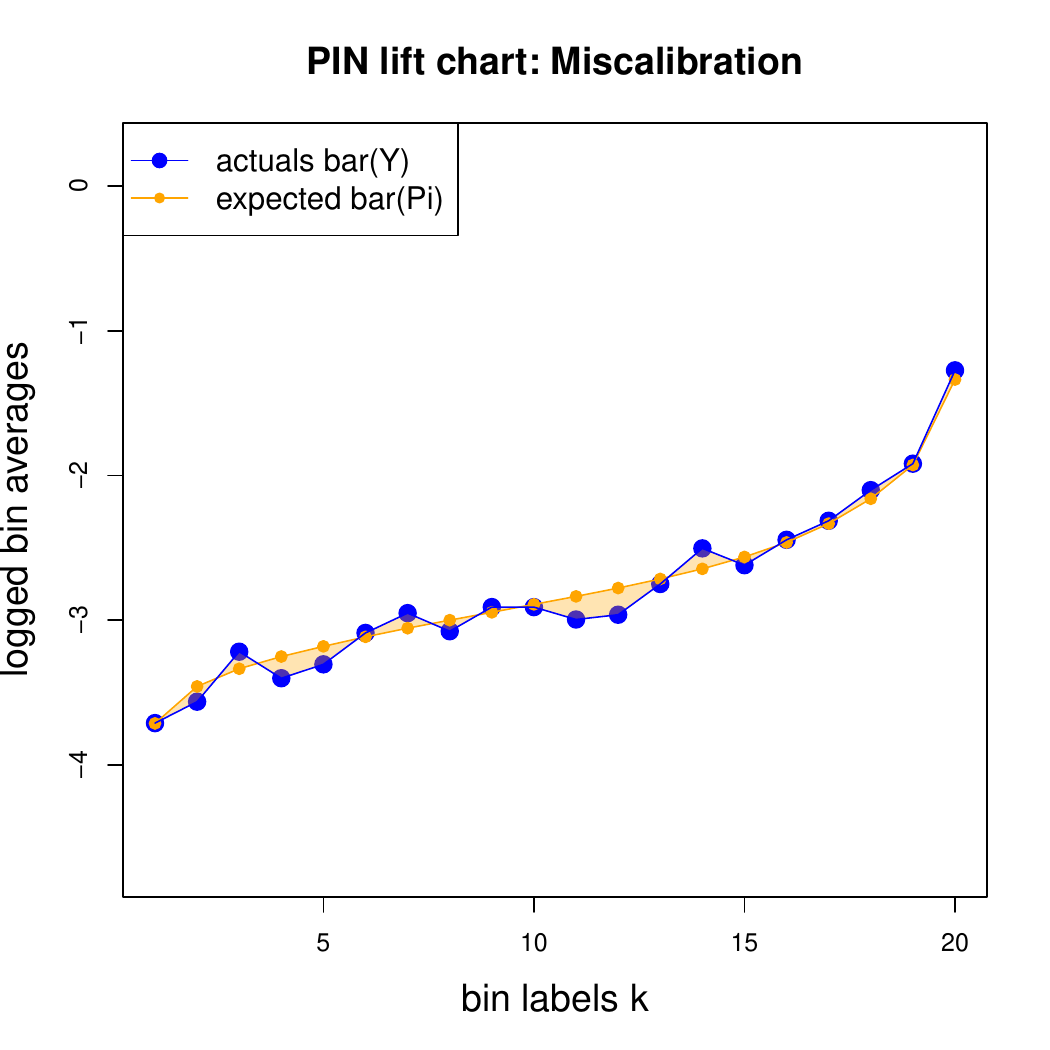}
\end{center}
\end{minipage}
\begin{minipage}[t]{0.45\textwidth}
\begin{center}
\includegraphics[width=\textwidth]{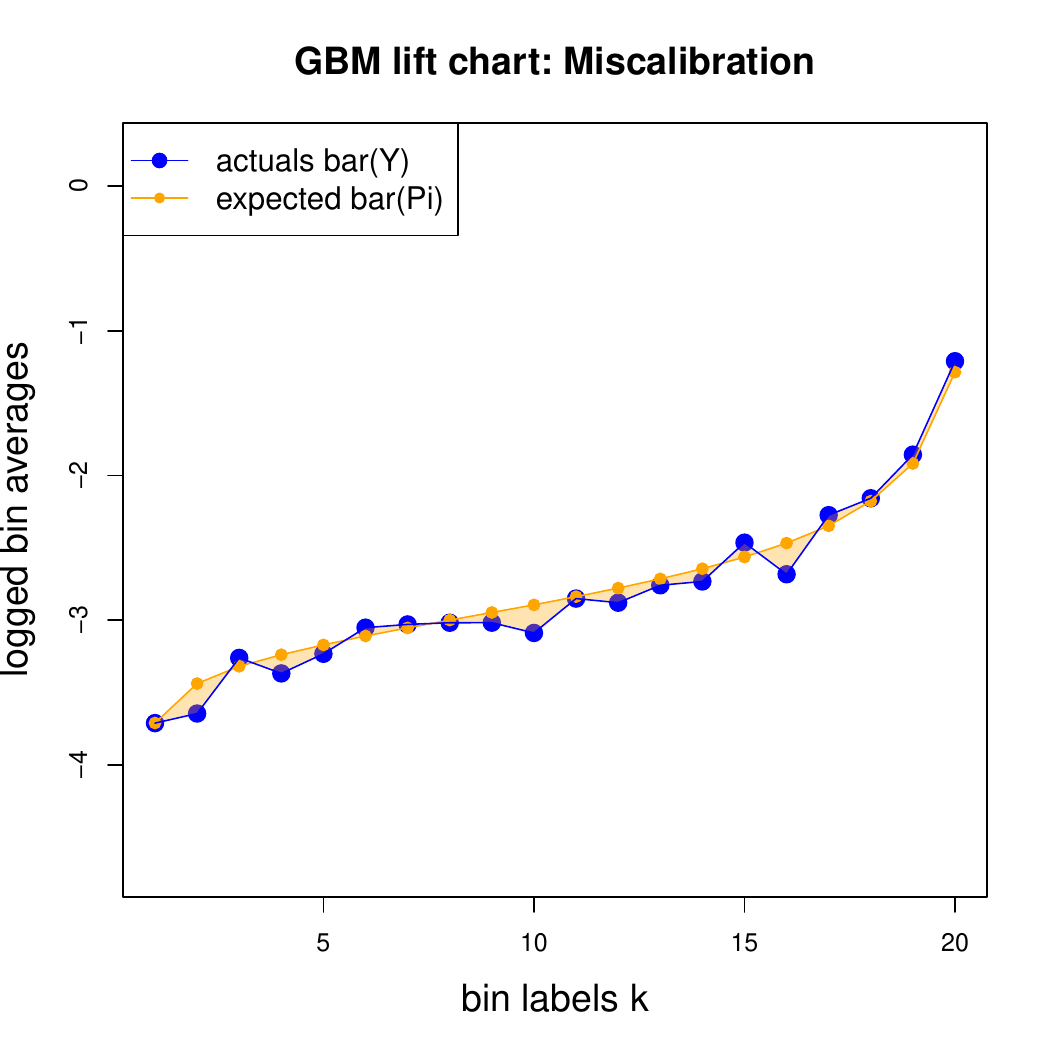}
\end{center}
\end{minipage}
\begin{minipage}[t]{0.45\textwidth}
\begin{center}
\includegraphics[width=\textwidth]{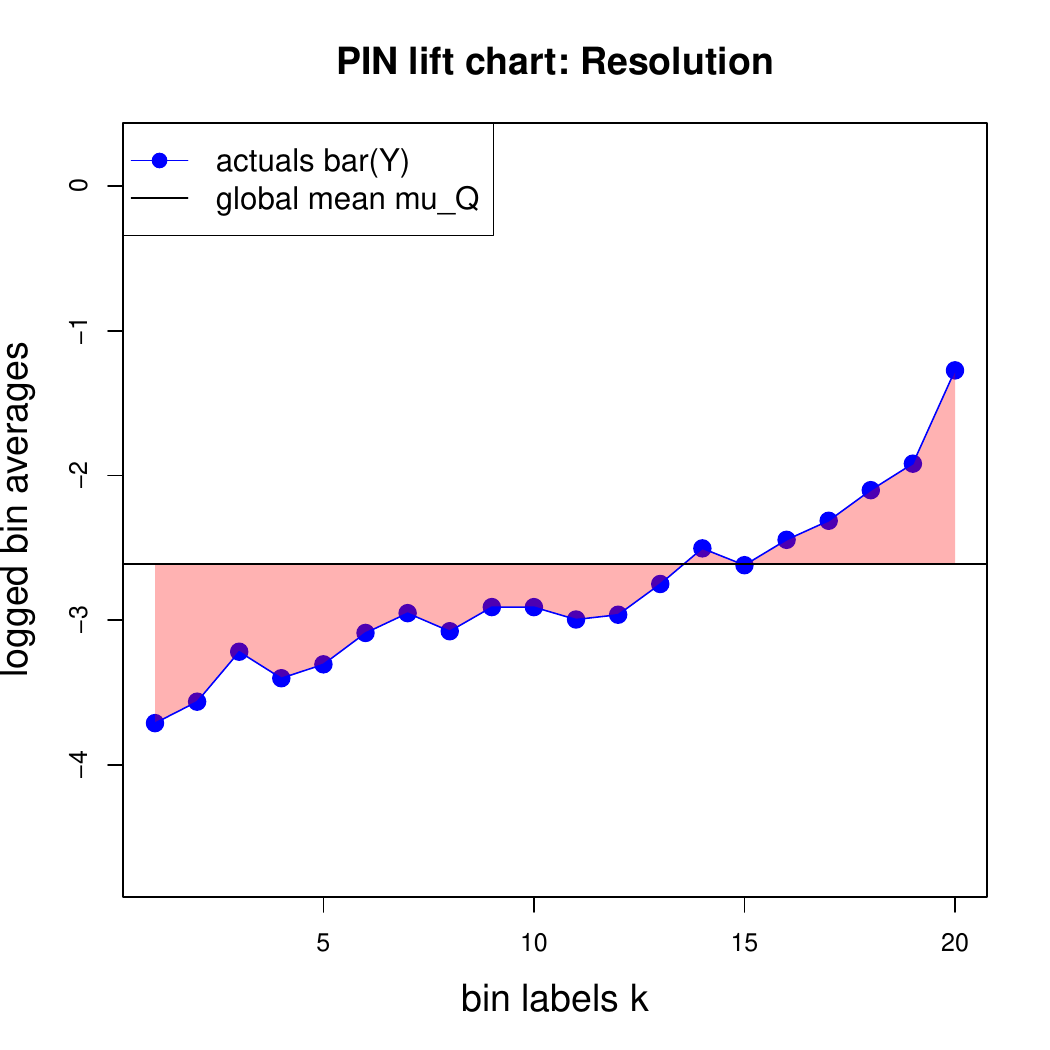}
\end{center}
\end{minipage}
\begin{minipage}[t]{0.45\textwidth}
\begin{center}
\includegraphics[width=\textwidth]{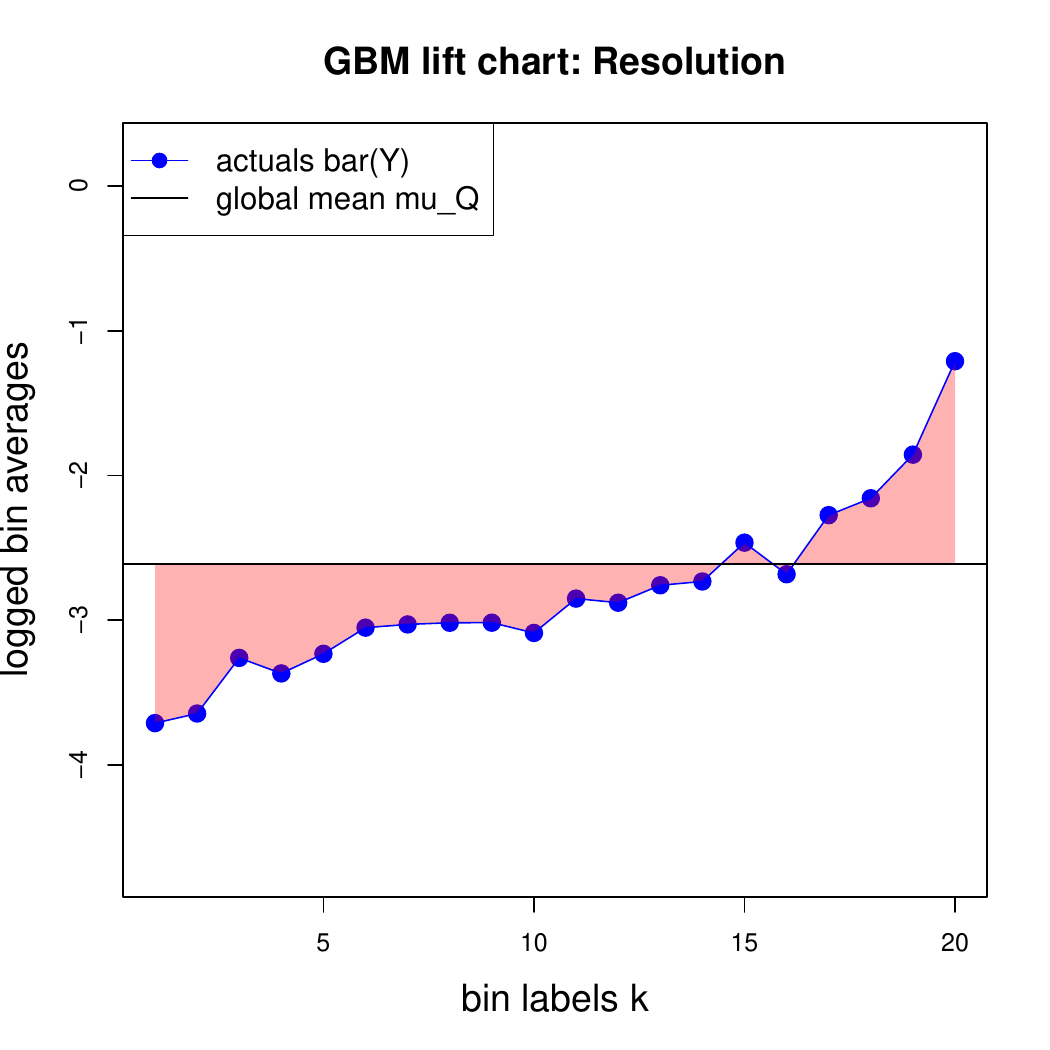}
\end{center}
\end{minipage}
\end{center}
\vspace{-.7cm}
\caption{MTPL lift charts of Figure \ref{MTPL lift chart plot}, revisited: visualizations of (top) estimated miscalibration of PIN and GBM, (bottom) estimated resolution of PIN and GBM using ventile binning $K=20$ and the bin averages $\overline{Y}_k$ for the recalibration estimate $\widehat{m}_{\q}(B)$, $B=k$, we also refer to Figure \ref{Resolution plot}; these graphs consider the raw differences before entering the Bregman score.}
\label{MTPL resolution and miscalibration plots}
\end{figure}

\end{document}